\documentclass{article}
\usepackage{color}
\usepackage[dvipsnames]{xcolor}
\usepackage{booktabs}
\usepackage{fullpage}
\usepackage[top=2.cm, bottom=2.cm, left=2.5cm, right=2.5cm, columnsep=0.75cm]{geometry}
\usepackage{hyperref}
\definecolor{myblue}{rgb}{0,0.2,0.8}
\hypersetup{ %
pdftitle={},
pdfauthor={},
pdfsubject={},
pdfkeywords={},
pdfborder=0 0 0,
pdfpagemode=UseNone,
colorlinks=true,
linkcolor=myblue,
citecolor=myblue,
filecolor=myblue,
urlcolor=myblue
}

\usepackage{natbib}
\usepackage{colortbl}
\usepackage{multirow}
\usepackage{mathtools}
\usepackage{graphicx}
\usepackage{caption}
\usepackage{rotating}
\usepackage{arydshln}
\usepackage{hhline}
\usepackage{hyperref}
\usepackage{bbm}
\usepackage{hyperref, comment}
\usepackage{multirow}
\usepackage[utf8]{inputenc}
\usepackage{graphicx}
\usepackage{caption}
\usepackage{expl3}
\usepackage[caption = false]{subfig}
\usepackage{tcolorbox}
\usepackage{amsmath,amsfonts,bm, hyperref, amsthm}

\newtheorem{defn}{Definition}[section]
\newtheorem{prop}{Proposition}[section]
\newtheorem{asmp}{Assumption}[section]
\newtheorem{thm}{Theorem}[section]
\newtheorem{lemma}{Lemma}[section]
\newtheorem{remark}{Remark}[section]
\def\eqref#1{(\ref{#1})}
\def\1{\bm{1}}

\def\v0{{\bm{0}}}

\def\mA{{\bm{A}}}
\def\mB{{\bm{B}}}
\def\mC{{\bm{C}}}

\DeclareMathAlphabet{\mathsfit}{\encodingdefault}{\sfdefault}{m}{sl}
\SetMathAlphabet{\mathsfit}{bold}{\encodingdefault}{\sfdefault}{bx}{n}

\PassOptionsToPackage{capitalise,noabbrev}{cleveref}
\usepackage{cleveref}
\usepackage{graphicx}

\def \a2 {\mA_2}
\def \qa2 {\mA_{2q}}
\def \b2 {\mB_2}
\def \qb2 {\mB_{2q}}
\def \c2 {\mC_2}
\def \qc2 {\mC_{2q}}

\def \omn {\omega_n}

\def \smqs#1 { \sum_{k=0}^{p-1} \sum_{r=\lceil -k/n\rceil }^{\lfloor (p-1-k)/n \rfloor} t_k^{2q} t_{n r+k}^#1 e_{s,k} }

\def \tr#1 {\text{tr}\left(#1\right)}

\def \p#1 {{\left(#1\right)}}
\def \norms#1 {\left\|#1\right\|^2}
\def \bs#1 {{\left[#1\right]}}
\def \br#1 {{\left\{#1\right\}}}
\def \inp#1 {{\langle#1\rangle}}

\def \suma#1 { \sum_{k=0}^{n-1} \p{ \sum_{\nu=0}^{l-1}  t_{k+n\nu}^#1 } }
\def \sumc#1 {  \sum_{k=0}^{n-1} \p{\sum_{\nu =l}^{\tau -1} t_{k+n\nu}^#1 } }
\def \las#1 { \suma#1 \p{ \sum_{j=0}^{n-1}  \omn^{(s-k)j}} }
\def \lcs#1 { \sumc#1 \p{ \sum_{j=0}^{n-1}  \omn^{(s-k) j}} }

\def \sinsuma#1 { \sum_{\nu=0}^{l-1}  t_{k+n\nu}^#1  }
\def \sinsumc#1 { \sum_{\nu =l}^{\tau  -1} t_{k+n\nu}^#1 } 

\def \a {\alpha}

\usepackage[flushleft]{threeparttable}
\usepackage{graphicx,array}
\usepackage{booktabs}
\usepackage{floatrow}
\definecolor{dkgreen}{HTML}{417E27}

\newcommand{\todo}[1]{\textcolor{dkgreen}{[#1]}}

\newcommand{\lingxiao}[1]{\textcolor{purple}{[Lingxiao: #1]}}
\usepackage{amssymb,amsmath}
\title{Adaptive Differentially Private Empirical Risk Minimization}

\author{%
  Xiaoxia Wu$^\star\dagger$\\
  \texttt{xwu@ttic.edu} \\
  \and
  Lingxiao Wang$^\dagger$\\
  \texttt{lingxw@ttic.edu} \\
  \and
 Irina Cristali$^\star$ \\
  \texttt{icristali@uchicago.edu} \\
  \and
  Quanquan Gu$^\ddagger$ \\
  \texttt{qgu@cs.ucla.edu} \\
  \and
  Rebecca Willett$^\star$  \\
  \texttt{willett@uchicago.edu} \\
  \and
\\$\star$ University of Chicago \\$\dagger$ Toyota Technological Institute at Chicago \\ $\ddagger$ University of California, Los Angeles }
\date{}
\usepackage{algorithm}
\usepackage{algpseudocode}
\usepackage{cleveref}

\crefname{prop}{Proposition}{Prop.}
\crefname{thm}{Theorem}{Thm.}
\crefname{algorithm}{Algorithm}{Alg.}
\crefname{figure}{Figure}{Fig.}
\crefname{defn}{Definition}{Defn.}
\crefname{section}{Section}{Sec.}
\crefformat{equation}{(#2#1#3)}
\def\deq{\triangleq}

\usepackage{enumitem}
\setitemize{noitemsep,labelwidth=.5em,labelsep=.4em,topsep=0pt,parsep=0pt,partopsep=0pt,leftmargin=1em}
\setenumerate{noitemsep,labelwidth=.5em,labelsep=.4em,topsep=0pt,parsep=0pt,partopsep=0pt,leftmargin=1em}

\newenvironment{squishlist}
{   \begin{list}{$\bullet$}
    { \setlength{\itemsep}{0pt}      \setlength{\parsep}{0pt}
      \setlength{\topsep}{0pt}       \setlength{\partopsep}{0pt}
      \setlength{\leftmargin}{1em} \setlength{\labelwidth}{.5em}
      \setlength{\labelsep}{0.4em} } }
      {\end{list}}
      
\begin{document}

\maketitle
\begin{abstract}
% \me{Need to emphasize the reason why we study adaptive learning rate and why adaptive privacy; adaptive to the unknown parameters: Lipschitz smoothness, universal constant, optimal parameters?}.
% ADP-SGD
% We propose an adaptive differentially private method for empirical risk minimization using gradient perturbation. 
% \todo{I feel the previous sentence is confusing; what is adaptive? How about: We propose a differentially private method for empirical risk minimization using adaptive stochastic gradient perturbation}
% \me{adaptive gradient methods are well-known for AdaGrad or Adam, I don't want people get confused by saying adaptive gradient}

We propose an adaptive (stochastic) gradient perturbation method for differentially private empirical risk minimization. At each iteration, the random noise added to the gradient is optimally adapted to the stepsize; we name this process adaptive differentially private (ADP) learning.  Given the same privacy budget, we prove that the ADP method considerably improves the utility guarantee compared to the standard differentially private method in which vanilla random noise is added. Our method is particularly useful for gradient-based algorithms with time-varying learning rates, including variants of AdaGrad (Duchi et al., 2011). We provide extensive numerical experiments to demonstrate the effectiveness of the proposed adaptive differentially private algorithm.
%For adaptive gradient methods, more generally a learning rate with approximately square root decay ${1}/\sqrt{1+t}$ at $t$ iteration, we can improve the convergence bound with ${\log(T)}$ after $T$ iterations while maintaining $(\varepsilon, \delta)-$DP.
\end{abstract}
% \tableofcontents
% \section{Introduction}
\section{Introduction}

%%%%%%%%%%%%%%%%%%%%%%%%%%% background
Publishing deep neural networks such as ResNets \citep{he2016deep} and Transformers \citep{vaswani2017attention} (with billions of parameters) trained on private datasets has become a major concern in the machine learning community; these models can memorize the private training data and can thus leak personal information,  such as social security numbers \citep{carlini2020extracting}. Moreover, these models are vulnerable to privacy attacks, such as membership inference \citep{shokri2017membership,gupta2021membership} and reconstruction \citep{fredrikson2015model,nakamura2020kart}. Therefore, over the past few years, a considerable number of methods have been proposed to address the privacy concerns described above. One main approach to preserving data privacy is to apply differentially private (DP) algorithms \citep{dwork2006our, dwork2014algorithmic,abadi2016deep,jayaraman2020revisiting} to train these models on private datasets. %\irina{say where to apply, as it's a bit unclear; in the optimization step in the NN training...}.
% A common differentially private method for training a model via gradient-based optimization is to add some random noise, such as Gaussian or Laplacian noise, to the gradients during the training process, which is called differentially private stochastic gradient descent (DP-SGD)
 Differentially private stochastic gradient descent (DP-SGD) is a common privacy-preserving algorithm used for training a model via  gradient-based optimization; DP-SGD adds random noise to the gradients during the optimization process
\citep{bassily2014private,song2013stochastic,bassily2020stability}.
% \irina{re-write as:}.

% a paragraph to explain the popularity of adaptive stochastic gradient methods 

% To tackle the stepsize tuning problem, we turn our attention to adaptive gradient methods such as AdaGrad \citep{duchi2011adaptive,mcmahan2010adaptive} and Adam \citep{kingma2014adam,j.2018on}; these methods have become the de facto standard stochastic optimization methods for non-convex deep learning  problems, particularly in the domain of large-scale natural language processing  \citep{vaswani2017attention,devlin2019bert,NEURIPS2020_1457c0d6,liu2019roberta,2020t5}. Empirical performances show that adaptive gradient methods are particularly beneficial to sparse gradient update \citep{duchi2011adaptive}.  Moreover, recent work provided both theoretical and empirical evidence that  the celebrated  methods are highly adaptive to the noisy gradients and achieve favourable convergence behavior, which results in significant reduction in stepsize tuning efforts  \citep{levy2018online,orabona18,ward2019adagrad,reddi2021adaptive}. %wu2019global,xie2020linear
%%%%%%%%%%%%%%%%%%%%%%%%%%% problem 
% \me{need to rewrite: how to say the adaptive gradient methods is helpful? how to say DP-SGD is not the best option?} \todo{yes, need much more here. What are the limitations of DP-SGD?}
To be concrete, consider the empirical risk minimization (ERM) on a dataset $\mathcal{D}=\{x_i\}_{i=1}^n$, where each data point $x_i\in{\cal X}$. 
% with the non-convex smooth loss function $f(\cdot): \mathbb{R}^d \to \mathbb{R}$. Suppose $\|\nabla f(\theta)\|\leq G$ and $\|\nabla f(x)-\nabla f(y)\|\leq L\|x-y\|$. Define a dataset with $|\mathcal{D}|=n$. Consider the empirical risk with a sample $x_i\in \mathcal{D}$: 
We aim to obtain a private high dimensional parameter $\theta\in \mathcal{R}^d$ by solving
% \todo{$\{x_i\}_{i=1}^n = {\cal D}$?}
% {\begin{footnotesize}
\begin{align}\label{eq:erm}
% %%\vspace{-0.15cm}
    \min_{\theta\in \mathbb{R}^d} F(\theta) := \frac{1}{n}\sum_{i=1}^n f_i(\theta), \text{ with }f_i(\theta) = f(\theta;x_i)
    % %%\vspace{-0.15cm}
\end{align} 
% \end{footnotesize}}
where the loss function $f(\cdot): \mathbb{R}^d \times \cal X \to \mathbb{R}$ is non-convex and smooth at each data point. To measure the performance of gradient-based algorithms for ERM, which enjoys privacy guarantees, we define the \emph{utility} by using the expected $\ell_2$-norm of gradient, i.e.,  $\mathbb{E}[\|\nabla F(\theta)\|]$, where the expectation is taken over the randomness of the algorithm \citep{wang2017differentially,zhang2017efficient,wang2019efficient,zhou2020private}.\footnote{We examine convergence through the lens of utility guarantees; one may interchangeably use the two words ``utility" or ``convergence". } % to measure the performance of different algorithms, 
% As we mentioned before, DP-SGD is one of the most commonly used method to achieve differential privacy when solving the problem in \eqref{eq:erm}. More specifically,
%  \irina{can we re-phrase as: The DP-SGD with a Gaussian mechanism solves the ERM in equation 1, and also offers a privacy guarantee, by performing the following update...?} To solve the ERM in \eqref{eq:erm} with a privacy guarantee, DP-SGD with a Gaussian mechanism performs the following update at the $t$-th iteration for $t\geq 0$ and $\theta_0 \in \mathbb{R}^d$,  by performing the following update %can be expressed as
 The DP-SGD with a Gaussian mechanism solves ERM in \eqref{eq:erm} by performing the following update   with the released gradient $g_t$ at the $t$-th iteration: 
\begin{align}
\text{DP-SGD: }\theta_{t+1} = \theta_t - \eta_t g_t; g_t  = \nabla f_{\xi_t}(\theta_t) + Z  ,  \label{eq:dp-sgd1}
\end{align} 
where $Z \sim \mathcal{N}(0, \sigma^2 I)$, $\xi_t \sim\text{Uniform}(\{1,2,\ldots,n\})$, and $\eta_t>0$ is the stepsize or learning rate. Choosing the appropriate stepsize $\eta_t$ is challenging in practice, as $\eta_t$ depends on the unknown Lipschitz parameter of the gradient $\nabla f(\theta;x_i)$ \cite{ghadimi2013stochastic}. %estimating $L$ is an active research area could be NP hard \citep{NEURIPS2020_5227fa9a,scaman2018lipschitz}. 
% In practice, $\eta_t$ is chosen by grid search which not only requires an enormous amount of engineering effort but also costs extra privacy budgets. 
Recent popular techniques for tuning $\eta_t$ include adaptive gradient methods \cite{duchi2011adaptive} and decaying stepsize schedules \cite{goyal2017accurate}.
When applying non-constant stepsizes, most of the existing differentially private algorithms 
% are often naive since they
directly follow the standard DP-SGD strategy by 
adding 
a simple perturbation
(i.e, $Z \sim \mathcal{N}(0, \sigma^2 I)$) to each gradient over the entire sequence of iterations \citep{zhou2020private}. This results in a uniformly-distributed privacy budget for each iteration \citep{bassily2014private}.

Several theoretical, as well as experimental results, corroborate the validity of the DP-SGD method with a uniformly-distributed privacy budget \citep{bu2020deep,zhou2020towards,zhou2020private}. Indeed, using a constant perturbation intuitively makes sense after noticing that the update in \eqref{eq:dp-sgd1} is equivalent to $ \theta_{t+1} = \theta_t - \eta_t \nabla f(\theta_t;x_{\xi_t}) -\eta_tZ$. This implies that the size of the true perturbation (i.e., $\eta_tZ$) added to the updated parameters is controlled by  $\eta_t$. The decaying learning rate thus diminishes the true perturbation added to $\theta_t$. Although the DP-SGD method with decaying noise $\eta_tZ$ is reasonable, prior to this paper \emph{it was unknown whether this is the optimal strategy using the utility measure}.% 

\begin{figure}[tb]
\includegraphics[width=.492\linewidth]{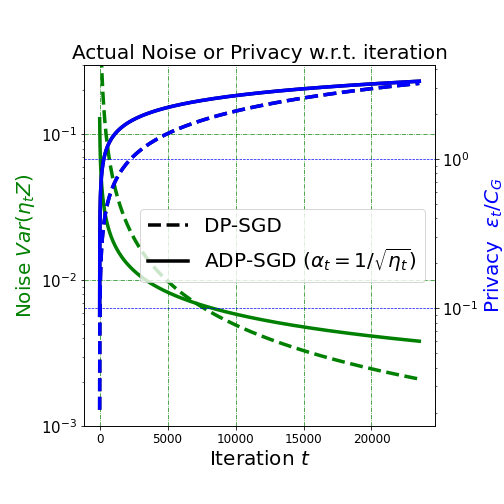}
\includegraphics[width=.492\linewidth]{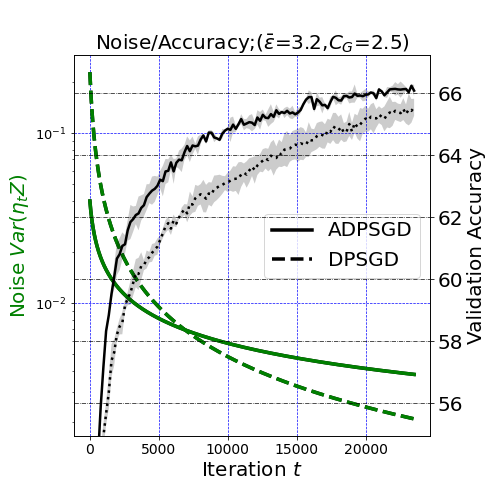}
\caption{\small \textbf{Comparison between $\alpha_t=1$ and  $\alpha_t=1/\sqrt{\eta_t}$  in \eqref{eq:ada-sgd2}}. Set the stepsize  $\eta_t=1/\sqrt{20+t}$ and the same privacy budget at final iteration.  The green curves in the left and right plots are the same; they correspond to the left vertical y-axis illustrating the actual Gaussian noise (i.e., $\eta_t\alpha_tZ$) added to the parameter $\theta_t$ for $\alpha_t=1$ (dash line) and $\alpha_t=1/\sqrt{\eta_t}$ (solid line). The blue (black) curves on the left (right) plot  corresponding to the right vertical y-axis show the overall privacy (validation accuracy) for $\alpha_t=1$ (DP-SGD), represented by the dashed line, and $\alpha_t=1/\sqrt{\eta_t}$ (ADP-SGD), represented by  the  solid line. 
The variance of the perturbation using our proposed ADP-SGD decreases more slowly than that using DP-SGD, and so  spreads across to entire optimization process more evenly than DP-SGD. 
Note that the privacy value $\bar{\varepsilon}=\varepsilon/16=3.2$ is based on the theoretical upper bound. The validation accuracy (black curves) is for CIFAR10 dataset with the gradient clipping $C_G=2.5$  comparable to $G$ (see detailed explanation in \cref{sec:exp}).
}
\label{fig:compare} 
\end{figure}
% To answer the above question, we propose to 
To study the above question, we propose adding a hyperparameter $\alpha_t>0$ to the private mechanism:
% to  \eqref{eq:dp-sgd1} 
% \todo{somewhat unclear. Can you be more precise? e.g. could we just say here that the released gradient is $g_t  = \nabla f(\theta_t;x_{\xi_t}) + \alpha_t Z$?}
\begin{align}
 \text{{ADP-SGD}: } \theta_{t+1} =  \theta_t - \eta_t g_t; g_t  = \nabla f_{\xi_t}(\theta_t) +\eta_t\alpha_t Z.  \label{eq:ada-sgd2}
\end{align} 
% \begin{align}
% \text{DP-SGD: }\theta_{t+1} = \theta_t - \eta_t g_t; g_t  = \nabla f_{\xi_t}(\theta_t) + Z  ,  \label{eq:dp-sgd1}
% \end{align} 
The role of the hyperparameter $\alpha_t$ is to adjust the variance of the added random noise given the stepsize $\eta_t$. It is thus natural to add ``adaptive'' in front of the name DP-SGD and call our proposed algorithm ADP-SGD. 
% \me{add some explanation o} We note that this 
To establish the privacy and utility guarantees of this new method, we first extend the advanced composition theorem \citep{dwork2014algorithmic} so that it treats the case of a non-uniformly distributed privacy budget.
We then show that our method achieves an improved utility guarantee when choosing $\alpha_t=1/\sqrt{\eta_t}$, compared to the standard method using uniformly-distributed privacy budget, which corresponds to $\alpha_t =1$.% \todo{i.e., $\alpha_t =1$}.
% \irina{changed "adding" with "added"}\todo{previous phrase unclear. I suggest "is to adjust the variance of the noise"}
% particularly adaptive to the choice of stepsize $\eta_t$ 
% \irina{end the phrase and start a new phrase: "It is thus natural to add "adaptive" in front of..."}
% We then derive the best choice of $\alpha_t=1/\sqrt{\eta_t}$ such that the utility or convergence bound of the DP algorithm is the tightest. \lingxiao{I rewrite this sentence as above, please check.}
This relationship between $\alpha_t$ and $\eta_t$ is surprising.  Given the same privacy budget and the decaying stepsize $\eta_t<1$, the best choice -- $\alpha_t=1/\sqrt{\eta_t}$ -- results in $\theta_{t+1} = \theta_t - \eta_t \nabla f(\theta_t;x_{\xi_t}) -\sqrt{\eta_t}Z$. This implies that the actual Gaussian noise $\sqrt{\eta_t}Z$ of ADP-SGD decreases \emph{more slowly} than that of the conventional DP-SGD (i.e., $\eta_t Z$). To some extent, this is counter-intuitive in terms of convergence: one may anticipate that a more accurate  gradient or smaller perturbation will be necessary as the parameter $\theta_t$ reaches a stationary point (i.e., as $\|\nabla F(\theta_t)\|\to 0$) \cite{lee2018concentrated}. See \cref{fig:compare} for an illustration. We will explain how this interesting finding is derived in Section~\ref{sec:adp-sgd}.  
% \todo{I don't think Section 4 really explains this. That is, Section 4 presents bounds, but it doesn't explain \textit{why} the intution is wrong or why it intuitively makes sense to have larger variance at later iterations.}

%, given the restriction on the privacy guarantee,
%%%%%%%%%%%%%%%%%%%%%%%%%%% solution %%%%%%%%%%%%%%%%%%%%%%%%%%%% 
 %on Gaussian variance  $c_t^{real}$ for the DP-SGD algorithm,

% Thus, the variance of $c_t^{real}$ in \eqref{eq:dp-sgd2} becomes $\sigma_t^2 =  \eta_t^2\alpha_t^2\sigma^2$.
% In this paper, we investigate the relationship between the the Gaussian variance of $\eta_tZ$ and the gradient update term $ \eta_t \nabla F(\theta_t;x_{\xi_t})$ given in \eqref{eq:dp-sgd1} to see if one could obtain a tighter utility bound for the $(\varepsilon,\delta)$-DP SGD and in the context of decaying or adaptive stepsize $\eta_t$.
% %%\vspace{-1.5cm}
\paragraph{Contribution.} Our contributions include:
\begin{squishlist}
    \item We propose an adaptive (stochastic) gradient perturbation method -- ``Adaptive Differentially Private Stochastic Gradient Descent" (ADP-SGD) (\cref{alg:privacy-general} or \eqref{eq:ada-sgd2}) -- and show how it can be used to perform differentially private empirical risk minimization. We show that APD-SGD provides a solution to the core question of this paper: \emph{given the same overall privacy budget and iteration complexity, how should we select the gradient perturbation adaptively - across the entire SGD optimization process -  to achieve better utility guarantees?} %\lingxiao{should we use convergence or utility?}
    To answer this, we establish the privacy guarantee of ADP-SGD (\cref{thm:privacy}) and find that the best choice of $\alpha_t$ follows an interesting dynamic: $\alpha_t=1/\sqrt{\eta_t}$
    % \todo{(I removed Cf's. First, the punctution was wrong. Second, cf mean "see for comparison or contrast", where as you want "see"} 
    (\cref{thm:optimal-bound}). Compared to the conventional DP-SGD, ADP-SGD with $\alpha_t=1/\sqrt{\eta_t}$ results in a better utility given the same privacy budget $\varepsilon$ and complexity $T$. 
    \vspace{0.1cm}
    
\item As the ADP-SGD method can be applied using any generic  $\eta_t$, we discuss the two widely-used stepsize schedules: (1) the polynomially decaying stepsize of the form $\eta_t=1/\sqrt{1+t}$, and (2) $\eta_t$ updated by the gradients \cite{duchi2011adaptive}. When using $\eta_t=1/\sqrt{1+t}$, given the same privacy budgets $\varepsilon$, we obtain a stochastic sequence
 $\{\theta_t^{\rm ADP}\}$ for ADP-SGD with $\alpha_t=1/\sqrt{\eta_t}$, and  $\{\theta_t^{\rm DP}\}$ for standard DP-SGD. We have the utility guarantees of the two methods, respectively\footnote{This is an informal statement of \cref{thm:theorem-adp-descrease}; the order $\widetilde{\mathcal{O}}$  hides $\log(1/\delta)$, $L G^2$ and $F(\theta_0)-F^*$ terms. We keep the iteration number $T$ in our results since the theoretical best value of $T$ depends on some unknown parameters such as the Lipschitz parameter of the gradient, which we try to tackle using non-constant stepsizs.
% The bounds is not optimized with respect to $T$. This is due to the optimal expression of $T$ depends on the known parameter $L$. We emphasize on unknown $L$ and so decide to keep bound in terms of $T$.
}
% {\begin{footnotesize}
\begin{align*}
     &\mathbb{E} [\|\nabla F(\theta_\tau^{\rm ADP})\|^2]
    =  \widetilde{\mathcal{O}}\left(\frac{\log(T)}{\sqrt{T}} + \frac{ d \sqrt{T}}{  n^2 \varepsilon^2} \right); \quad \quad \quad
    &\mathbb{E} [\|\nabla F(\theta_\tau^{\rm DP})\|^2]
    =  \widetilde{\mathcal{O}}\left(\frac{\log(T)}{\sqrt{T}} + \frac{ d \log(T)\sqrt{T}}{  n^2 \varepsilon^2} \right)
\end{align*} 
% \end{footnotesize} }

where $\tau := \text{arg}\min_{k\in{[T-1]}}  \mathbb{E} [\|\nabla F(\theta_k)\|^2]$. Compared to the standard DP-SGD, ADP-SGD with $\alpha_t=1/\sqrt{\eta_t}$ improves the bound by a factor of $\mathcal{O}(\log(T))$ when $T$ and $d$ are large (i.e. high-dimensional settings). When $\eta_t$ is updated by the gradients \cite{duchi2011adaptive}, the same result holds. See Section \ref{sec:sec5} for the detailed discussion. % Although this $\log(T)$ improvement can also be achieved  by using moments accountant \cite{Mironov2019RnyiDP}, \emph{we emphasize that our method is orthogonal and complementary to moments accountant.} %\lingxiao
% It is true that both our method and moment accountant can achieve a $\log (T)$ improvement individually. Nevertheless, since these two techniques are complementary to each other, we can apply them simultaneously, and achieve $\log^2(T)$ improvement over DP-SGD using advanced composition for $O(1/\sqrt{t})$ stepsizes, compared to the $\log(T)$ improvement using only one of them. Thus, our method serves as a unique technique complementary to the moment accountant.
  %Note that the $\log(T)$ improvement is particularly for the stepsize $\eta_t=1/\sqrt{1+t}$.
  %%\vspace{0.1cm}
% \todo{hard to read, as the two equation LHSs are the same. consider
% \begin{footnotesize}
% \begin{align*}
% \mathbb{E} [\|\nabla F(\theta_\tau^{\rm ADP})\|^2]
%     =  \widetilde{\mathcal{O}}\left(\frac{\log(T)}{\sqrt{T}} + \frac{ d \sqrt{T}}{  n^2 \varepsilon^2} \right); \quad    \mathbb{E} [\|\nabla F(\theta_\tau^{\rm DP})\|^2]
%     =  \widetilde{\mathcal{O}}\left(\frac{\log(T)}{\sqrt{T}} + \frac{ d \log(T)\sqrt{T}}{  n^2 \varepsilon^2} \right).
% \end{align*}
% \end{footnotesize}
% }
% %%\vspace{-0.15cm}
 \vspace{0.1cm}
    \item Finally, we perform numerical experiments to systematically compare the two algorithms: ADP-SGD ($\alpha_t=1/\sqrt{\eta_t}$) and DP-SGD. In particular, we verify that ADP-SGD with $\alpha_t=1/\sqrt{\eta_t}$ consistently outperforms DP-SGD when $d$ and $T$ are large.  Based on these theoretical bounds and supporting numerical evidence, we believe ADP-SGD has important advantages over past work on differentially private empirical risk minimization. %\irina{has important advantages over previously proposed differentially private empirical risk minimization problems?}.
\end{squishlist}
%%\vspace{-0.15cm}
\paragraph{Notation.} In the paper, $[N] := \{ 0,1,2, \dots, N\}$ and $\{\cdot\}: = \{\cdot\}_{t=1}^T$.  We write $\| \cdot \|$ for the $\ell_2$-norm. $F^*$ is a global minimum of $F$ assuming $F^*>0$. We use $D_F:=F(\theta_0)-F^*$ and set stepsize $\eta_t=\eta/b_{t+1}$.

\section{Preliminaries}\label{sec:privacy}
%\vspace{-0.15cm}

% We provide several definitions, assumptions, and lemmas that will be used in our theoretical analysis.
We first make the following assumptions for the objective loss function in \eqref{eq:erm}.
\begin{asmp}\label{asmp:l-lip}
Each component function $f(\cdot)$ in \eqref{eq:erm} has $L$-Lipschitz gradient, i.e., 
\begin{equation}\label{eq:L-smooth}
\| \nabla f(x) - \nabla f(y) \| \leq L \| x - y \|, \quad \forall x, y \in \mathbb{R}^d.
\end{equation}
\end{asmp}

% \begin{defn}\label{def:lip}
% A function $F:\mathbb{R}^d \rightarrow \mathbb{R}$ has $L$-Lipschitz gradient if
% \begin{equation}\label{L-smooth}
% \| \nabla F(x) - \nabla F(y) \| \leq L \| x - y \|, \quad \forall x, y \in \mathbb{R}^d.
% \end{equation}
% \end{defn}
% A function $F:\mathbb{R}^d \rightarrow \mathbb{R}$ has $L$-Lipschitz gradient if
% \begin{equation}
% \| \nabla F(x) - \nabla F(y) \| \leq L \| x - y \|, \quad \forall x, y \in \mathbb{R}^d
% \label{L-smooth}
% \end{equation}
% In the following discussion, we write $F \in \mathbb{C}_L^1$ and refer to $L$ as the smoothness constant for $F$ if $L > 0$ is the smallest number such that the inequality in \eqref{L-smooth} holds. 
% \todo{is "the above" \eqref{L-smooth}?}
\begin{asmp}\label{asmp:G-bound}%{asmp:l-lip}
Each component function $f(\cdot)$ in \eqref{eq:erm} has bounded gradient, i.e., 
\begin{equation}\label{eq:G-bound}
\|\nabla f(x)\| \leq G, \quad \forall x\in \mathbb{R}^d.
\end{equation}
\end{asmp}
%\vspace{-0.25cm}
The bounded gradient assumption is a common assumption for the analysis of DP-SGD algorithms \citep{wang2017differentially,zhou2020private,zhou2020towards} and also frequently used in general adaptive gradient methods such as Adam \citep{reddi2021adaptive,chen2018convergence,j.2018on}. One recent popular approach to relax this assumption is using the gradient clipping method \citep{chen2020understanding,andrew2019differentially,pichapati2019adaclip}, which we will discuss more in Section \ref{sec:exp} as well as in Appendix \ref{sec:append-basic}. Nonetheless, this assumption would serve as a good starting point to analyze our proposed method. Next, we introduce differential privacy \citep{dwork2006calibrating}.% \begin{defn}
% A function $F:\mathbb{R}^d \rightarrow \mathbb{R}$ is $G$-Lipschitz if 
% \begin{equation}\label{L-lip}
% \| F(x) - F(y) \| \leq G \| x - y \|, \quad \forall x, y \in \mathbb{R}^d.
% \end{equation}
% \end{defn}
% In our paper, we assume each component function $f(\cdot)$ in \eqref{eq:erm} is $G$-Lipschitz and has $L$-Lipschitz gradient, i.e., $f \in \mathbb{C}_L^1$.
% In addition, we assume each component function $f(\cdot)$ in \eqref{eq:erm} is $G$-Lipschitz, i.e.,  $\|\nabla f(x)\|\leq G$ for all $x$.
\begin{defn}[\textbf{$(\varepsilon,\delta)$-DP}]\label{def:dp}
A randomized mechanism
$\mathcal{M}: \mathcal{D} \rightarrow \mathcal{R}$ with domain $\mathcal{D}$
and range $\mathcal{R}$ is $(\varepsilon,\delta)$-differentially private if for 
any two adjacent datasets $\mathcal{D}, \mathcal{D}'$ differing in one sample, and for any subset of outputs $S \subseteq \mathcal{R}$, we have
\[
\text{Pr} [\mathcal{M}(D) \in S ] \le e^{\varepsilon} \text{Pr} [\mathcal{M}(D') \in S] + \delta.
\]
\end{defn}
% $b_{t+1}=1$
\begin{lemma} [\textbf{Gaussian Mechanism}]\label{def:gaussian} 
For a given function $h: \mathcal{D} \to \mathbb{R}^d$, the Gaussian mechanism
$\mathcal{M}(\mathcal{D}) = h(\mathcal{D}) + Z$ with $Z\sim\mathcal{N}(0, \sigma^2I_d
)$ satisfies $(\sqrt{2\log(1.25/\delta)}\Delta/\sigma,\delta)$-DP, where $\Delta = \sup_{\mathcal{D},\mathcal{D}'} \|h(\mathcal{D})-h(\mathcal{D}') \|$, $\mathcal{D}, \mathcal{D}'$ are two adjacent datasets, and $\varepsilon,\delta>0$.
 \end{lemma}
%  \lingxiao{need a complete version for Gaussian mechanism}
To achieve differential privacy, we can use the above Gaussian mechanism \citep{dwork2014algorithmic}. In our paper, we consider iterative differentially private algorithms, which prompts us to use privacy composition results  to establish the algorithms' privacy guarantees after the completion of the final iteration. To this end, we extend the advanced composition theorem \citep{dwork2014algorithmic} to the case in which each mechanism $\mathcal{M}_i$ has its own specific $\varepsilon_i$ and $\delta_i$ parameters. %} Thus, we need the composition results to establish the privacy guarantees of a combination of Gaussian mechanisms
% we establish the following extended advanced composition lemma, which can characterize the privacy loss of a sequence of mechanisms with different privacy budgets.
% Then, we establish the privacy guarantee for adaptive differentially private stochastic gradient descent (ADP-SGD) algorithm.
\begin{lemma}[\textbf{Extended Advanced Composition}]
\label{thm:advanced-comp}
Consider two sequences $\{\varepsilon_i\}_{i=1}^k,\{\delta_i\}_{i=1}^k$ of positive numbers satisfying  $\varepsilon_i\in(0,1)$ and  $\delta_i\in(0,1)$. 
%  \todo{Is there an implicit constraint that since $\tilde \delta > 0$, we need $\delta' < 1-(1-\delta_1)\cdots (1-\delta_k)$?}
 Let  $\mathcal{M}_i$ be $(\varepsilon_i, \delta_i)$-differentially private for all $i \in \{1, 2, \hdots, k\}$.
 %and $\mathcal{M}_i$'s are potentially chosen adaptively.
 Then $\mathcal{M}=(\mathcal{M}_1, \hdots, \mathcal{M}_k)$ is $(\tilde{\varepsilon}, \tilde{\delta})$-differentially private for $\delta'\in(0,1)$ and

 \begin{align*}
   \tilde{\varepsilon}& = \sqrt{\sum_{i=1}^k 2\varepsilon_i^2 \log\left(\frac{1}{\delta'}\right)} + \sum_{i=1}^k \frac{\varepsilon_i(e^{\varepsilon_i}-1)}{(e^{\varepsilon_i}+1)},  \quad
   \tilde{\delta} = 1 - (1- \delta_1)(1-\delta_2) \hdots (1-\delta_k) + \delta'.
\end{align*}

\end{lemma}
% \todo{with this phrasing, it seems like the mechanisms are potentially adaptive but the $\varepsilon$s and $\delta$s are arbitrary. In fact, you're choosing the $\varepsilon$s adaptive, right? I don't think this makes a difference in the conclusions, but consider different phrasing here}
% \todo{I think some clear definitions are in order. What does it mean for privacy to be adaptive? Or do you mean differential privacy for adaptive methods? To me that's very different than "adaptive privacy"}

When $\varepsilon_i=\varepsilon_0$ and $\delta_i=\delta_0$ for all $i$, \cref{thm:advanced-comp} reduces to the classical advanced composition theorem \citep{dwork2014algorithmic} restated in  \cref{lem:advanced} in the Appendix.  
% Using the above theorem and \cref{prop:mechasim}, we have following privacy guarantee for Algorithm
 
% \todo{something is missing here. What is the function being optimized? what is the precise privacy goal? How does this relate to Defn 3.1? Is private SGD the mechanism? Are the $\theta$'s the range ${\cal R}$?}
%\vspace{-0.15cm}
\section{The ADP-SGD algorithm}
%\vspace{-0.15cm}
% \lingxiao{need to rewrite this section}

In this section, we present our proposed algorithm: \emph{adaptive differentially private stochastic gradient descent} (ADP-SGD, \cref{alg:privacy-general}). The ``adaptive" part of the algorithm is tightly connected with the choice of the hyper-parameter $\alpha_t$ (see line 5 of Algorithm~\ref{alg:privacy-general}).  For $\alpha_t = 1$, ADP-SGD reduces to DP-SGD.
% For the function $\phi_2:\mathbb{R}^{2}\to \mathbb{R}$ 
% in line 6 of \cref{alg:privacy-general}, 
As mentioned before, we aim to investigate whether an uneven allocation of the privacy budget for each iteration (via ADP-SGD) will provide a better utility guarantee than the default DP-SGD given the same privacy budget. To achieve this, our proposed ADP-SGD with hyper-parameter $\alpha_{t}$ adjusts the privacy budget consumed at the $t$-th iteration according to the current learning rate $\eta/b_{t+1}$ (see line 6 of Algorithm~\ref{alg:privacy-general}). Moreover, we will update $\alpha_{t}$ dynamically (see line 5 of Algorithm~\ref{alg:privacy-general}) and show how to choose $\alpha_{t}$ in \cref{sec:adp-sgd}. Before proceeding to analyze \cref{alg:privacy-general}, we state \cref{prop:mechasim} to clearly explain  the adaptive privacy mechanism for the algorithm. %\lingxiao{should we change the name related to private mechanism?}
% which is easily verified by checking that line 6 by ignoring  ${\eta}/{b_{t+1}}$.
% which is updated dynamically (see line 5) and controls the magnitude of the random perturbation ($\alpha_{t+1} c_t$ in line 6).
%, is purposely adapt to the previous iterations, their associated mechanisms and, more importantly, the choice of stepsize $\eta/b_t$ which is the focus in this paper.\me{need to rewrite...}

% where the random noise added at each iteration to achieve differential privacy aims to optimally adapt to the learning rate.
% As a result, the proposed method allows us to find a sequence of $\alpha_t$ for an improved utility guarantee compared to the method with evenly splitting privacy budgets across all iterations.
% that we call \cref{alg:privacy-general} an {adaptive differentially private algorithm} when $\alpha_t\neq1$. We aim

% such that the random noise is optimally adaptive to the current learning rate 

% \lingxiao{we may then discuss the example of sgd.}

% Compared to the standard differentially private SGD, we add an extra parameter $\alpha_t$ to make the privacy mechanism more flexible. We will show later in \cref{sec:adp-sgd} that there exists a better choice of $\alpha_t$ than a simple constant $\alpha_t=1$. 
%\vspace{-0.1cm}
 \begin{defn}[\textbf{Adaptive Gaussian Mechanism}] \label{prop:mechasim}
At iteration $t$ in \cref{alg:privacy-general}, the privacy mechanism $\mathcal{M}_t:$ $\mathbb{R}^d\to\mathbb{R}^d$ is:
\[\mathcal{M}_t(X)=\nabla f(\theta_t;x_{\xi_t})+ \alpha_{t+1}c_t.\]
The hyper-parameter $\alpha_{t+1}$ is \emph{adaptive} to the DP-SGD algorithm specifically to the stepsize $\eta_{t}:=\eta/b_{t}$. \end{defn}

\begin{algorithm}[tb]
\caption{\textbf{ADP-SGD} (\textbf{DP-SGD} if $\alpha_t=1$)}
\label{alg:privacy-general}
\begin{algorithmic}[1]
	    \State Input: $\theta_0,b_0,\alpha_0$ and $\eta>0$
	    \For {$ t=0,1,\ldots,T-1$} 
	      \State $\xi_t\sim \text{Uniform}(1,...,n)$  and $c_t\sim \mathcal{N}(0,\sigma^2 I)$
	    \State update $b_{t+1}=\phi_1(b_t,\nabla f(\theta_{t};x_{\xi_t}))$  
	   \State update $\alpha_{t+1}=\phi_2(\alpha_t, b_{t+1})$ 
	   %\todo{vague and confusing; what are update equations? this is your main algorithm; it needs to be clear.}
	    \State \textbf{release}  $g^{b}_t =  \frac{\eta}{b_{t+1}}(\nabla f(\theta_{t};x_{\xi_t})+\alpha_{t+1}c_t)$    \label{ln:release} 
	   % \State $g_{t}= ,$ where
	    \State update $\theta_{t+1}=\theta_{t}-  g^{b}_t$
	    \EndFor
\end{algorithmic}
\end{algorithm}
\cref{alg:privacy-general} is a general framework that can cover many variants of stepsize update schedules, including the adaptive gradient algorithms \citep{duchi2011adaptive,kingma2014adam}. 
% \becca{it would be helpful to say something here about the connection to equation 3.} 
Rewriting $\eta_t=\eta/b_{t+1}$ in \cref{alg:privacy-general} is equivalent to \eqref{eq:ada-sgd2}.
In particular,  we use functions $\phi_1:\mathbb{R}^{2}\to \mathbb{R}$ and $\phi_2:\mathbb{R}^{2}\to \mathbb{R}$ to denote
the updating rules for parameters $b_t$ and $\alpha_t$, respectively. For example,  when $\phi_1$ is $1/\sqrt{a+ct}$, $\phi_2$ is the constant $1$ for all $t$ and $a,c>0$,  ADP-SGD reduces to DP-SGD with polynomial decaying stepsizes \citep{bassily2014private}.  When $\phi_1$ is $b_{t+1} = \sqrt{b_t^2 +\|\nabla f(\theta_{t};x_{\xi_t})\|^2}$ and $\phi_2$ is the constant 1, the algorithm reduces to DP-SGD with a variant of adaptive stepsizes \citep{duchi2011adaptive}.  In particular, if we choose $\phi_2$ to be 0, the algorithm reduces to the vanilla SGD.

Similar to classical works on the convergence of the SGD algorithm  \citep{bassily2020stability,bottou2018optimization,ward2019adagrad}, we will use %standard \irina{is "de facto" and "standard" a repetition?}
Assumption~\ref{asmp:sgd} in addition to Assumption~\ref{asmp:G-bound} and  Assumption~\ref{asmp:l-lip}.
\begin{asmp}\label{asmp:sgd}%{asmp:l-lip}
$\nabla f(\theta_{t};x_{\xi_t})$ is an unbiased estimator of $\nabla F(\theta_k)$. %\footnote{$\mathbb{E}_{\xi_t} \left[\nabla f(\theta_{t};x_{\xi_t})\right]= \nabla F(\theta_t)$ where $\mathbb{E}_{\xi_t} \left[\cdot \right]$ is the expectation w.r.t. $\xi_t$ conditional on previous $\{\xi_{i}\}_{i=0}^{t-1}$} 
 The random indices $\xi_t$, $t = 0,1,2, \dots, $ are independent of each other and also independent of $\theta_t$ and $c_1,\ldots, c_{t-1}$.
\end{asmp}

Having defined the ADP-SGD algorithm and established our assumptions, in what follows, we will be answering the paper's central question: \emph{Given the same privacy budget $\varepsilon$, how should one design the gradient perturbation parameters $\alpha_t$ adaptively for each iteration $t$  to achieve a better utility guarantee?}  %\irina{across the entire sgd optimization process sounds a bit confusing; can we say: for each iteration t of ADP-SGD (not sgd) }
% Solving this question is of paramount importance as one can only run these algorithms for a finite number of iterations and any improvement in the constants of the utility bound matters a lot % \irina{the part after "and" should be re-written...how about "Solving this question is of paramount importance as one can only run these algorithms for a finite number of iterations, therefore, given these constraints, a clear and efficient strategy for improving the constants of the utility bound is necessary }. 
Solving this question is of paramount importance as one can only run these algorithms for a finite number of iterations. Therefore, given these constraints, a clear and efficient strategy for improving the constants of the utility bound is necessary.
% \lingxiao{duplicate?}
% \todo{before proceeding, maybe make a table summarizing results. e.g. \cref{tab:summary}. 
% \begin{table}[h]
%     \centering
%     \begin{tabular}{c||c|c|c}
%     & $b_t = 1$  &$b_t = (a+ct)^{1/2}$  &$b_t^2 = b_{t-1}^2+\max \{ \|\nabla f(\theta_{t-1})\|, c_b \}$ \\ \hline \hline 
%      $\alpha_t = 1$ & $\frac{dT}{n^2 \varepsilon^2}$& $\frac{d \sqrt{T} \log(T)}{n^2 \varepsilon^2}$ \pcref{thm:theorem-adp-descrease}
%      & \\ \hline 
%      $\alpha_t = b_t^{1/2}$  & $\frac{dT}{n^2 \varepsilon^2}$ & $\frac{d \sqrt{T}}{n^2 \varepsilon^2}$ \pcref{thm:theorem-adp-descrease}& 
%     \end{tabular}
%     \caption{Bound on $\min_{k \in [T-1]} \mathbb{E}\|\nabla F(\theta_k)\|^2$ in the high-dimensional setting with large $d$, up to a constant factor depending on $\eta, L, G$, and $\delta$}
%     \label{tab:summary}
% \end{table}
% }
\vspace{-0.15cm}
\section{Theoretical results for ADP-SGD} 
\vspace{-0.15cm}
\label{sec:adp-sgd}
In this section, we provide the main results for our  method -- the privacy and utility guarantees.
\begin{thm}[\textbf{Privacy Guarantee}] \label{thm:privacy}
Suppose the sequence $\{\alpha_t\}_{t=1}^T$ is known in advance and that Assumption \ref{asmp:G-bound} holds. \cref{alg:privacy-general}  satisfies $(\varepsilon,\delta)$-DP if the random noise $c_t$ has variance

 \begin{align}
 \sigma^2  &= \frac{(16G)^2B_{\delta}}{n^2  \varepsilon^2  } \; {\sum_{t=0}^{T-1}  \frac{ 1}{\alpha^2_{t+1}}  }\text{ with }  B_\delta=\log\left(\frac{16T}{n\delta}\right)\log\left(\frac{1.25}{\delta}\right).  \label{eq:key-gen}\vspace{-0.6cm}
\end{align}

\end{thm}
The theorem is proved by using \cref{thm:advanced-comp} and \cref{prop:mechasim} (see \cref{sec:proof-privacy} for details). 
 Note that the term $B_{\delta}$ could be improved by using the moments accountant method \citep{Mironov2019RnyiDP}, to $\mathcal{O}(\log(1.25/\delta))$ independent of $T$ but with some additional constraints \citep{abadi2016deep}. We keep this format of $B_\delta$ as in \eqref{eq:key-gen} in order to compare directly with \citep{bassily2014private}.

% at iteration $t$ the algorithm is $(\varepsilon_t, \delta_0)$-DP for some $\delta_0>0$ if $\varepsilon_t = { 2G\sqrt{2\ln(1.25/\delta_0)}}/({ n\sigma\alpha_{t+1}})$. see 
% Then, the lemma is obtained by the extended advanced theorem 
%  \irina{The presentation seems a bit circular to me; Can we say, instead, that at each iteration t, $\epsilon_t = \hdots$ , which by Extended Advanced Composition, gives final $\epsilon = \hdots$ (expression in terms of sigma)}\lingxiao{Can we just say: , and then give a simple explanation of $\alpha_t$?}
\cref{thm:privacy} shows that  $\sigma^2$ must scale with $\sum_{t=1}^{T}1/\alpha^{2}_{t}$.  When the complexity $T$ increases, the variance $\sigma^2$, regarded as a function of $T$, could be either large or small, depending on the sequence $\{\alpha_t\}$. 
% If $\alpha_t$ is monotone with rate $\alpha_t^2 \propto t^p$, where $p\in[0,1]$, then
\begin{equation*} \vspace{-0.15cm}\text{If }  \alpha_t^2 \propto t^p, p\in[0,1], \quad \text{then }
\textstyle \sigma^2 \propto   \begin{cases} T^{1-p} &0\leq p<1\\
  \log(T) & p = 1
    \end{cases}
\end{equation*} and $p=0$ is the default DP-SGD.
From a convergence view, $ \theta_{t+1} = \theta_t - \eta_t \nabla f(\theta_t;x_{\xi_t}) -\eta_t \alpha_t Z$ implies that the actual Gaussian noise added to the updated parameter $\theta_t$ has variance $\eta_t^2 \alpha_t^2\sigma^2$. Therefore, it is subtle to determine what $p$ would be the best choice for ensuring convergence. In \cref{thm:optimal-bound}, we will see that the optimal choice of the sequence $\{\alpha_t\}_{t=1}^T$ is closely related to the stepsize. % \lingxiao{shoud we specify $\eta_t=\eta/b_{t}$?}
\begin{thm}
[\textbf{Convergence for ADP-SGD}]\label{thm:optimal-bound} Suppose we choose $\sigma^2$
% \todo{$\sigma^2$?}
- the variance of the random noise in \cref{alg:privacy-general} - according to \eqref{eq:key-gen} in \cref{thm:privacy} and that Assumption \ref{asmp:l-lip}, \ref{asmp:G-bound} and \ref{asmp:sgd} hold. Furthermore, suppose $\alpha_t, b_t$ are deterministic. The utility guarantee of \cref{alg:privacy-general} with $\tau \deq \text{arg}\min_{k\in{[T-1]}}   \mathbb{E} [\|\nabla F(\theta_k)\|^2]$ and $B_{\delta}=\log(16T/(n\delta))\log(1.25/\delta)$ is  
\begin{align}\label{eq:utility-general}
 \mathbb{E} \|\nabla F(\theta_\tau)\|^2
 \leq   \frac{ 1}{\sum_{t=0}^{T-1}b_{t+1}}  \left(W_{opt}+ \frac{ d (16G)^2 B_{\delta} }{2n^2\varepsilon^2 }  M(\{\alpha_t\},\{b_t\}) \right)
\end{align}
where $W_{opt}:=\frac{D_F}{\eta} + \frac{\eta L}{2}\sum_{t=0}^{T-1} \frac{\mathbb{E}\left[ \|\nabla f(\theta_t,{\xi_t}) \|^2\right] }{b^2_{t+1}}$ and 
$M(\{\alpha_t\},\{b_t\}) {\deq }\textstyle \sum_{t=1}^{T} (\alpha_{t}/b_{t})^2 \sum_{t=1}^{T}  1/\alpha^2_{t}.$ 
\vspace{-.12cm}
 \end{thm}

Although the theorem assumes independence between $b_{t+1}$ and the stochastic gradient $\nabla f(\theta_t;x_{\xi_t})$, we shall see in \cref{sec:adaptive} that a similar bound holds for correlated $b_t$ and $\nabla f(\theta_t;x_{\xi_t})$.
\begin{remark}[\textbf{An optimal relationship between $\alpha_t$ and $b_t$}]\label{remark:optimal}
According to \eqref{eq:utility-general}, the utility guarantee of \cref{alg:privacy-general} consists of two terms. The first term ($W_{opt}$) corresponds to the optimization error and the last term ($\frac{ d (16G)^2 B_{\delta} }{2n^2\varepsilon^2}  M(\{\alpha_t\},\{b_t\})$) is introduced by the privacy mechanism, which is also the dominating term. Note 
that if we fix $\{b_t\}$ and minimize $M$ with respect to $\{\alpha_t\}$, the minimal value denoted by
$ M_{\rm adp}$ expresses as

\begin{align}\label{eq:M-adp}
\vspace{-0.2cm}
\min_{\{\alpha_t\}}M(\{\alpha_t\},\{b_t\}) = M_{\rm adp} \deq  \big(\textstyle\sum_{t=0}^{T-1} 1/b_{t+1}\big)^2. 
\vspace{-0.2cm}
\end{align} 
 Furthermore,   $M(\{\alpha_t\},\{b_t\})= M_{\rm adp}$ if $\alpha^2_t = b_t.$ Therefore, if we choose $\alpha_t,b_t$ such that the relationship of $\alpha^2_t = b_t$ holds, we can achieve the minimum utility guarantee for  \cref{alg:privacy-general}.
\end{remark}

% Suppose the last term  in \eqref{eq:utility-general}, i.e., $\frac{ d (16G)^2 B_{\delta} }{2n^2\varepsilon^2 }  M(\{\alpha_t\},\{b_t\})$, is much larger than the term $W_{opt}$. 
Based on the utility bound in \cref{thm:optimal-bound}, we now compare the strategies between using the arbitrary setting of $\{\alpha_t\}$ and  the optimal setting $\alpha^2_t = b_t$ by examining the ratio $M(\{\alpha_t\},\{b_t\})/M_{\rm adp}$; a large value of this ratio implies a significant reduction in the utility bound is achieved by using \cref{alg:privacy-general} with $\alpha_t=\sqrt{b_t}$. 
For example, for the standard DP-SGD method,  the function $M$ reduces to {\small $M_{\rm dp} \deq  T\sum_{t=0}^{T-1}{1}/{b^2_{t+1}}$}. Our proposed method - involving $\alpha_t=\sqrt{b_t}$ - admits a bound improved by a factor of 
% {\small $M_{\rm dp}/M_{\rm adp} =  T\sum_{t=0}^{T-1} {1}/{b^2_{t+1}}/(\sum_{t=0}^{T-1}{1}/{b_{t+1}})^2$}.

%  \vspace{-0.15cm}
\[
 {M_{\rm dp}}/{M_{\rm adp}}
% =  T\left(\sum_{t=0}^{T-1} {1}/{b^2_{t+1}}\right)/\left(\sum_{t=0}^{T-1}{1}/{b_{t+1}}\right)^2
=T \left({\textstyle \sum_{t=0}^{T-1} {1}/{b^2_{t+1}}}\right)/\left( {\textstyle \sum_{t=0}^{T-1}{1}/{b_{t+1}}}\right)^2\overset{(a)}{\geq} 1,\]
%  \vspace{-0.15cm}

% Using our proposed adaptive differentially private parameter $\alpha_t=\sqrt{b_t}$, we obtain an improvement by a factor of $M/M_{\rm adp}$ for the utility guarantees.
% under the same condition. %\irina{"the same setup" is not clear}.
% \todo{same condition as what?}
 %\irina{can we cut the first sentence and just keep the latter one?} 
\noindent where (a) is due to the Cauchy-Schwarz inequality; thus, ADP-SGD is not worse than 
DP-SGD for any choice of $\{b_t\}$.
In the following section, we will analyze this factor of $M_{\rm dp}/M_{\rm adp}$ for two widely-used stepsize schedules: 
(a) the  polynomially decaying stepsize given by $\eta_t=1/\sqrt{1+t}$; and 
(b) a variant of adaptive gradient methods \citep{duchi2011adaptive}.
% \todo{flesh out. say, e.g., imagine $\alpha_t = ...$ and $b_t=...$. Then $M_{dp}/M_{adp}$ is $...$, revealing a significant factor of improvement associated with ADP.}

%$.% \todo{Where does this factor appear? This signifies the factor of improvement in one term of the utility function? Or something else? Not clear.}

% Although $b_t$ in \cref{thm:optimal-bound} is assumed to be independent of the stochastic gradient $\nabla f(\theta_t;x_{\xi_t})$, we will prove in \cref{sec:adp-adagrad-norm} that if $b_t$ chosen adaptive to the gradients (i.e., the adaptive gradient methods),

% \cref{thm:optimal-bound} implies that bound $M$ is a general update rule $\alpha_t$ and $b_t$

% we see that 

% See the proof in Appendix \ref{sec:proof-theorem-adp-adagrad-norm}.  
% Compare with .... \me{not yet finish writing}

% The relation $\alpha^2_t = b_t $ is not the unique one to obtain $M_{adp}$. 
% There are other relationships between the two sequences $\{(\alpha_t/b_t)^2\}$ and $\{\alpha^2_t\}$ relusting in the same $M_{adp}$. 
Note that, in addition to $\alpha^2_t = b_t$, there are other relationships between the sequences $\{(\alpha_t/b_t)^2\}$ and $\{\alpha^2_t\}$ that could lead to the same $M_{\rm adp}$. 
For instance, 
% \begin{align*}
% M = \sum_{t=0}^{T-1} \frac{ \alpha^2_{t+1}}{b^2_{t+1}}  \sum_{t=0}^{T-1}  \frac{ 1}{\alpha^2_{T-1-t}}     & \geq   \left(\sum_{t=0}^{T-1}\sqrt{ \frac{ \alpha^2_{t+1}}{b^2_{t+1}} }\sqrt{   \frac{ 1}{\alpha^2_{T-1-t}}}  \right)^2  =\left( \sum_{t=0}^{T-1} \frac{ 1}{\alpha^2_{t+1}}  \right)^2=M_{\rm adp},
% \end{align*}
% which suggests 
setting $\alpha_t\alpha_{T-(t-1)}=b_t$ is another possibility. Nevertheless, in this paper, we will focus on the $\alpha_t^2=b_t$ relation, and leave the investigation of other appropriate choices to future work.  We emphasize that the bound in \cref{thm:optimal-bound} only assumes $f$ to have Lipschitz smooth gradients and be bounded. Thus, the theorem applies to both convex or non-convex functions.
 Since our focus is on the improvement factor $M_{\rm dp}/M_{\rm tadp}$, we will assume our functions are non-convex, but the results will also hold for convex functions.
%  \irina{can we say "have its gradient Lipschitz smooth and bounded?"}.\me{assumes is a better word for papers?}
%  Thus, the theorem applies to both convex or non-convex functions. Since our focus is on the improved factor $M_{dp}/M_{adp}$, we use non-convex function in which the results should be hold similarly for convex functions.  \irina{rephrase as: "Since our focus in on the improved factor...we will assume our functions are non-convex, but the results will similarly follow for convex functions.}

%\vspace{-0.2cm}
\section{Examples for ADP-SGD}\label{sec:sec5}
%\vspace{-0.2cm}
\label{sec:dp-sgd-adaptive}
% \lingxiao{should we combine this with the section before?}\me{I guess we should let this section stand along as the contents are}
We now analyze the convergence bound given in \cref{thm:optimal-bound} and obtain an explicit form for $M$  
% \todo{where is $M_{adp}$ defined?}
in terms of $T$ by setting the stepsize to be $1/b_{t+1} \propto 1/\sqrt{t}$, which is closely related to the polynomially decreasing rate of adaptive gradient methods \citep{duchi2011adaptive,ward2019adagrad} studied in \cref{sec:adaptive}.  
%\vspace{-0.15cm}
\paragraph{Constant stepsize v.s. time-varying stepsize.} If the constant step size is used, then there is no need to use the adaptive DP mechanism proposed in this paper as we verify that constant perturbation to the gradient is optimal in terms of convergence. However, as we explained in the introduction, to ease the difficulty of stepsize tuning, time-varying stepsize is widely used in many practical applications of deep learning. We will discuss two examples below. In these cases, the standard DP mechanism (i.e., constant perturbation to the gradient) is not the most suitable technique, and our proposed adaptive DP mechanism can give better utility results. %
%\vspace{-0.15cm}
\paragraph{Achieving $\log T$ improvement.} We present Proposition \ref{thm:theorem-adp-descrease}  and Proposition \ref{prop:adp-adagrad-norm} to show that our method achieves $\log(T)$ improvement over the vanilla DP-SGD. 
Although this $\log(T)$ improvement can also be achieved by using the moments accountant method (MAM) \citep{Mironov2019RnyiDP},  \textit{we emphasize that our proposed method is orthogonal and complementary to MAM}.  This is because the $\log(T)$ improvement using MAM is over $B_{\delta}$  (see discussion after Theorem \ref{thm:privacy}), while ours is during the optimization process depending on stepsizes.  Nevertheless, since the two techniques are complementary to each other, we can apply them simultaneously and achieve a $\log^2(T)$ improvement over DP-SGD using the advanced composition for $O(1/\sqrt{t})$ stepsizes, compared to a $\log(T)$ improvement using either of them.   Thus,  an adaptive DP mechanism for algorithms with time-varying stepsizes is advantageous.
%\vspace{-0.15cm}
% \todo{support this statement with references}
%We refer to \cref{sec:append-p} \todo{missing ref} for $p\ne1/2$. 
%As \cref{prop:mechasim} shows that each iteration follows the Gaussian Mechanism, total $k$ iteration  
%  Thanks to advanced decomposition and amplification by sampling (c.f. Lemma \ref{lem:sampling}),  the variance of the Gaussian noise $\sigma^2$ in  \cref{alg:privacy-general} for this section all admits much smaller value, i.e., $\sigma = (4G\sqrt{2T}\log(1.25/\delta)/(n \varepsilon))$, in order to maintain $(\varepsilon,\delta)-$DP.  See \cref{sec:append-basic} for the detailed derivation of $\sigma$. With the $\sigma$ ready, we now state the convergence guarantee for polynomially decaying stepsizes  $b_{t+1}=1/{(a+ct)}^p$ where $a>0, c>0, p={1/2}$, which is closely related to the polynomially decaying rate of adaptive gradient methods. We refer to \cref{sec:append-p} \todo{missing ref} for $p\ne1/2$.
%  %\vspace{-0.15cm}
% To get  a tight bound  in $B_2$,
\subsection{Example 1: ADP-SGD with polynomially decaying stepsizes}
% %\vspace{-0.15cm}
\label{subseq:4.2}
The first case we consider is the stochastic gradient descent with polynomially decaying stepsizes. More specifically, we let $b_t= (a+ct)^{1/2}$, $a>0, c>0$.
\begin{prop}[\textbf{ADP-SGD v.s. DP-SGD for a polynomially decaying stepsize schedule}]\label{thm:theorem-adp-descrease}  
Under the  conditions of \cref{thm:optimal-bound} on $f$ and $\sigma^2$, let $b_t= (a+ct)^{1/2}$ with $a>0, c>0$ in \cref{alg:privacy-general}. Denote $  \tau=\arg\min_{t \in [T-1]} \mathbb{E} [ \|\nabla F(\theta_{t})\|^2]$,  and 
$B_{\delta}=\log(16T/(n\delta))\log(1.25/\delta)$. %\todo{Call this $B_T$ so that in the bound below, it's clear that the second term grows faster than $\sqrt{T}$}
If we choose $T\geq 5+4{a}/{c}$,  we have the following utility guarantee for ADP-SGD ($\alpha_t^2=b_{t}$) and DP-SGD ($\alpha_t^2=1$) respectively,

\begin{align}
  &\textbf{(ADP-SGD) } \quad  \mathbb{E}  [\|\nabla F(\theta_{\tau}^{\rm ADP})\|^2] 
    \leq
 \frac{ W^{decay}_{opt}}{\sqrt{T-1}}  +  \frac{ \eta d L(16G)^2 B_{\delta}\sqrt{T}}{2n^2\varepsilon^2 \sqrt{c} };%$T\geq 5+4 \frac{a}{c}$
 \label{eq:decrease-main}\\
%\label{eq:decrease}{thm:theorem-adp-descrease}  \\
% \end{align}
% \end{footnotesize}
% In addition, if we choose $T\geq 5+4{a}/{c}$ and $\alpha_t=1$, we have the utility guarantee for 
% \begin{footnotesize}
% \begin{align}
 & \textbf{(DP-SGD) }\quad  \mathbb{E}  [\|\nabla F(\theta_{\tau}^{\rm DP})\|^2] 
\leq \frac{W^{decay}_{opt}}{\sqrt{T-1}} + \frac{\eta d L (16G)^2  B_{\delta}\sqrt{T}\log\left( 1+T\frac{c}{a}\right)}{  n^2 \varepsilon^2\sqrt{c} }.\label{eq:decrease1-main}%\\ \label{eq:decrease1}
\end{align}
where  $ W^{decay}_{opt} = \sqrt{c}\left(\frac{D_F}{\eta} +\frac{\eta G^2 L B_T}{2c}\right)$.

\end{prop}
% Set Algorithm 1 with $b_{t}$ defined in \eqref{eq:decay}. Denote $  \tau=\arg\min_{t \in [T-1]} \mathbb{E} [ \|\nabla F(\theta_{t})\|^2]$ and $B_1=\log\left( 1+T{c}/{a}\right)$. \todo{Call this $B_T$ so that in the bound below, it's clear that the second term grows faster than $\sqrt{T}$}

% Suppose  $T\geq 5+4{a}/{c}$. With $\alpha_t^2=b_{t}$, the gradients satisfies
% \begin{footnotesize}
% \begin{align}
%  \textbf{(ADP-SGD)}\quad   \mathbb{E}  [\|\nabla F(\theta_{\tau})\|^2] 
%     &\leq
%  \frac{\sqrt{c}\left(\frac{D_F}{\eta} +\frac{\eta G^2 L B_1}{2c}\right)}{\sqrt{T-1}}  +  \frac{ 32\eta d LG^2 \log^2\left(\frac{1}{\delta} \right)\sqrt{T}}{n^2\varepsilon^2 \sqrt{c} }. %$T\geq 5+4 \frac{a}{c}$
% \label{eq:decrease}
% \end{align}
% \end{footnotesize}
% \end{prop}
% \begin{prop}[\textbf{DP-SGD Convergence}] \label{thm:decrease}
% Set Algorithm 1 with $b_{t}$ defined in \eqref{eq:decay}  and $\sigma = \frac{4G\sqrt{2T}\log(1.25/\delta)}{  n \varepsilon }$ for some $T\geq 5+4{a}/{c}$. 

% Suppose  $T\geq 5+4{a}/{c}$. With $\alpha_t=1$, the gradients follow 
The proof of \cref{thm:theorem-adp-descrease} is given in \cref{sec:adp-decrease} and \cref{sec:dp-decrease}.
% , which are direct results of \cref{thm:optimal-bound} using Lemma \ref{lem:sum}.
% Observe that the error introduced by the private mechanism, i.e., the second term in the utility bound, using ADP-SGD is smaller than that using DP-SGD by a factor -- $\mathcal{O}(\log(T))$. 
% \irina{some typo in the reference here; "proppsition"...}
\cref{thm:theorem-adp-descrease} implies $M_{\rm dp}/M_{\rm adp} = {\cal O}(\log T)$ -- that is,
% This shows that 
ADP-SGD has an improved utility guarantee compared to DP-SGD. Such an improvement can be significant when $d$ is large  and $LG^2$ is large. %  \emp{Note that  $\log(T)$ in \ref{eq:key-gen} by using moments accountant \citep{Mironov2019RnyiDP}, we decide to keep this term to directly compare with  \citep{bassily2014private}}.

\subsection{Example 2: ADP-SGD with adaptive stepsizes}\label{sec:adaptive}
% \vspace{-0.2cm}
We now examine another choice of the term $b_t$, which relies on a variant of adaptive gradient methods \citep{duchi2011adaptive}. %we extend the predefined $\alpha_t$ in previous section to an unknown sequence of $\alpha_t$ since $\alpha_t^2=b_t$ where $b_t$ is updated on the fly depending on the gradients during the optimization process.
To be precise, we assume $b_t$ is updated according to the norm of the gradient, i.e., $b_{t+1}^2 =b_{t}^2+ \max\{\|\nabla f(\theta_{t};x_{\xi_t})\|^2,\nu \}$, where $\nu>0$ is a small value to prevent the extreme case in which $1/b_{t+1}$ goes to infinity (when $b_0^2=\|\nabla f(\theta_{t};x_{\xi_t})\|^2\to 0$, then $\eta/b_1\to\infty$). %enough that the gradient $\|\nabla f(\theta_{t};x_{\xi_t})\|^2$ dominates at the beginning of the optimization \irina{can we say "dominates at the beginning of the optimization" a little better?; not sure how}.
We choose this precise equation formula because it is simple, and it also represents the core of adaptive gradient methods - updating the stepsize on-the-fly by the gradients \citep{levy2018online,ward2019adagrad}. The conclusions for this variant may transfer to other versions of adaptive stepsizes, and we defer this to future work. 

Observe that $b_t\propto t^{1/2}$ since $b_t^2\in[b_0^2+tv, b_0^2+tG]$, which at a first glance  indicates that the bound for this adaptive stepsize could be derived via a straightforward application of 
\cref{thm:theorem-adp-descrease}.
However, since $b_t$ is now a random variable correlated to the stochastic gradient $\nabla f(\theta_t;x_{\xi_t})$, we cannot directly apply \cref{thm:optimal-bound} to study $b_t$. To tackle this, we adapt the proof technique from  \citep{ward2019adagrad} and obtain Theorem \ref{thm:adp-adagrad-norm}, which we defer to Appendix \ref{sec:adp-adagrad-norm}.

As we see, $b_t$ is updated on the fly during the optimization process. Applying our propsoed method with  $\alpha_t^2 = b_t$ for this adaptive stepsize is not possible since $\alpha_t$ has to be set beforehand according to Equation~\eqref{eq:key-gen} in \cref{thm:privacy}. To address this, we note $b_t^2\in[b_0^2+tv, b_0^2+tG]$. Thus, we propose to set $\alpha_t^2 = \sqrt{b_0^2 +tC}$ for some  $C\in[\nu,G^2]$ and obtain Proposition \ref{prop:adp-adagrad-norm} based on  Theorem \ref{thm:adp-adagrad-norm}.
% \begin{footnotesize}
% % \begin{align*}
% %     M =\sum_{t=0}^{T-1} \frac{\alpha^2_{t+1}}{b_{t+1}^2}  \sum_{t=1}^{T-1}  \frac{ 1}{\alpha^2_{t+1}} \leq\sum_{t=1}^{T} \frac{ \alpha^2_{t}}{b_0^2+t\nu}  \sum_{t=1}^{T}  \frac{ 1}{\alpha^2_{t}} = \left(\sum_{t=1}^{T} 1/\sqrt{b_0^2+t\nu}\right)^2
% % \end{align*} 
% \end{footnotesize}
% where the second equality is due to setting $\alpha_t^2=\sqrt{b_0^2+t\nu}$. Thus we have following proposition for the comparision between ADP with $\alpha_t^2=\sqrt{b_0^2+t\nu}$ and DP ($\alpha_t=1$). 
% \todo{I'm lost at this point. I thought the point of 6.2 was that now the $b_t$s are adaptive. What I expected was to see a bound that I could compare with earlier bounds for non-adaptive $b_t$, but this isn't clear or easy. Make it blindingly obvious for reviewers who will not take time to figure out every detail. is there a toy model that could be used to do this? }

% Since the noise variance $\sigma^2$ is required to set beforehand according to \eqref{eq:key-gen} in \cref{thm:privacy}, we can not make $\alpha_t^2 = b_t$ as the stochastic gradient sequence $\{\|\nabla f(\theta_{t};x_{\xi_t})\|^2 \}_{t=0}^{T-1}$ is unknown ahead.
% to be clear, AdaGrad-Norm adapts the step size, but fixes the $\alpha_t$s, so that the privacy mechanism is non-adaptive;
\begin{prop}[\textbf{ADP v.s. DP with an adaptive stepsize schedule}]\label{prop:adp-adagrad-norm}
Under the same conditions of \cref{thm:adp-adagrad-norm} on $f$, $\sigma^2$, and $b_{t}$, if $\alpha_t = (b_0^2+tC)^{1/4}$ for some $C\in [\nu, G^2]$, %For \cref{alg:adp-adagrad-norm1}, we have
then 
% \begin{footnotesize}
\begin{align}
 &\textbf{(ADP-SGD)} \quad   \mathbb{E}  \|\nabla F(\theta_{\tau}^{\rm ADP})\|^2
\leq \frac{W_{opt}^{adap}}{\sqrt{T-1}} +\frac{128G^3\eta d LB_{\delta}\sqrt{T}}{n^2\varepsilon^2 \nu}. \nonumber \\%\label{eq:adagrad1}\\
 &  \textbf{(DP-SGD)} \quad  \mathbb{E}  \|\nabla F(\theta_{\tau}^{\rm DP})\|^2
  \leq\frac{W_{opt}^{adap}}{\sqrt{T-1}} +\frac{32G^3\eta d LB_{\delta}\sqrt{T}\log \left(1+T\frac{\nu}{b_0^2} \right)}{n^2\varepsilon^2\nu}. \nonumber %\label{eq:adagrad2}
\end{align}
where  %{\footnotesize
$\textstyle W_{opt}^{adap}=  2G\left(2G+ \frac{\eta L}{2}\right) \left(1+ \log \left(\frac{T(G^2+\nu^2)}{b_0^2}  +1 \right)\right)+\frac{2GD_F }{\eta}$.
% \vspace{-0.26cm}
% \end{footnotesize}
% \vspace{-0.3cm}
\end{prop}
See the proof in \cref{proof:adp-adagrad-norm}. 
% According to \cref{thm:adp-adagrad-norm}, ADP-SGD can reduce the error introduced by the random noise, i.e., the second term in the utility bound, by
Similar to the comparison in \cref{thm:theorem-adp-descrease}, the key difference between  two bounds in Proposition \ref{prop:adp-adagrad-norm} is the last term; using ADP-SGD gives us a tighter utility guarantee than the one provided by DP-SGD by a factor of $\mathcal{O}(\log(T))$.
% or an extra term 
% ${32G^3\eta d LB_{\delta}\sqrt{T}(\log \left(1+T{\nu}/{b_0^2} \right)}-4)/({n^2\varepsilon^2\nu})$.
This improvement is significant when the dimension $d$ is very high, or when either $L$, $G$, or $T$ are sufficiently large. Note that the bound in \cref{prop:adp-adagrad-norm} does not reflect the effect of the different choice of $C$, as the bound corresponds to the worst case scenarios. 
% In order to examine the role of the parameter $C$, we will perform experiments testing a range of $C$ values, in order to thoroughly examine the properties of APD-SGD for adaptive stepsizes \irina{phrase is complicated; say somehow that "we'll perform experiments testing a wide range of c values and this will allow us to thoroughly examine  the properties of ADP-SGD for adaptive step-sizes"}.
We will perform experiments testing a wide range of $C$ values and this will allow us to thoroughly examine the properties of ADP-SGD for adaptive stepsizes.

\section{Experiments}\label{sec:exp}
%  \vspace{-0.15cm}
In this section, we present numerical results to support the theoretical findings of our proposed methods. We perform two sets of experiments: (1) when $\eta_t$ is polynomially decaying, we compare ADP-SGD ($\alpha_t^2=b_t$) with DP-SGD (setting $\alpha_t=1$ in \cref{alg:privacy-general} ); and (2) when $b_t$ is updated by the norm of the gradients, we compare ADP-SGD ($\alpha_t^2=\sqrt{b_0^2+tC}$) with DP-SGD. The first set of experiments is designed to examine the case when the learning rate is precisely set in advanced (i.e., \cref{thm:theorem-adp-descrease}), while the second concerns when the learning rate is not known ahead (i.e., \cref{prop:adp-adagrad-norm}). In addition to the experiments above, in the supplementary material (\cref{sec:decaying-stepsize-better}),
we present  strong empirical evidence in support of the claim that using a decaying stepsize schedule yields  better results than simply employing a constant stepsize schedule. See Section \ref{sec:code} for our code demonstration.

\paragraph{Assumption \ref{asmp:G-bound} and the gradient clipping method.} One limitation for the above proposition is the bounded gradient (Assumption \ref{asmp:G-bound}). However, as discussed in Section \ref{sec:privacy}, this is common. But $G$ could be very large, particularly for all $\theta \in \mathbb{R}^d$ in highly over-parameterized models (such as neural networks). To make the algorithm work in practice, we use gradient clipping \citep{chen2020understanding,andrew2019differentially,pichapati2019adaclip}. That is, given the current gradient $\nabla f_{\xi_t}(\theta_t)$, we apply a function $h(\cdot; C_G):\mathbb{R}^d\to\mathbb{R}^d$, which depends on the the positive constant $C_G>0$ such that of $\|h(\nabla f_{\xi_t}(\theta_t)\|\leq C_G$. Thus, the implementation of our algorithms (sample codes are shown in Figure \ref{fig:code1}) are 
\begin{align}
 \text{{ADP-SGD}: }\quad &\theta_{t+1} =  \theta_t - \eta_t g_t;\quad  g_t  =h(\nabla f_{\xi_t}(\theta_t), C_G) +\eta_t\alpha_t Z.  \label{eq:ada-sgd-gc}\\
% \end{align} 
% \begin{align}
\text{DP-SGD: } \quad & \theta_{t+1} = \theta_t - \eta_t g_t;\quad g_t  = h(\nabla f_{\xi_t}(\theta_t), C_G)  + Z  ,  \label{eq:dp-sgd1-gc}
\end{align} 
Regarding the convergence result of using the gradient clipping method instead of the bounded gradient assumption (Assumption~\ref{asmp:G-bound}),  \citep{chen2020understanding} show that if the gradient distribution is “approximately” symmetric, then the gradient norm goes to zero (Corollary 1). Furthermore, \citep{chen2020understanding}  showed (in Theorem 5) that the convergence of DP-SGD with clipping (without bounded gradient assumption) is $O(\sqrt{d}/(n\varepsilon))+$ clipping bias with the specified constant learning rate $O(1/\sqrt{T})$. 

There is a straightforward way to apply our adaptive perturbation to the above clipping result (e.g., Theorem 5 in \citep{chen2020understanding}) using the time-varying learning rate. The bounds for ADP-SGD and DP-SGD respectively are \eqref{eq:decrease-main}+clipping bias and \eqref{eq:decrease1-main}+clipping bias where the constant $G$ in \eqref{eq:decrease-main} and \eqref{eq:decrease1-main} is now replaced with $C_G$. Thus, there is still a $log(T)$ factor gain if clipping bias is not larger than the bounds in \eqref{eq:decrease-main} and \eqref{eq:decrease1-main}. In our experiments, we will be using various gradient clipping values $C_G\in\{0.5, 1.0, 2.5, 5.0\}$ to understand how it affects our utility. 
% Codes will be released in Github.
% \vspace{-0.15cm}
\paragraph{Datasets and models.}\label{dataset} We perform our experiments on   CIFAR-10 \citep{krizhevsky2009learning}, using a convolution neural network (CNN). See our CNN design in the appendix. Notably, following previous work \citep{abadi2016deep}, the CNN model is pre-trained on CIFAR-100 and fined-tuned on CIFAR-10. The mini-batch size is $256$, and each independent experiment runs on one GPU. We set $\eta=1$ in  \cref{alg:privacy-general} (line 6) and use the gradient clipping with $C_{G} \in\{0.5,1,2.5,5\}$ \citep{chen2020understanding,andrew2019differentially,pichapati2019adaclip}. Note that one might need to think about $C_{G}$ as being approximately closer to the bounded gradient parameter $G$. We provide a more detailed discussion in Appendix \ref{gradient-clipping}. The privacy budget is set to be $\bar{\varepsilon}=\varepsilon/C_{\varepsilon}\in\{0.8, 1.2,1.6,3.2,6.4\}$ and
% \becca{I think you use $\varepsilon$ and $\varepsilon_T$ interchangably and it can be a little confusing}
we choose $\delta=10^{-5}$.\footnote{The constant $C_{\varepsilon}=16$  in \eqref{eq:key-gen}. Although $\varepsilon=16 \bar{\varepsilon}$ is large for $ \bar{\varepsilon}\in\{0.8, 1.2,1.6,3.2,6.4\}$, they match the numerical privacy $\{0.29,0.43,0.57, 1.23,$ $3.24\}$ calculated by the moments accountant with the noise determined by  $T=11700$ (60 epochs) and the gradient clipping $C_G=1.0$. The code is based on \url{https://github.com/tensorflow/privacy}}  Given these privacy budgets, we calculate the corresponding variance by \cref{thm:privacy} (See Appendix \ref{sec:code} for the code to obtain $\sigma$). %We used gradient clipping $G=0.5$. 
 We acknowledge that our empirical investigation is limited by the computing budget and this limitation forced us to choose between diversity in the choice of iteration complexity $T$ and type of stepsizes as opposed to diversity in the model architectures and datasets. 
 
\paragraph{Performance measurement.} There are $50000$ images used for training  and $10000$ images for validation, respectively.
 We repeat the experiments five times.  For the $i$-th independent experiment, we calculate the validation accuracy at every 20 iterations and  select the best  validation accuracy $acc^{best}_{i}$  during the optimization process (over the entire iteration $\{t\}_{t=1}^T$). In Table \ref{table:cifar-decay1} and Table \ref{table:cifar-decay}, we report the average and standard deviation of $\{\text{acc}^{best}_{i}\}_{i=1}^5$, which is closely related to $\min_{t\in[T]}\mathbb{E}  \|\nabla F(\theta_{t}^{\rm ADP})\|^2$ and  $\min_{t\in[T]}\mathbb{E}  \|\nabla F(\theta_{t}^{\rm DP})\|^2$, the convergence metric for our theoretical analysis in Theorem \ref{thm:optimal-bound}. Additionally, to further understand the method's final performance,  we report in Table \ref{table:cifar-decay-b} and \ref{table:cifar-decay-a} the average and standard deviation of the accuracy $\{\text{acc}^{last}_{i}\}_{i=1}^5$ where $\text{acc}^{last}_{i}$ represents the validation accuracy at iteration $T$ for the $i$-th experiment.

\subsection{ADP-SGD v.s. DP-SGD for polynomially decaying stepsizes}
% \vspace{-0.15cm}

We focus on understanding the optimality of the theoretical guarantees of  Theorem \ref{thm:optimal-bound} and \cref{thm:theorem-adp-descrease}; the experiments help us further understand how this optimality reflects in generalization. We consider training with $T=11760,23520,39200$ iterations corresponding to $60,120,200$ training epochs (196 iterations/epoch), which represents the practical scenarios of  limited,  standard and large computational time budgets. We use two kinds of monotone learning rate schedules: i) $\eta_t=0.1-\alpha_T\sqrt{t}$ (see orange curves in Figure \ref{fig:stepsize}) and report the results in  Table \ref{table:cifar-decay1}; ii) $\eta_t=\eta/b_{t+1}=1/\sqrt{20+t}$ (see blue curves in Figure \ref{fig:stepsize}), in which the results are given in Table \ref{table:cifar-decay}. 
The former learning rate schedule is designed to make sure the learning rate reaches close to zero at $T$, while the latter one is to match precisely Proposition \ref{thm:theorem-adp-descrease}.
%  We include plots in Figure  \ref{fig:cifar-decay} to provide detailed comparisons between ADP-SGD and DP-SGD. 
% We investigate the performance of our proposed ADP-SGD method with $b_t=\sqrt{20+t}$ by comparing it with DP-SGD. %More specifically, we choose  $a=20$ and $c\in\{ 0.001,0.01,0.1,0.5,1,2\}$.
\begin{figure}[H]
\includegraphics[width=0.85\linewidth]{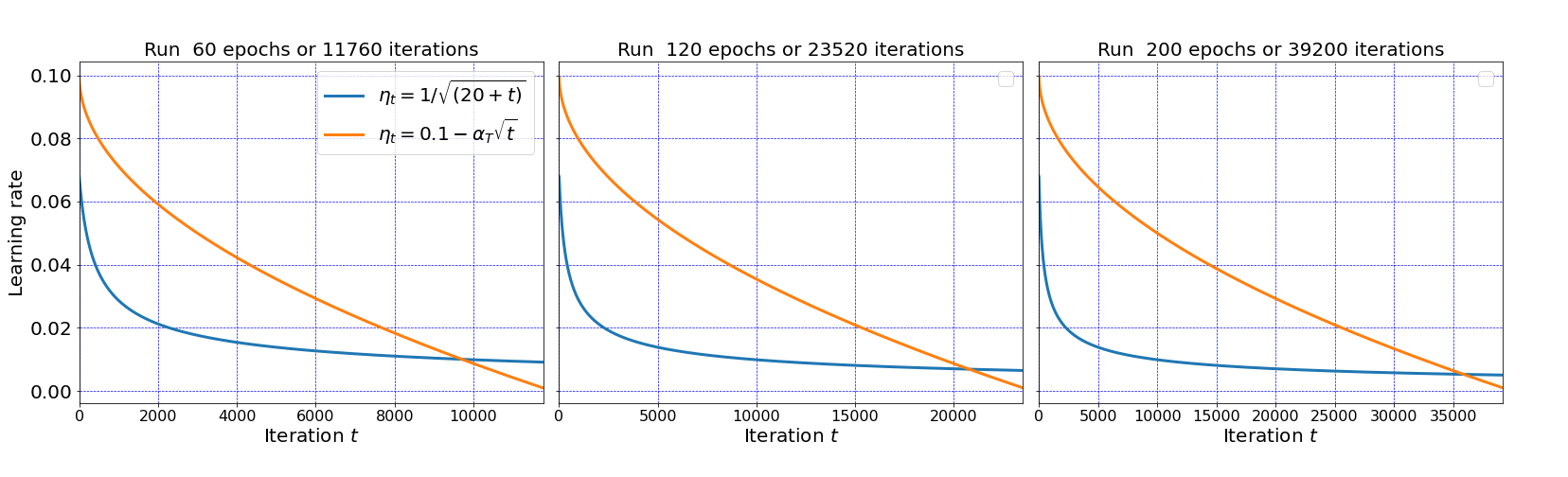}
\caption{\small\textbf{Two schedules of decaying learning rates.} The blue curve in three plots are the same, while the orange ones are different. The blue learning rate is  $\eta_t=\eta/b_{t+1}=1/\sqrt{20+t}$ used for Table \ref{table:cifar-decay} and \ref{table:cifar-decay-a}. The orange ones, used for Table \ref{table:cifar-decay1} and \ref{table:cifar-decay-b}, are described by $\eta_t = 0.1-\alpha_T\sqrt{t}$ where $\alpha_T$ is the ratio depending on the final iterations $T$ (i.e. $196\times$epochs) such that  the learning rate at $T$ is $\eta_T=10^{-10}$. That is the $\alpha_T = (0.1-10^{-10})/\sqrt{T}$.}\label{fig:stepsize}
\end{figure}
\paragraph{Observation from  Table \ref{table:cifar-decay1} and Table \ref{table:cifar-decay}.} The results from the two tables show that the overall performances of our method (ADP-SGD) are mostly better than DP-SGD given a fixed privacy budget and the same complexity $T$, which matches our theoretical analysis. Particularly, the increasing $T$ tends to enlarge the gap between ADP-SGD and DP-SGD, especially for smaller privacy; for $\bar{\varepsilon}=0.8$ with $C_G=1$ in Table \ref{table:cifar-decay}, we have improvements of $0.8\%$ at epoch $60$, $1.48\%$ at epoch 120, and  $7.03\%$ at  epoch $200$. This result is reasonable since, as explained in \cref{thm:theorem-adp-descrease}, ADP-SGD  improves over DP-SGD by a factor $\log(T)$.

Furthermore, our method is more robust to the predefined complexity $T$ and thus provides an advantage when using longer iterations. For example, for $\bar{\varepsilon}=3.2$ with $C_{G}=2.5$ in Table \ref{table:cifar-decay}, our method increases from $65.34\%$ to $66.41\%$ accuracy when the iteration complexity of 60 epochs is doubled; it maintains the accuracy $65.74\%$ at the longer epoch $200$. In contrast, under the same privacy budget and  gradient clipping,  DP-SGD suffers the degradation from $66.08\%$ (epoch 60) to $65.17\%$ (epoch 200).

\paragraph{Discussion on gradient clipping.}\label{gradient-clipping}
Both results in Table \ref{table:cifar-decay1} and Table \ref{table:cifar-decay} indicate that: (1) Smaller gradient clipping $C_G$ could help achieve better accuracy when the  privacy requirement is strict (i.e., small privacy $\bar{\varepsilon} \in \{0.8, 1.2\}$). For a large privacy requirement, i.e., $\bar{\varepsilon} \in \{3.2, 6.4\}$, a bigger gradient clipping value is more advantageous; 
(2) As the gradient clipping $C_G$ increases, the gap between DP and ADP tends to be more significant. This matches our theoretical analysis (e.g. Proposition \ref{thm:theorem-adp-descrease}) that the improvement of ADP-SGD over DP-SGD is by a magnitude $\mathcal{O}(d L G^2\log(T)\sqrt{T}/n^2)$ where $C_G$ can replace $G$ as we discussed in paragraph ``Assumption \ref{asmp:G-bound} and the gradient clipping method".
\begin{table}[H]
\caption{\small \textbf{Mean accuracy of ADP-SGD/DP-SGD with polynomially decaying stepsizes $\eta_t = 0.1-\alpha_T\sqrt{t}$ where $\alpha_T$ is the ratio depending on the final epochs/iterations $T$ such that  the learning rate at $T$ is $\eta_T=10^{-10}$ (see the orange curves in Figure 
\ref{fig:stepsize}).} This table reports \emph{accuracy} for CIFAR10 with the mean and the corresponding standard deviation over $\{\text{acc}^{best}_i\}_{i=1}^5$. Here,  $\text{acc}^{best}_i$ is the best validation accuracy over the entire iteration process for the $i$-th independent experiment.
Each set $\{\text{acc}^{best}_i\}_{i=1}^5$ corresponds to a pair of $(\bar{\varepsilon},C_G,T, \text{Alg})$. The difference  (``Gap") between DP and ADP is provided for visualization purpose.  The results suggest that the more iterations or epochs we use, the more improvements ADP-SGD can potentially gain over DP-SGD. The results are reported in percentage ($\%$). The bolded number is the best accuracy in a row among epoch 60, 120 and 200 for the same $C_G$. See paragraph \textbf{Datasets and models} and \textbf{Performance measurement} for detailed information.} 
\label{table:cifar-decay1}
\scalebox{0.65}{ 
\begin{tabular}{l|l|ccc|ccc|ccc}
\hline
\multirow{2}{*}{$\bar{\varepsilon}$}  & \multirow{2}{*}{Alg}   & \multicolumn{3}{c|}{Gradient clipping $C_G=0.5$}     & \multicolumn{3}{c|}{Gradient clipping $C_G=1$}     & \multicolumn{3}{c}{Gradient clipping $C_G=2.5$}          \\
                  &               & epoch=$60     $ & epoch$=120    $ & epoch$=200    $ & epoch$=60     $ & epoch$=120    $ & epoch$=200   $ & epoch$=60     $ & epoch$=120    $ & epoch$=200   $\\ \hline\hline
\multirow{3}{*}{0.8} & ADPSGD & $\mathbf{57.05} \pm 0.505$& $52.14 \pm 0.641$& $44.93 \pm 0.594$& $\mathbf{51.61} \pm 0.849$& $42.81 \pm 1.015$& $36.57 \pm 0.532$& $\mathbf{38.85} \pm 1.279$& $30.05 \pm 1.238$& $22.24 \pm 1.807$\\ 
 & DPSGD & $\mathbf{56.12} \pm 0.631$& $44.16 \pm 0.140$& $31.79 \pm 1.131$& $\mathbf{44.07} \pm 1.350$& $29.67 \pm 0.656$& $21.17 \pm 0.583$& $\mathbf{27.24} \pm 1.675$& $17.32 \pm 1.677$& $15.23 \pm 0.478$\\ 
 & Gap& $0.93$& $7.98$& $13.14$& $7.54$& $13.14$& $15.4$& $11.61$& $12.73$& $7.01$\\ \hline 
\multirow{3}{*}{1.2} & ADPSGD & $\mathbf{59.84} \pm 0.248$& $59.01 \pm 0.833$& $54.78 \pm 0.512$& $\mathbf{58.92} \pm 0.279$& $52.7 \pm 0.861$& $47.27 \pm 0.742$& $\mathbf{50.08} \pm 0.601$& $41.57 \pm 1.572$& $33.7 \pm 1.393$\\ 
 & DPSGD & $\mathbf{59.71} \pm 0.682$& $56.93 \pm 0.539$& $45.52 \pm 0.969$& $\mathbf{57.23} \pm 0.358$& $41.44 \pm 1.079$& $32.78 \pm 0.971$& $\mathbf{36.72} \pm 0.942$& $27.26 \pm 0.656$& $20.22 \pm 0.658$\\ 
 & Gap& $0.13$& $2.08$& $9.26$& $1.69$& $11.26$& $14.49$& $13.36$& $14.31$& $13.48$\\ \hline 
\multirow{3}{*}{1.6} & ADPSGD & $61.26 \pm 0.264$& $\mathbf{62.04} \pm 0.196$& $59.37 \pm 0.257$& $\mathbf{62.16} \pm 0.419$& $57.72 \pm 0.643$& $53.19 \pm 0.520$& $\mathbf{55.97} \pm 0.702$& $48.4 \pm 0.740$& $41.75 \pm 0.657$\\ 
 & DPSGD & $60.76 \pm 0.454$& $\mathbf{61.38} \pm 0.156$& $55.39 \pm 0.954$& $\mathbf{61.53} \pm 0.638$& $52.72 \pm 0.500$& $41.17 \pm 1.011$& $\mathbf{47.83} \pm 0.263$& $35.37 \pm 1.327$& $27.97 \pm 0.704$\\ 
 & Gap& $0.5$& $0.66$& $3.98$& $0.63$& $5.0$& $12.02$& $8.14$& $13.03$& $13.78$\\ \hline 
\multirow{3}{*}{3.2} & ADPSGD & $61.82 \pm 0.267$& $65.77 \pm 0.272$& $\mathbf{66.42} \pm 0.505$& $65.36 \pm 0.265$& $\mathbf{66.06} \pm 0.171$& $64.35 \pm 0.270$& $\mathbf{66.13} \pm 0.380$& $60.96 \pm 0.260$& $57.31 \pm 0.271$\\ 
 & DPSGD & $61.7 \pm 0.300$& $65.54 \pm 0.066$& $\mathbf{66.2} \pm 0.156$& $65.14 \pm 0.254$& $\mathbf{65.83} \pm 0.339$& $61.73 \pm 0.405$& $\mathbf{64.68} \pm 0.479$& $54.42 \pm 0.434$& $48.04 \pm 0.878$\\ 
 & Gap& $0.12$& $0.23$& $0.22$& $0.22$& $0.23$& $2.62$& $1.45$& $6.54$& $9.27$\\ \hline 
\multirow{3}{*}{6.4} & ADPSGD & $62.19 \pm 0.642$& $66.29 \pm 0.220$& $\mathbf{68.39} \pm 0.197$& $66.04 \pm 0.034$& $69.07 \pm 0.213$& $\mathbf{69.89} \pm 0.139$& $\mathbf{69.53} \pm 0.201$& $69.51 \pm 0.369$& $66.95 \pm 0.474$\\ 
 & DPSGD & $61.94 \pm 0.436$& $66.28 \pm 0.289$& $\mathbf{68.29} \pm 0.208$& $66.36 \pm 0.265$& $68.73 \pm 0.173$& $\mathbf{69.26} \pm 0.131$& $\mathbf{69.41} \pm 0.051$& $68.79 \pm 0.213$& $63.91 \pm 0.209$\\ 
 & Gap& $0.25$& $0.01$& $0.1$& $-0.32$& $0.34$& $0.63$& $0.12$& $0.72$& $3.04$\\ \hline 
\end{tabular}}
\vspace{-0.15cm}
% \end{minipage}
\end{table}

\begin{table}[H]
% \begin{minipage}[t]{1\hsize}\centering
\caption{\small \textbf{Mean accuracy of ADP-SGD/DP-SGD with polynomially decaying stepsizes $\eta_t=\eta/b_{t+1}=1/\sqrt{20+t}$ (see the blue curve in Figure 
\ref{fig:stepsize}).} This table reports \emph{accuracy} for CIFAR10 with the mean and the corresponding standard deviation over $\{\text{acc}^{best}_i\}_{i=1}^5$. Here,  $\text{acc}^{best}_i$ is the best validation accuracy over the entire iteration process for the $i$-th independent experiment. See Table \ref{table:cifar-decay1} for reading instruction.} \label{table:cifar-decay}
\scalebox{0.75}{ 
\begin{tabular}{l|l|ccc|ccc}
\hline
\multirow{2}{*}{$\bar{\varepsilon}$}  & \multirow{2}{*}{Alg}   & \multicolumn{3}{c|}{Gradient clipping $C_G=0.5$}     & \multicolumn{3}{c}{Gradient clipping $C_G=1.0$}             \\
                  &               & epoch=$60     $ & epoch$=120    $ & epoch$=200    $ & epoch$=60     $ & epoch$=120    $ & epoch$=200   $ \\ \hline\hline
\multirow{3}{*}{0.9} & ADP-SGD & $55.59 \pm 0.580$& $\mathbf{57.5} \pm 0.109$& $57.29 \pm 0.447$& $\mathbf{56.38} \pm 0.092$& $54.2 \pm 0.730$& $51.71 \pm 1.092$\\ 
 & DP-SGD & $55.79 \pm 0.234$& $\mathbf{56.86} \pm 0.648$& $56.33 \pm 0.496$& $\mathbf{56.13} \pm 0.909$& $52.72 \pm 0.938$& $44.68 \pm 0.576$\\ 
 & Gap& $-0.2$& $0.64$& $0.96$& $0.25$& $1.48$& $7.03$\\ \hline 
\multirow{3}{*}{1.2} & ADP-SGD & $56.69 \pm 0.446$& $59.03 \pm 0.429$& $\mathbf{59.96} \pm 0.494$& $\mathbf{60.26} \pm 0.319$& $60.24 \pm 0.365$& $58.68 \pm 0.505$\\ 
 & DP-SGD & $56.0 \pm 0.987$& $59.08 \pm 0.393$& $\mathbf{60.2} \pm 0.790$& $\mathbf{60.09} \pm 0.450$& $60.02 \pm 0.204$& $57.56 \pm 0.514$\\ 
 & Gap& $0.69$& $-0.05$& $-0.24$& $0.17$& $0.22$& $1.12$\\ \hline 
\multirow{3}{*}{1.6} & ADP-SGD & $57.69 \pm 0.104$& $59.72 \pm 0.430$& $\mathbf{60.17} \pm 0.165$& $61.3 \pm 0.219$& $\mathbf{61.98} \pm 0.420$& $61.88 \pm 0.507$\\ 
 & DP-SGD & $56.52 \pm 0.251$& $59.03 \pm 0.638$& $\mathbf{61.49} \pm 0.195$& $61.18 \pm 0.195$& $\mathbf{61.89} \pm 0.317$& $61.46 \pm 0.490$\\ 
 & Gap& $1.17$& $0.69$& $-1.32$& $0.12$& $0.09$& $0.42$\\ \hline 
\multirow{3}{*}{3.2} & ADP-SGD & $57.21 \pm 1.165$& $59.84 \pm 0.256$& $\mathbf{61.64} \pm 0.299$& $61.76 \pm 0.490$& $64.27 \pm 0.257$& $\mathbf{65.54} \pm 0.066$\\ 
 & DP-SGD & $57.79 \pm 0.208$& $60.26 \pm 0.072$& $\mathbf{61.79} \pm 0.133$& $62.02 \pm 0.248$& $63.88 \pm 0.275$& $\mathbf{65.11} \pm 0.359$\\ 
 & Gap& $-0.58$& $-0.42$& $-0.15$& $-0.26$& $0.39$& $0.43$\\ \hline 
\multirow{3}{*}{6.4} & ADP-SGD & $58.08 \pm 0.309$& $60.03 \pm 0.275$& $\mathbf{61.68} \pm 0.364$& $62.2 \pm 0.270$& $64.57 \pm 0.515$& $\mathbf{65.74} \pm 0.270$\\ 
 & DP-SGD & $56.75 \pm 0.596$& $59.84 \pm 0.816$& $\mathbf{61.85} \pm 0.381$& $62.06 \pm 0.244$& $64.61 \pm 0.180$& $\mathbf{65.84} \pm 0.206$\\ 
 & Gap& $1.33$& $0.19$& $-0.17$& $0.14$& $-0.04$& $-0.1$\\ \hline 
\end{tabular}}
\scalebox{0.75}{ 
\begin{tabular}{l|l|ccc|ccc}
\hline
\multirow{2}{*}{$\bar{\varepsilon}$}  & \multirow{2}{*}{Alg}   & \multicolumn{3}{c|}{Gradient clipping $C_G=2.5$}     & \multicolumn{3}{c}{Gradient clipping $C_G=5.0$}             \\
                  &               & epoch=$60     $ & epoch$=120    $ & epoch$=200    $ & epoch$=60     $ & epoch$=120    $ & epoch$=200   $ \\ \hline\hline
\multirow{3}{*}{0.8} & ADP-SGD & $\mathbf{48.61} \pm 1.003$& $44.11 \pm 1.097$& $39.92 \pm 0.284$& $\mathbf{38.33} \pm 1.025$& $32.49 \pm 0.694$& $29.16 \pm 1.514$\\ 
 & DP-SGD & $\mathbf{38.06} \pm 1.029$& $23.64 \pm 0.796$& $17.75 \pm 1.068$& $\mathbf{21.06} \pm 1.507$& $15.83 \pm 0.245$& $15.87 \pm 1.291$\\ 
 & Gap& $10.55$& $20.47$& $22.17$& $17.27$& $16.66$& $13.29$\\ \hline 
\multirow{3}{*}{1.2} & ADP-SGD & $\mathbf{56.63} \pm 0.308$& $52.26 \pm 0.328$& $50.7 \pm 1.038$& $\mathbf{49.98} \pm 0.742$& $44.99 \pm 0.248$& $40.51 \pm 0.816$\\ 
 & DP-SGD & $\mathbf{55.71} \pm 0.418$& $43.16 \pm 0.604$& $32.0 \pm 2.281$& $\mathbf{34.26} \pm 0.906$& $22.62 \pm 0.596$& $16.46 \pm 0.437$\\ 
 & Gap& $0.92$& $9.1$& $18.7$& $15.72$& $22.37$& $24.05$\\ \hline 
\multirow{3}{*}{1.6} & ADP-SGD & $\mathbf{61.52} \pm 0.313$& $58.6 \pm 0.352$& $56.07 \pm 0.046$& $\mathbf{55.17} \pm 0.482$& $51.71 \pm 0.193$& $48.98 \pm 1.128$\\ 
 & DP-SGD & $\mathbf{61.76} \pm 0.454$& $55.68 \pm 0.243$& $46.74 \pm 0.428$& $\mathbf{47.31} \pm 0.631$& $32.15 \pm 1.254$& $23.96 \pm 1.700$\\ 
 & Gap& $-0.24$& $2.92$& $9.33$& $7.86$& $19.56$& $25.02$\\ \hline 
\multirow{3}{*}{3.2} & ADP-SGD & $65.64 \pm 0.0$& $\mathbf{66.41} \pm 0.054$& $65.74 \pm 0.106$& $\mathbf{65.38} \pm 0.171$& $62.95 \pm 0.132$& $61.46 \pm 0.261$\\ 
 & DP-SGD & $\mathbf{66.08} \pm 0.130$& $65.73 \pm 0.353$& $65.17 \pm 0.115$& $\mathbf{65.11} \pm 0.341$& $61.16 \pm 0.339$& $55.3 \pm 0.479$\\ 
 & Gap& $-0.44$& $0.68$& $0.57$& $0.27$& $1.79$& $6.16$\\ \hline 
\multirow{3}{*}{6.4} & ADP-SGD & $67.35 \pm 0.057$& $68.72 \pm 0.045$& $\mathbf{69.51} \pm 0.179$& $69.62 \pm 0.388$& $\mathbf{69.63} \pm 0.170$& $69.29 \pm 0.249$\\ 
 & DP-SGD & $67.06 \pm 0.244$& $68.46 \pm 0.321$& $\mathbf{69.28} \pm 0.147$& $69.34 \pm 0.205$& $\mathbf{69.63} \pm 0.123$& $68.6 \pm 0.254$\\ 
 & Gap& $0.29$& $0.26$& $0.23$& $0.28$& $0.0$& $0.69$\\ \hline 
\end{tabular}}
\end{table}

\begin{figure}[H]
\includegraphics[width=0.32\linewidth]{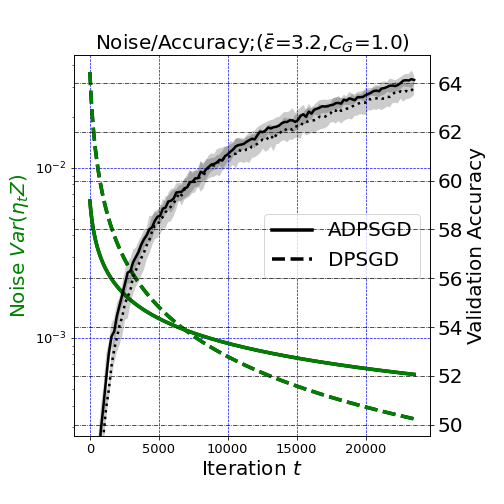}
\includegraphics[width=0.32\linewidth]{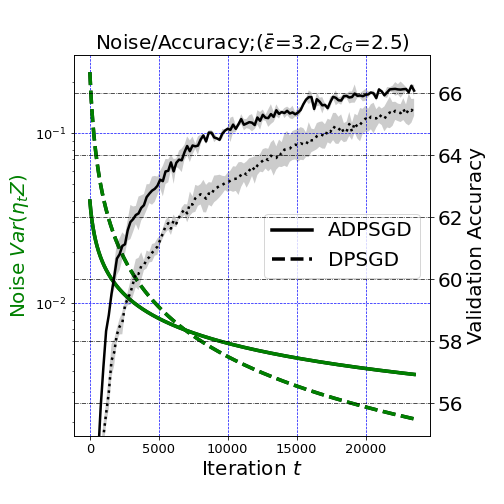} 
\includegraphics[width=0.32\linewidth]{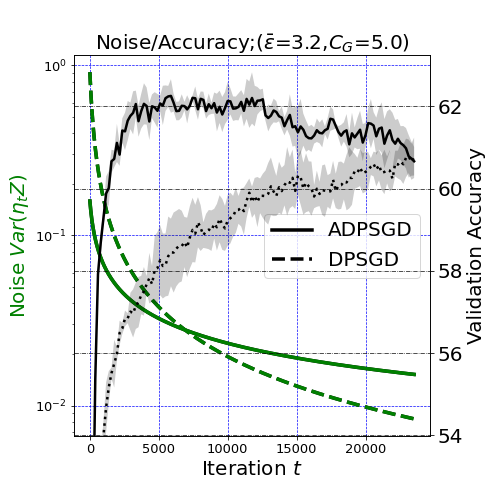}
\includegraphics[width=0.32\linewidth]{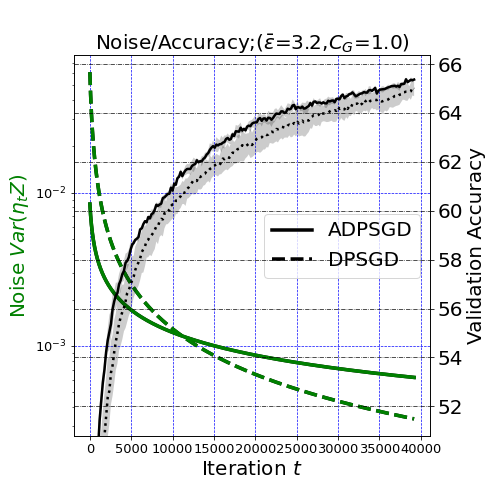}
\includegraphics[width=0.32\linewidth]{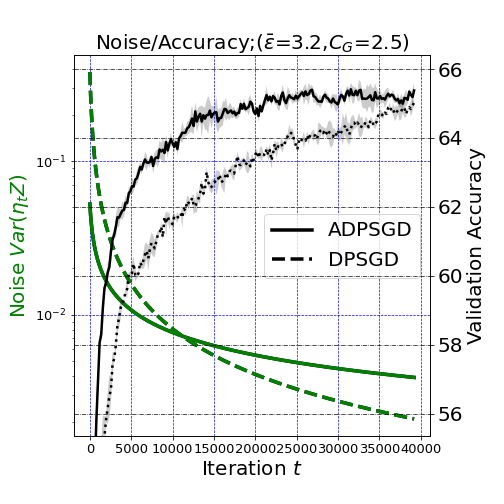} 
\includegraphics[width=0.32\linewidth]{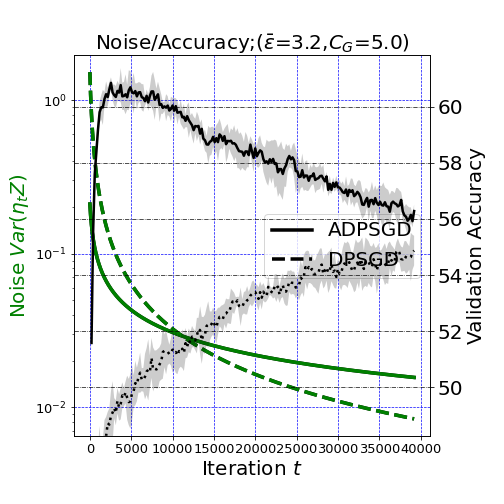}
\caption{\small  \textbf{Validation accuracy with respect to iteration $t$ using ADP-SGD/DP-SGD with polynomially decaying stepsizes $\eta_t=\eta/b_{t+1}=1/\sqrt{20+t}$ (see the blue curve in Figure \ref{fig:stepsize}).} The black lines corresponding to the right y-axis are measured by the validation accuracy on CIFAR10 for a CNN model using ADP-SGD  (solid line) and DP-SGD (dash line).   The shaded region is the one standard deviation.  Each plot corresponds to a privacy budget $\bar{\varepsilon}=3.2$ with a fixed gradient clipping value $C_G$ (see title) and a fixed $T$ (see x-axis) where the top (bottom) row is for  120 (200) training epochs. Same as \cref{fig:compare}, the monotone green curves, which correspond to the left vertical y-axis, show the actual noise for  $\alpha_t=1/\sqrt{\eta_t}$ (ADP-SGD, the solid line) and  $\alpha_t=1$ (DP-SGD, the dashed line). The top middle plot is the same as the right plot in Figure \ref{fig:compare}.  %The top/bottom rows from 1st to 4th column correspond to the privacy budgets from $0.8$ (epoch 60), $1.2$ (epoch 60), $1.6$ (epoch 120), and $3.2$ (epoch 200).  
% \becca{the color scheme is confusing here. for noise, they're both green, but for accuracy, there's a mix of orange and blue. I think both noise plots should be the same color (e.g. green) and both accuracy plots should be the same color (e.g. blue)}
% \becca{I cannot tell how the top and bottom rows are different}
% \becca{What is the objective? How is accuracy measured?}
}\label{fig:cifar-decay}
%\vspace{-0.36cm}
\end{figure}
To further understand DP and ADP with respect to different privacy parameters and the gradient clipping values, we present the detailed performance in Figure \ref{fig:cifar-decay}  for $T=23520$ (120 epochs) where each plot corresponds to a pair of $(\bar{\varepsilon}=3.2, C_G)$.  We see that from Figure \ref{fig:cifar-decay} that the gap between ADP-SGD and DP-SGD becomes more significant as $C_G$ increases from $1$ to $5$; ADP achieves the highest accuracy at $C_G=2.5$ with the best mean accuracy $66.41\%$. The intuition is that the noise added to the gradient with a large gradient clipping $C_G$ is much higher than that with a small gradient clipping value. In this situation, our method proves to be helpful by spreading out the total noise across the entire optimization process more evenly than DP-SGD (see the green curves in Figure \ref{fig:cifar-decay}). On the other hand, using a large gradient clipping $C_G=5$, our method suffers an over-fitting issue while DP-SGD performs considerably poorer. Thus one should be cautious when selecting the gradient clipping values  $C_G$.  
% For those small privacy $\varepsilon\in\{0.8,1.2,2.4\}$, we observe that the validation accuracy degrades as $T$ is set larger (for ADP with $0.8$,  we have $56.38\%$, $54.2\%$ and $51.71\%$ respectively for epoch $30,120, 200$). However, our methods
% for thos ADP-SGD is more robust to epochs from \cref{table:cifar-decay} that, under the privacy budget $\varepsilon=1.6$, the best test error for DP-SGD is $42.85\%\pm0.31\%$ at the $30^{th}$ epoch, while, for APD-SGD, it is $42.38\%\pm0.35\%$ at the $200^{th}$ epoch. 

% \lingxiao{should we emphasize our results in table?}

% \vspace{-0.15cm}
\subsection{ADP-SGD v.s. DP-SGD for adaptive stepsizes}
% \vspace{-0.15cm}
In this section, we focus on understanding the optimality of the theoretical guarantees of \cref{prop:adp-adagrad-norm}; we study the numerical performance of
% illustrate the theoretical findings of \cref{sec:adaptive}, which analyzed 
ADP-SGD with stepsizes updated by the gradients. We notice that, at the beginning of the training, the gradient norm in our model lies between $0.0001$ and $0.001$ when $C_{G}=1.0$. To remedy this small gradient issue, we let $b_t$ follow a more general form: $b^2_{t+1}=b^2_{t}+\max{\{\beta_t\|\nabla f(\theta_{t};x_{\xi_t})\|^2,10^{-5}\}}$ with $\beta_t> 1$.  %\lingxiao{please rewrite this reason}.   and 
% For the parameter $\beta_t$, 
% To do this, we tune an additional hyper-parameter $\beta_t>1$ compared to the expression given in \cref{thm:adp-adagrad-norm}, which lacks this parameter. 
Specifically, we set $\beta_t=\max\{\beta/((t \mod 195)+1)),1\}$ with $\beta$ searching in a set $\{1, 512, 1024, 2048, 4096, 8192\}$.\footnote{This set for $\beta$ is due to the values of gradient norm as mentioned in the main text. These elements cover a wide range of values that the best test errors are doing as good as or better than the ones given in \cref{table:cifar-decay}.} See \cref{sec:adaptive-stepsize-compare} for a detailed description. As mentioned in \cref{sec:adaptive}, we set $\alpha_t^2=\sqrt{b_0^2+tC}$ in advance with $b_0^2=20$, and choose $C\in\{10^{-5},10^{-4},0.001,0.01,0.1,1\}$. We consider the number of iterations to be $T=11700$ with the gradient clipping $1.0$ and $2.5$. %, which correspond to 60 and 120 epochs, respectively. 
 \cref{table:cifar-adptive} summarizes the results of DP-SGD and ADP-SGD with the best hyper-parameters. 
% We run one set of these experiments, select the best hyper-parameters both DP-SGD and ADP-SGD, and then use these tuned hyper-parameters to perform six new experimental trials. The results are displayed in \cref{table:cifar-adptive}.
% Note that we will not compare with Table 1 and Table 2
\begin{table}[tb]
\hspace{-0.2cm}
\small
\caption{\small \textbf{Errors of ADP-SGD vs. DP-SGD with adaptive stepsizes.} This table reports \emph{accuracy} with the mean and the corresponding standard deviation over five independent runs. The value inside the bracket is the highest accuracy over the five runs. Each entry is the best value over 36 pairs of $(\beta,C)$ for ADP-SGD and 6 values of $\beta$ for DP-SGD. See the corresponding $(\beta,C)$ in Table~\ref{table:cifar-adptive-beta-c}. The results indicate that when using adaptive stepsizes, ADP-SGD with various $C$  performs better than DP-SGD.
% \todo{this sounds strange -- you're reporting {\em test} error $\pm$ std of \textit{validation} error?}
% \todo{It sounds strange to say $C$ is a tuning parameter but you're just taking a set of $C$'s and computing mean and std over that set.}
}\label{table:cifar-adptive}
% \begin{tabular}{l|l|cccc}
% \hline
% Iteration / Epoch                          & Algorithms   & $\varepsilon=1.6$ & $\varepsilon=3.2$ &  $\varepsilon=6.4$&   $\varepsilon=12.8$\\  \hline\hline
% \multicolumn{1}{l|}{\multirow{2}{*}{$11700$ / $60$}}    &  DP-SGD    & $43.05  \pm 0.58 $ & $38.18 \pm   0.33$ & $  35.05 \pm 0.27  $ & $ 30.29 \pm  0.25 $  \\
% \multicolumn{1}{l|}{}                        &  ADP-SGD   & $ 42.74 \pm  0.59 $ & $ 37.96 \pm  0.51 $ & $ 35.06 \pm 0.22 $ & $  30.29\pm  0.11 $ \\ \hline
% \multicolumn{1}{l|}{\multirow{2}{*}{$23400$ / $120$}}    &  DP-SGD    & $43.49 \pm 0.53    $ & $37.99\pm 0.42   $ & $ 33.85\pm 0.20  $ & $ 31.08 \pm 0.19 $ \\     
% \multicolumn{1}{l|}{}                        &  ADP-SGD   & $ 43.14 \pm  0.37 $ & $  37.91\pm 0.27  $ & $  33.81\pm0.25   $ & $ 30.92 \pm  0.16$ \\ \hline
% \end{tabular}
% % \\
\scalebox{0.9}{
\begin{tabular}{l|l|cc}
\hline
            $C_G$            & Alg   & $\bar{\varepsilon}=0.8$  & $\bar{\varepsilon}=1.6$\\  \hline\hline
 \multicolumn{1}{l|}{\multirow{2}{*}{$1.0$}} &  ADP   & $56.68 \pm 0.646$ $(57.65) $ & $62.09 \pm 0.346$ $(62.57) $ \\              
 \multicolumn{1}{l|}{}   &  DP    & $56.24 \pm 0.535$ $(57.02) $ & $62.02 \pm 0.264$  $(62.33) $\\\hline
\multicolumn{1}{l|}{\multirow{2}{*}{$2.5$}}    &  ADP   & $56.27 \pm 0.174$  $(56.46) $ & $62.38 \pm 0.428$ $(62.86) $ \\
\multicolumn{1}{l|}{}                        &  DP   & $55.65 \pm 0.448$ $(55.98) $ & $62.23 \pm 0.238$ $(62.62) $ \\ \hline
\end{tabular}}
\scalebox{0.9}{
\begin{tabular}{l|l|cc}
\hline
            $C_G$            & Alg  & $\bar{\varepsilon}=3.2$ &  $\bar{\varepsilon}=6.4$\\  \hline\hline
 \multicolumn{1}{l|}{\multirow{2}{*}{$1.0$}}                       &  ADP   & $64.51 \pm 0.100$ $(64.61) $ & $67.75 \pm 0.171$ $(67.91) $\\              
 \multicolumn{1}{l|}{}   &  DP     & $64.33 \pm 0.329$ $(65.03) $ & $67.42 \pm 0.141$ $(67.7)  $\\ \hline
\multicolumn{1}{l|}{\multirow{2}{*}{$2.5$}}    &  ADP   & $64.29 \pm 0.408$ $(64.85) $ & $67.55 \pm 0.156$ $(67.77) $\\
\multicolumn{1}{l|}{}                        &  DP   & $64.26 \pm 0.140$ $(64.39) $ & $66.23 \pm 0.367$ $(66.62)$\\ \hline
\end{tabular}
}%\vspace{-0.26cm}
\end{table}

% %\vspace{-0.2cm}
\section{Related work}
% %\vspace{-0.2cm}
\paragraph{Differentially private empirical risk minimization.} Differentially Private Empirical Risk Minimization (DP-ERM) has been widely studied over the past decade. Many algorithms have been proposed to solve DP-ERM including objective perturbation \citep{chaudhuri2011differentially,kifer2012private,iyengartowards}, output perturbation \citep{wu2017bolt,zhang2017efficient}, and gradient perturbation \citep{bassily2014private,wang2017differentially,jayaraman2018distributed}.  While most of them focus on convex functions,  we study DP-ERM with nonconvex loss functions. As most existing algorithms achieving differential privacy in ERM are based on the gradient perturbation \citep{bassily2014private,wang2017differentially,wang2019efficient,zhou2020private}, we  thus study gradient perturbation.%...and cut the rest?}and we also study the gradient perturbation based algorithms in our paper. 

%  and develop our differentially private algorithms based on the gradient perturbation.  
% To achieve differential privacy in nonconvex ERM, most existing algorithms are based on the gradient perturbation
% Compared to the fruitful results and empirically sound performance for DP-SGD with vanilla random noise, the non-constant random noise for DP-SGD has not yet been well-studied.  

\paragraph{Non-constant stepsizes for SGD and DP-SGD.} To ease the difficulty of stepsize tuning, we could apply polynomially decaying stepsize schedules \citep{NEURIPS2019_2f4059ce} or adaptive gradient methods that update the stepsize using the gradient information \citep{duchi2011adaptive,mcmahan2010adaptive}. We called them adaptive stepsizes to distinguish our adaptive deferentially private methods. These non-private algorithms update the stepsize according to the noisy gradients, and achieve favorable convergence behavior \citep{levy2018online,orabona18,ward2019adagrad,reddi2021adaptive}. 

Empirical evidence suggests that differential privacy with adaptive stepsizes could perform almost as well as -- and sometimes better than -- DP-SGD with well-tuned stepsizes. This results in a significant reduction in stepsize tuning efforts and also avoids the extra privacy cost \citep{bu2020deep,zhou2020towards,zhou2020private}.  Several works \citep{lee2018concentrated,pmlr-v108-koskela20a} also studied the nonuniform allocation of the privacy budget for each iteration. However, \cite{lee2018concentrated} only proposes a heuristic method and the purpose of \cite{pmlr-v108-koskela20a} is to avoid the need for a validation set used to tune stepsizes.
% nonuniform allocation of the privacy budget for each iteration is not a new idea as it has been studied in \cite{lee2018concentrated,pmlr-v108-koskela20a}. The purpose of \cite{pmlr-v108-koskela20a} is to avoid the need for validation set use while the motivation of \cite{lee2018concentrated} seems contradict to ours as explained in the introduction. 
In this work, we emphasize the optimal relationship between the stepsize and the variance of the random noise, and aim to improve the utility guarantee of our proposed method. 
% which we provides rigorously proof for any arbitrary stepsizes. 
% Adaptive gradient methods are \todo{would it be better to say they adapt the stepsize to the gradient?} 
 % \me{need to add some reference}
%\vspace{-0.25cm}
\section{Conclusion and future work}\label{sec:future-work}
%\vspace{-0.25cm}
In this paper, we proposed an adaptive differentially private stochastic gradient descent method in which the privacy mechanisms can be optimally adapted to the choice of stepsizes at each round, and thus obtain improved utility guarantees over prior work. Our proposed method has not only strong theoretical guarantees but also superior empirical performance. Given high-dimensional settings with only a fixed privacy budget available, our approach  with a decaying stepsize schedule shows an improvement in convergence by a magnitude $\mathcal{O}(d\log(T)\sqrt{T}/n^2)$ or  a  factor with $\mathcal{O}(\log(T))$ relative to DP-SGD.

Note that the sequence $\{\alpha_t\}$ has to be fixed before the optimization process begins, as our method require that the variance $\sigma^2$ for some privacy budget $\varepsilon$ depends on the $\{\alpha_t\}$ (\cref{thm:privacy}). However, our theorem suggests that the optimal choice of $\alpha_t$ depends on the stepsize (\cref{thm:optimal-bound}), meaning that we have to know the stepsizes a priori; this is not possible for those stepsizes updated on the fly, such as AdaGrad \citep{duchi2011adaptive} and Adam \citep{kingma2014adam}. Thus, one potential avenue of future work is to see whether $\{\alpha_t\}$ can be updated on the fly in line with AdaGrad and Adam while maintaining a predefined privacy budget $\varepsilon$. Other future directions can be related to examining more choices of $\alpha_t$ given $b_t$.  As mentioned in the main text, the relation $\alpha_t^2=b_t$ is not the unique setting to achieve the improved utility guarantees. A thorough investigation on $\alpha_t$ and $b_t$ with various gradient clipping values would therefore be an interesting extension. 
Finally, our adaptive differential privacy is applied only to a simple first-order optimization; generalizing to variance-reduced or momentum methods is another potential direction. 

% accelerated methods such as would be valuable.  \me{@lingxiao}

% \todo{I find this difficult to follow. First, we say the method is adaptive, bu so far there is no sign of adaptivity -- only that the $\alpha_t$s and $b_t$s can be arbitrary. For instance, take Thm 5.1 -- this suggests I need to know all the $\alpha_t$s before starting my method in order to choose $\sigma^2$, but if I'm choosing them adaptively, I wouldn't know them all up front. I feel like in order for this to be adaptive, we'd need to set some $\sigma$ and some privacy budget $\varepsilon$, and then at each iteration $t$, can choose $\alpha_t$ (or $b_t$ until the privacy budget is exhausted and then exit. What am I missing?}

\section*{Acknowledgments}
\label{sec:ack}
This work is funded by AFOSR FA9550-18-1-0166, NSF DMS-2023109, and DOE DE-AC02-06CH11357.
\bibliography{ref.bib}

\appendix
\newpage

\section{Privacy guarantees and convergence of DP-SGD}\label{sec:append-basic}
% \subsection{Private SGD with $\alpha_{t+1}=1$ } 
 With the preliminaries given in \cref{sec:privacy}, we will briefly  summarize the analysis of privacy guarantees for the standard differentially private stochastic gradient descent (DP-SGD) described in \cref{alg:privacy-general} with $\alpha_{t}=1, \forall t\in[T]$. To make our algorithm more general and suitable to the practice where we select $m<n$ samples instead of selecting a single sample for each iteration \cite{goyal2017accurate,pmlr-v107-wang20a}, we restate the DP-SGD algorithm with $m$ random samples in \cref{alg:privacy-general-dpsgd}. This $m$ is called size of mini-batch. Denote $\mathcal{B}_i=\{x_{i_1}, \ldots,x_{i_{m}}\}$ for the $i$-th mini-batch where $\mathcal{B}_i\cap\mathcal{B}_j=\emptyset$ and each element in $\{i_k\}$ is chosen uniformly in $[n]$ without replacement. In our experiments, $m=256$ and $n=50000$ for CIFAR10 (see \cref{sec:exp} for details). For the rest of this section, we focus on analysis of \cref{alg:privacy-general-dpsgd}.
 
 \begin{minipage}[t]{1.\textwidth}
 \vspace{-.6cm}
\begin{algorithm}[H]
\caption{\textbf{DP-SGD} with mini-batch size $m$}
\label{alg:privacy-general-dpsgd}
\begin{algorithmic}[1]
	    \State Input $\theta_0,b_0$ and $\eta$, $m<n/2$.
	    \For {$ t=1,\ldots,T$} 
	       \State prepare mini-batches $\mathcal{B}_1,\mathcal{B}_2,\ldots,\mathcal{B}_{\lceil n/m\rceil}$ such that  $\mathcal{B}_i\cap\mathcal{B}_j=\emptyset$ for $i\ne j$ and $|\mathcal{B}_i|=m$
	      \State get $\xi_t\sim \text{Uniform}(1,...,\lceil n/m\rceil)$  and $c_t\sim \mathcal{N}(0,\sigma I)$
	    \State update $b_{t+1}=\phi_1\left(b_t,\frac{1}{|\mathcal{B}_{\xi_t}|}\sum_{i\in\mathcal{B}_{\xi_t}}\nabla f(\theta_{t};x_{i})\right)$	 
	    \State release gradient $g^b_t = 
	   \frac{\eta}{b_{t+1}}( \frac{1}{|\mathcal{B}_{\xi_t}|}\sum_{i\in\mathcal{B}_{\xi_t}}\nabla f(\theta_{t};x_{i})+c_t)$  
	   % \State $G_{t}= ,$ where
	    \State update $\theta_{t+1}=\theta_{t}- g^b_t $
	    \EndFor
\end{algorithmic}
\end{algorithm}
\end{minipage}   
 
\cref{thm:privacy-dp-sgd} presented below has been well studied in prior work \cite{bassily2014private,song2013stochastic,pmlr-v107-wang20a}. We stated here for the completeness of the paper and to clarify the constant in the expression of $\sigma^2$.
      \begin{lemma}[\textbf{Privacy Amplification via Sampling} \cite{kasiviswanathan2011can}]\label{lem:sampling}
    Let the mechanism $\mathcal{M}:\mathcal{D}\to \mathcal{R}$ be $(\varepsilon,\delta)$-DP. Consider $\mathcal{M}_q$ follows the two steps (1) sample a random $q$ fraction of $\mathcal{D}$ (2) run $\mathcal{M}$ on the sample.
   Then the mechanism $\mathcal{M}_q$ is $((e^{\varepsilon}-1)q, q\delta)$-DP.\footnote{The amplification by subsampling, a standard tool for SGD analysis \cite{bassily2014private}, is first appear in \cite{kasiviswanathan2011can}. The proof can be also found here \url{http://www.ccs.neu.edu/home/jullman/cs7880s17/HW1sol.pdf} }
%   \todo{notation is confusing. ${\cal D}$ is a collection of datasets. So you are choosing $q$ datasets? But a mechanism can only be applied to a single dataset according to the definition above. Something is off here.}
    \end{lemma}
\begin{lemma}[\textbf{Advanced Composition} \citep{dwork2006our}]\label{lem:advanced}
For all $\varepsilon_0, \delta_0, \delta' >0$, let $\mathcal{M} = (\mathcal{M}_1, \hdots, \mathcal{M}_k)$ be a sequence of $(\varepsilon_0, \delta_0)$-differentially private algorithms. Then, $\mathcal{M}$ is $(\varepsilon, \delta)$-differentially private, where $\varepsilon = \varepsilon_0 \sqrt{2k \log(1/\delta')} + k \varepsilon_0 \frac{e^{\varepsilon_0}-1}{e^{\varepsilon_0}+1}$ and $\delta = 1-(1-\delta_0)^k + \delta'$. 
\footnote{In Theorem III.3 \cite{dwork2010boosting}, $\delta=k\delta_0+\delta'$, a further simplification of  $\delta = 1-(1-\delta_0)^k + \delta'$} %since $ 1-(1-\delta_0)^k\leq k\delta_0$. But since $k$ is total iteration number for DP-SGD, we keep  $\delta = 1-(1-\delta_0)^k + \delta'$. }
\end{lemma}

\begin{thm}[\textbf{Privacy Guarantee for DP-SGD}] \label{thm:privacy-dp-sgd}
Suppose the sequence $\{\alpha_t\}_{t=1}^T$ is all constant $1$ and that Assumption \ref{asmp:G-bound} holds. \cref{alg:privacy-general-dpsgd} satisfies $(\varepsilon,\delta)$-DP if the random noise $c_t$ has variance
 \begin{align}
 \sigma^2  &= \frac{(16G)^2B_{\delta}T}{n^2  \varepsilon^2  }  \text{ with }  B_{\delta}= \log(16Tm/n\delta))\log(1.25/\delta), \label{eq:key-dp-sgd}%\label{thm:privacy-dp-sgd}
\end{align}
where $T\geq n^2\varepsilon^2B_{\delta}/(8m^2\log(1.25/\delta_0))$.
\end{thm}
\begin{proof}
The proof can be summarized into three steps \cite{bassily2014private}:
\begin{itemize}
    \item \textit{Step One.} 
By Assumption \ref{asmp:G-bound}, the gradient sensitivity of the loss function is % \[\sup_{w,\mathcal{S}\subset\mathcal{D},\mathcal{S}'\subset\mathcal{D}'\\|\mathcal{S}|=|\mathcal{S}'|=m} \left\|\frac{1}{m}\sum_{x\in\mathcal{S}}\nabla f(w;x_i)-\frac{1}{m}\sum_{x'\in\mathcal{S}'}\nabla f(w;x')\right\|\leq  2G\] % as $\|\frac{1}{m}\sum_{i\in\mathcal{B}}\nabla f(w;x_i)\|^2\leq \frac{m}{m^2}\sum_{i\in\mathcal{B}}\|\nabla f(w;x_i)\|^2\leq G^2$.
\[\Delta = \frac{1}{m}\sup_{x\in \mathcal{D},x\in \mathcal{D'}}\left\|\nabla f(w;x)-\nabla f(w;x')\right\|\leq  2G/m.\]
Given the privacy budget $\varepsilon$, we apply Gaussian mechanism with $\sigma$ define in \eqref{eq:key-dp-sgd}. By Lemma \ref{def:gaussian}, the Gaussian privacy mechanism given in \cref{prop:mechasim} with $\alpha_t=1$ satisfies that $\varepsilon_0= {2G/m\sqrt{2 \log(1.25/\delta_0)}}/{\sigma}$ for some $\delta_0\leq 10^{-5}$ such that $\delta_0Tm/n\ll 0.1$ \cite{gautam2020}. As we have $T\geq n^2\varepsilon^2B_{\delta}/(32m^2\log(1.25/\delta_0))$, the privacy  
 {\small \[\varepsilon_0^2= {8G^2\log(1.25/\delta_0)n^2\varepsilon^2}/({16^2G^2m^2B_{\delta}T})=n^2\varepsilon^2\log(1.25/\delta_0)/(32m^2B_{\delta}T)\leq 1.\] }
    % \begin{lemma}(\textbf{Gaussian DP})\label{lem:gaussian} 
    % The Gaussian mechanism is $(\varepsilon, \delta)$-differentially private if $\sigma = {\sqrt{2\ln(1.25/\delta)}\Delta}/{\varepsilon}$ where $\Delta = \sup_{X,X'} \|h(X)-h(X') \|$.
    % \end{lemma}
    % \vspace{0.3cm}
    \item \textit{Step Two.} 
   Applying amplification by sub-sampling (i.e. \cref{lem:sampling}), We have $(\varepsilon_p,\delta_p)-DP$ for each step in DP-SGD with $\varepsilon_p= 2\varepsilon_0m/n\geq (e^{\varepsilon_0}-1)m/{n}$ and $\delta_p={\delta_0}m/{n}$, since $p=m/n<0.5$ and $\varepsilon_0\leq 1$. 
% \vspace{0.3cm}
 \item \textit{Step Three.} After Step One and Two, we apply advanced (strong) composition stated in Lemma \ref{lem:advanced} (Theorem III.3 \cite{dwork2010boosting}
%  \footnote{If $\mathcal{M}_1,\mathcal{M}_2,...\mathcal{M}_k$ are $(\varepsilon_0,\delta_0)-DP$, then $\mathcal{M}_{1:k}$ is $(\varepsilon,k\delta_0+\delta)-DP$ for $\varepsilon=\varepsilon_0\sqrt{2k \log(1/\delta)}+k\varepsilon_0 (exp(\varepsilon_0-1))$} 
 or Theorem 3.20 in \cite{dwork2014algorithmic}) for the $T$ iterations. Then, Algorithm 1 follows $(\varepsilon_{dpsgd},\delta_{dpsgd})$-DP  satisfying for some $\delta'$ such that $\delta_0Tm/(0.25n)\leq \delta'\leq \delta_0Tm/(0.1n)\ll 1 $.%$\delta'^2\leq {\delta_0}/{2}-\delta_0^2T^2/n^2$ %  \todo{for what choice of $b_t$ and $b_t'$?}
   \begin{align}
   \delta_{dpsgd} =&1-(1-\delta_0m/n)^T + \delta' \overset{(a)}{\leq} \delta_0Tm/n+\delta'\leq 1.25\delta' \leq 12.5\delta_0Tm/n,\label{eq:delta} \\%\leq 1.25\delta' \label{eq:delta} \\
   \varepsilon_{dpsgd}=&\varepsilon_p \sqrt{2T \log(1/\delta')} + T \varepsilon_p \frac{e^{\varepsilon_p}-1}{e^{\varepsilon_p}+1}, \label{eq:epsilon}
   \end{align}
where $(a)$ follows from the fact that $ 1-(1-\delta_0m/n)^k\leq \delta_0mk/n$. 
 We now simplify \eqref{eq:epsilon}
%   for some $\delta'\in (1/(2(1+T^2/n^2)),n/4T)$ for sufficiently large $T\ n$.
%   In addition to  $\delta_0T/n\ll1$,  we also need $\delta_0\geq (n^2/2T^2)$ so that $\delta'^2\leq {\delta_0}/{2}-\delta_0^2T^2/n^2\leq 1$. For sufficient large $T$,  $ (n^2/2T^2)\leq \delta_0\leq {n}/{T}\ll 1$ is always true.
%  We simplify \eqref{eq:delta} with the fact that  $\delta_0T/n\ll1$ 
%  \small{
%   \begin{align}
%       \delta_{dpsgd}^2 \leq & ({\delta_0T}/{n} + \delta')^2  \leq 2{\delta_0^2T^2}/{n^2}+2\delta'^2 \leq \delta_0 \quad \Rightarrow  \quad 1/\delta_{dpsgd}\geq 1/\sqrt{\delta_0}  \label{eq:delta1}
%   \end{align}}
%   Note that $\delta'^2\leq {\delta_0}/{2}-\delta_0^2T^2/n^2$ implies $\delta_0\in \left[\frac{1-\sqrt{1-16\delta'^2T^2/n^2}}{4T^2/n^2},\frac{1+\sqrt{1-16\delta'^2T^2/n^2}}{4T^2/n^2}\right]$, thus 
%   \small{
%  \begin{align}
%   \delta_{dpsgd}^2 &\leq\delta_0\leq\frac{1+\sqrt{1-16\delta'^2T^2/n^2}}{4T^2/n^2} \overset{(a)}{\leq} \frac{n^2}{2T^2}
%   %=\frac{4\delta'^2}{(1-\sqrt{1-16\delta'^2T^2/n^2})} \leq \delta'
%   \Rightarrow  1/\delta_{dpsgd}\geq 1/\sqrt{\delta'}
%   \label{eq:delta2}
%  \end{align}}
%  where (a) follows from the fact that $1+\sqrt{1-16\delta'^2T^2/n^2}\leq2$.
   \begin{align*}
\varepsilon_{dpsgd}=&\varepsilon_p \sqrt{2T \log(1/\delta')} + T \varepsilon_p \frac{e^{\varepsilon_p}-1}{e^{\varepsilon_p}+1} \\
\overset{(a)}{\leq} & \varepsilon_p(\sqrt{2T \log({1}/{\delta'})}+T\varepsilon_p)\\ 
\overset{(b)}{\leq} & 2(2{\varepsilon_0}m/{n}) \sqrt{2T \log({1}/{\delta'})}\\
\overset{(c)}{=} &  \frac{8mG/m\sqrt{2\log(\frac{1.25}{\delta_0})2T\log({1}/{\delta'})}}{  n  \sigma}\\
\overset{(d)}{\leq} &  \frac{16G\sqrt{T\log(16Tm/(n\delta_{dpsgd}))\log(1.25/\delta_{dpsgd})}}{  n  \sigma},
   \end{align*}

where (a) follows from $ ({e^{\varepsilon_p}-1})/({e^{\varepsilon_p}+1})\leq \varepsilon_p(1+{\varepsilon_p})/(2+\varepsilon_p)\leq \varepsilon_p$ as $1+{\varepsilon_p}\leq e^{\varepsilon_p}\leq 1+\varepsilon_p+\varepsilon_p^2$ for $\varepsilon_p<1$; (b) is due to that  $T\varepsilon_p\leq  \sqrt{2T \log(1/\delta')}$ which is derived from
\begin{footnotesize}
\begin{align*}
T^2\varepsilon_p^2\leq \frac{T^2\varepsilon_0^2m^2}{n^2}\leq\frac{8T(G/m)^2m^2\log(1.25/\delta_0)}{n^2\sigma^2}=T^2 \frac{8G^2\log(1.25/\delta_0)\varepsilon^2}{64G^2B_{\delta}T}\leq 2T\log(1/\delta'),
\end{align*}
\end{footnotesize}
where the last inequality is from $\varepsilon^2\leq\frac{16\log(1/\delta')B_{\delta}}{\log(1/\delta_0)}$; (c) follows by substituting the $\varepsilon_0={2G\sqrt{2 \log(1.25/\delta_0)}}/{\sigma}$ given in Step One;  (d) follows from the \eqref{eq:delta}.

Now we let $\delta_{dpsgd}=\delta$ and compare the relationship between $\varepsilon_{dpsgd}$ and $\varepsilon$:
   \begin{align*}
\varepsilon^2_{dpsgd}
&{\leq}   \frac{(16G)^2{T\log(16Tm/(n\delta_{dpsgd}))\log(1.25/\delta_{dpsgd})}}{  n^2  \sigma^2} \\
&= \frac{(16G)^2{T\log(16Tm/(n\delta_{dpsgd}))\log(1.25/\delta_{dpsgd})}}{  n^2 } \frac{n^2  \varepsilon^2  } {(16G)^2B_{\delta}T}\\
&= \frac{\varepsilon^2\log(16Tm/(n\delta_{dpsgd}))\log(1.25/\delta_{dpsgd})}{ \log(16Tm/(n\delta))\log(1.25/\delta)} \\
&\leq\varepsilon^2.
   \end{align*}
\end{itemize}
Thus, setting the $\sigma$ in \cref{eq:key-dp-sgd} is sufficient to obtain an $(\varepsilon,\delta)$-DP algorithm.

\end{proof}

%  \begin{lemma}(\textbf{DP-SGD with constant learning rate})\footnote{ Theorem 7 in \cite{pmlr-v107-wang20a}. Note that using their methods particularly setting $\sigma =0$ for $A_\sigma$ in DP-LSSGD of  \cite{pmlr-v107-wang20a}, one can obtain the bound.} Assume A1--A3, set the hyper-parameters in \cref{alg:privacy-general}   with $b_{t+1}=\alpha_{t+1}=1$, $\eta<\frac{2}{L}$ and $\sigma = \frac{8G\log^2(1.25/\delta_{SGD})}{  n \varepsilon_{SGD}}$  for some $T$. Then the averaged gradient follows
%  \begin{align}
%      \frac{1}{T} \sum_{k=0}^{T-1} \mathbb{E}\left[ \|\nabla F(\theta_k)\|^2\right] \leq \frac{2}{\eta T}  (F(\theta_0)-F^*)+\eta L G^2 \left(1+d\frac{64\log^2(1.25/\delta_{SGD}) T}{  n^2\varepsilon_{SGD}^2 }\right)\label{eq:constant}
%  \end{align}
% Set  $T = C_1(2D_F+LG^2)n^2\varepsilon_{SGD}^2/(dLG^2\log(1/\delta_{SGD}))$   where  $D_F = F(\theta_0)-F^*$ and $C_1$ is a universal constant, and $\eta=\frac{1}{\sqrt{T}}$.  The bound in \eqref{eq:constant} is tight with respect to $\eta$ and $T$:
%   \begin{align}
%      \frac{1}{T} \sum_{k=0}^{T-1} \mathbb{E}\left[ \|\nabla F(\theta_k)\|^2\right] \leq \mathcal{O}\left( \frac{ G \sqrt{dL(2D_F+LG^2)} \log(1/\delta_{SGD})}{n\varepsilon_{\text{SGD}}} \right). \label{eq:optimal-constant}
%  \end{align} \label{lem:optimal-hyper} 
% \end{lemma}
 \begin{prop}(\textbf{DP-SGD with constant stepsizes})\footnote{  Setting $\sigma =0$ for $A_\sigma$ for Theorem 7 in \cite{pmlr-v107-wang20a} reduces to our bound.} \label{thm:con-dp-sgd}
 Under the  conditions of \cref{thm:optimal-bound} on $f$. Set $\sigma^2$ satisfying \eqref{eq:key-dp-sgd} in \cref{thm:privacy-dp-sgd}. Let $b_t= 1$ in \cref{alg:privacy-general-dpsgd} 
%\label
 and denote $  \tau=\arg\min_{t \in [T-1]} \mathbb{E} [ \|\nabla F(\theta_{t})\|^2]$ and $B_{\delta}= \log(16Tm/n\delta))\log(1.25/\delta)$. Then the gradients follow 
 \begin{footnotesize}
 \begin{align}
 \mathbb{E} [\|\nabla F(\theta_\tau)\|^2] \leq \frac{2D_F}{\eta T}+\eta L G^2 \left(1+d\frac{16^2B_{\delta} T}{  n^2\varepsilon^2 }\right).\label{eq:constant}
 \end{align}
 \end{footnotesize}
\end{prop}
We omit the proof of the proposition as it can be found in \cite{pmlr-v107-wang20a}. In fact, the proof is strarightfoward by applying \cref{thm:optimal-bound-a} and noticing that
\[\|\frac{1}{m}\sum_{i\in\mathcal{B}}\nabla f(w;x_i)\|^2\leq \frac{m}{m^2}\sum_{i\in\mathcal{B}}\|\nabla f(w;x_i)\|^2\leq G^2.\]
Set $\eta=\sqrt{1/T}$,   \cref{eq:constant} becomes
 \begin{footnotesize}
 \begin{align}
 \mathbb{E} [\|\nabla F(\theta_\tau)\|^2] \leq \frac{2D_F+L G^2}{\sqrt{T}}+\frac{L(16G)^2B_{\delta} \sqrt{T}}{  n^2\varepsilon^2 }.\label{eq:stepsize-1-T}
 \end{align}
 \end{footnotesize}
Let us compare DP-SGDs between the constant stepsize $\eta/b_t=\sqrt{1/T}$ and the decaying  stepsize $\eta/b_t=1/\sqrt{a+ct}$. Suppose the second term introduced by the privacy mechanism dominates the bound. We see that the ratio of second term in the bound   using the decaying stepsize (i.e., \cref{eq:decrease1} in \cref{thm:theorem-adp-descrease}) to that using the constant stepsize  (i.e., \cref{eq:stepsize-1-T})   is $\mathcal{O}(\log(T)/\sqrt{c})$. Thus, if we set $\sqrt{c}=\log(T)$, the second term in both \cref{eq:stepsize-1-T} and \cref{eq:decrease1} have the same order.

Let us now compare  between DP-SGD with the constant stepsize $\eta/b_t=\sqrt{1/T}$ and ADP-SGD with $\alpha_t^2=b_t$ and the decaying  stepsize $\eta/b_t=1/\sqrt{a+ct}$. We have the ratio of second term  in the bound \cref{eq:decrease1}  to that in \cref{eq:stepsize-1-T})  is $\mathcal{O}(1/\sqrt{c})$. Thus setting $\sqrt{c}=\log(T)$ in $\eta/b_t=1/\sqrt{a+ct}$ for ADP-SGD with $\alpha_t^2=b_t$ will results in a better utility bound than DP-SGD with $\eta/b_t=\sqrt{1/T}$.

% We see that using decaying learning rate $b_t=\sqrt{(a+ct)}$ is not better than 

% From the theorem, we see that setting $T = C_1(2D_F+LG^2)n^2\varepsilon^2/(dLG^2\log(1/\delta))$ and $\eta={1}/{\sqrt{T}}$ yields a tight utility bound in terms of $T$, which is
% \[ U_{cont}=\mathcal{O}\left( { G \sqrt{dL(2D_F+LG^2)} \log(1.25/\delta)}/(n\varepsilon) \right).\]

    %   \todo{Please make it easier to compare this with Prop 6.1. Here you present bound for optimal $T$; what is bound for optimal $T$ with decaying step size? What does this tell us? }

From \cref{eq:stepsize-1-T}, we see that setting $T_{\rm opt} = C_1(2D_F+LG^2)n^2\varepsilon^2/(dLG^2\log(1/\delta))$ for some $C_1$ results in a tight bound. If we know the Lipschitz smoothness parameter $L$ for the function $F$ and the distance $D_F=F(\theta_0)-F^*$, we could obtain the $T_{\rm opt}$. However, in practice, the Lipschitz smoothness  $L$ and the  distance $D_F$ are unknown values. Estimating these parameters has become an active research area \cite{NEURIPS2020_5227fa9a,scaman2018lipschitz}.  Thus we will not discuss about the optimal value of $T$ and think it is more reasonable to keep it in the bound.

%  \me{Maybe people would say: although you don't know $L$, you  may  want to pick $\eta=\frac{d}{{n\varepsilon}}$. If this comes up, maybe we should talk about the benefit of initial large learning rate? But $\frac{d}{n\varepsilon}$ is actually large for overparameterized model...}

%  \me{not sure I should talk about this, do we still need to compare with Lemma 5.1?}\lingxiao{Should we remove the following part?}
%  The above two propositions also imply that as the complexity $T$ increases, the dominating term in \eqref{eq:decrease} or \eqref{eq:decrease1}  shifts from the first term to the last term. In particular, when $T= n^2\varepsilon^2( D_F/{\eta
%       B_1} +{\eta G^2 L }/{ 2})/(16\eta dLG^2B_{\delta})$, we obtain the utility guarantee
%  \todo{where is $U_{decay}$ defined?}
%      $U_{decay}=\mathcal{O}\left(G\sqrt{ \eta c d L B_1({D_F}/{\eta 
%       B_1} +{\eta G^2 L }/{ 2})} \log({1.25}/{\delta })/( n \varepsilon )\right)$ for ADP-SGD. Note that one might want to this utility guarantee with  DP-SGD with constant stepsize ($\alpha_t=1$) where the optimal stepsize $\eta/b_t=1/\sqrt{T}$ is applied (\cref{lem:constant}). ......
      
\section{Proof for Extended Advanced Composition Theorem}
We restate Lemma \ref{thm:advanced-comp} as follows.
\begin{lemma}[\textbf{Extended Advanced Composition}]
\label{thm:advanced-comp-a}
Consider two sequences $\{\varepsilon_i\}_{i=1}^k,\{\delta_i\}_{i=1}^k$ of positive numbers satisfying  $\varepsilon_i\in(0,1)$ and  $\delta_i\in(0,1)$. 
%  \todo{Is there an implicit constraint that since $\tilde \delta > 0$, we need $\delta' < 1-(1-\delta_1)\cdots (1-\delta_k)$?}
 Let  $\mathcal{M}_i$ be $(\varepsilon_i, \delta_i)$-differentially private for all $i \in \{1, 2, \hdots, k\}$.
 %and $\mathcal{M}_i$'s are potentially chosen adaptively.
 Then $\mathcal{M}=(\mathcal{M}_1, \hdots, \mathcal{M}_k)$ is $(\tilde{\varepsilon}, \tilde{\delta})$-differentially private for $\delta'\in(0,1)$ and
%  \begin{footnotesize}
 \begin{align*}
   \tilde{\varepsilon} = \sqrt{\sum_{i=1}^k 2\varepsilon_i^2 \log\left(\frac{1}{\delta'}\right)} + \sum_{i=1}^k \frac{\varepsilon_i(e^{\varepsilon_i}-1)}{(e^{\varepsilon_i}+1)},  \quad \quad
   \tilde{\delta} = 1 - (1- \delta_1)(1-\delta_2) \hdots (1-\delta_k) + \delta'.
\end{align*}
%  \end{footnotesize}
\end{lemma}
The result follows immediately from Theorem 3.5 of \cite{KairouzCompThm}. 
Alternative proof would be using Renyi DP. This result immediately follows by invoking Lemmas 2.6 and 2.7 of \cite{FeldmanCompLemma}. Particularly,
Lemma 2.7 in \cite{FeldmanCompLemma}) gives a composition rule for R\'enyi differential privacy, which can then be used to obtain our version of composition for $(\epsilon, \delta)$-differential privacy. 
Lemma 2.6 in \cite{FeldmanCompLemma} allows translating R\'enyi differential privacy to $(\epsilon, \delta)$-differential privacy.

% \end{proof}
\section{Proof for Section \ref{sec:adp-sgd}}
As explain in \cref{sec:append-basic}, we we will select $m\leq n/2$ samples instead of selecting a single sample for each iteration. We restate \cref{alg:privacy-general} in \cref{alg:privacy-general-a} with mini-batches $m$ variable. 
% \hspace{2cm}
\begin{algorithm}[H]
\caption{\textbf{ADP-SGD} with mini-batch size $m$}
\label{alg:privacy-general-a}
\begin{algorithmic}[1]
	    \State Input: $\theta_0,b_0,\alpha_0$, $\eta>0$ and $m\leq n/2$
	    \For {$ t=0,1,\ldots,T-1$} 
	    \State prepare mini-batches $\mathcal{B}_1,\mathcal{B}_2,\ldots,\mathcal{B}_{\lceil n/m\rceil}$ such that  $\mathcal{B}_i\cap\mathcal{B}_j=\emptyset$ for $i\ne j$ and $|\mathcal{B}_i|=m$
	      \State get $\xi_t\sim \text{Uniform}(1,...,\lceil n/m\rceil)$  and $c_j\sim \mathcal{N}(0,\sigma I)$
	    \State update $b_{t+1}=\phi_1\left(b_t,\frac{1}{|\mathcal{B}_{\xi_t}|}\sum_{i\in\mathcal{B}_{\xi_t}}\nabla f(\theta_{t};x_{i})\right)$	
	   \State update $\alpha_{t+1}=\phi_2(\alpha_t, b_{t+1})$ 
	   %\todo{vague and confusing; what are update equations? this is your main algorithm; it needs to be clear.}
	    \State   \textbf{release}  $g^{b}_t =  \frac{\eta}{b_{t+1}}\left(\frac{1}{|\mathcal{B}_{\xi_t}|}\sum_{i\in\mathcal{B}_{\xi_t}}\nabla f(\theta_{t};x_{i})+\alpha_{t+1}c_j\right)$  
	   % \State $g_{t}= ,$ where
	    \State update $\theta_{t+1}=\theta_{t}-  g^{b}_t$
	    \EndFor
\end{algorithmic}
\end{algorithm}

\subsection{Proof for Theorem \ref{thm:privacy}} \label{sec:proof-privacy}
Let us restate \cref{thm:privacy} in \cref{thm:privacy-a} for a mini-batch described in \cref{alg:privacy-general-a}.  %with more details for $T$ and $\{\alpha_t\}$.
\begin{thm}[\textbf{Privacy Guarantee}] \label{thm:privacy-a}
Suppose the sequence $\{\alpha_t\}_{t=1}^T$ is known in advance and that Assumption \ref{asmp:G-bound} holds. Denote $ B_{\delta}= \log(16Tm/n\delta))\log(1.25/\delta)$ as in \cref{thm:privacy-dp-sgd}. \cref{alg:privacy-general-a}  with $m$ satisfies $(\varepsilon,\delta)$-DP if the random noise $c_j$ has variance
 \begin{align}
 \sigma^2  &= \frac{(16G)^2B_{\delta}}{n^2  \varepsilon^2  } \; {\sum_{t=0}^{T-1}  \frac{ 1}{\alpha^2_{t+1}}  } , \label{eq:key-gen-a}
\end{align}
where $T$ is required to satisfy $\alpha_t^2\sum_{t=1}^T1/\alpha^2_t\geq n^2\varepsilon^2B_{\delta}/(32m^2\log(1.25/\delta_0))$.
\end{thm}

\begin{proof}
Similar to the proof in \cref{thm:privacy-dp-sgd}, we will follow three steps. 
\begin{itemize}
    \item \textit{Step One.} At $t$ iteration, the Gaussian privacy mechanism given in \cref{prop:mechasim} with any $\alpha_t$ satisfies that $(\varepsilon_t, \delta_0)-$DP where $\varepsilon_t= {2G/m\sqrt{2 \log(1.25/\delta_0)}}/{(\alpha_t\sigma)}$ for some $\delta_0\ll 10^{-5}$ such that $\delta_0T/n\ll0.1$. Note that the privacy  
 \begin{align}
     (\varepsilon_p^{t})^2
     &= \frac{8G^2\log(1.25/\delta_0)}{\alpha_t^2m^2 }\frac{n^2\varepsilon^2}{(16G)^2B_{\delta}\sum_{t=1}^{T}1/\alpha_t^2}\\
     &=n^2\varepsilon^2\log(1.25/\delta_0)/(32m^2B_{\delta}\alpha_t^2\sum_{t=1}^{T}1/\alpha_t^2)\leq 1,
 \end{align}
 where the last inequality is due to the fact that $\alpha_t^2\sum_{t=1}^T1/\alpha^2_t\geq n^2\varepsilon^2B_{\delta}/(32\log(1.25/\delta_0))$. 
%  \me{this might be the problem need to think}
    \item \textit{Step Two.} 
   Applying amplification by sub-sampling (i.e. \cref{lem:sampling}), We have $(\varepsilon^t_p,\delta_p)$-DP for each step in DP-SGD with $\varepsilon^t_p= 2\varepsilon_tm/n\geq(e^{\varepsilon_t}-1)m/{n}$ and $\delta_p={\delta_0}m/{n}$, since $p=m/n$ and $\varepsilon_t\leq 1$. 
% \vspace{0.3cm}
 \item \textit{Step Three.} 
  Using Lemma \ref{thm:advanced-comp} or \cref{thm:advanced-comp-a}, we have Algorithm 1 satisfying $(\varepsilon_{adpsgd},\delta_{adpsgd})$-DP for some $\delta'$ such that $\delta_0Tm/(0.25n)\leq \delta'\leq \delta_0Tm/(0.1n)\ll 1 $.
\begin{footnotesize}
\begin{align}
   \delta_{adpsgd} &=1-(1-\delta_0m/n)^T + \delta' \leq \delta_0Tm/n+\delta'\leq 1.25\delta' \leq 12.5\delta_0Tm/n,\label{eq:delta-adp} \\%\leq 1.25\delta' \label{eq:delta} \\
  \varepsilon_{adpsgd} &= \sqrt{\sum_{t=1}^{T} 2(\varepsilon_p^{t})^2 \log\left(\frac{1}{\delta'}\right)} +  \sum_{t=1}^{T} \frac{\varepsilon^t_p(e^{\varepsilon_p^t}-1)}{(e^{\varepsilon_p^t}+1)} \overset{(a)}{\leq} 2\sqrt{2\sum_{t=1}^{T} (\varepsilon_p^{t})^2 \log\left(\frac{1}{\delta' }\right)},
\end{align}\end{footnotesize}
 where (a) is due to that {\small $\sum_{t=1}^T \frac{\varepsilon^t_p(e^{\varepsilon_p^t}-1)}{(e^{\varepsilon_p^t}+1)}$ }is considerable smaller  than {\small $\sqrt{\sum_{t=1}^{T} (\varepsilon_p^{t})^2 \log\left(\frac{1}{\delta'}\right)}$}. Indeed, 
 \begin{align*}
 \sum_{t=1}^T \frac{\varepsilon^t_p(e^{\varepsilon_p^t}-1)}{(e^{\varepsilon_p^t}+1)} &\leq \sum_{t=1}^T  (\varepsilon^t_p)^2 =  \frac{4m^2}{n^2}\sum_{t=1}^T \varepsilon_t^2 =\frac{32G^2\log(1.25/\delta_0)}{n^2\sigma^2}\sum_{t=1}^T\frac{1}{\alpha_t^2}= \frac{\log(1.25/\delta_0)\varepsilon^2}{8B_{\delta}},\\
 \sqrt{\sum_{t=1}^{T} (\varepsilon_p^{t})^2 \log\left(\frac{1}{\delta'}\right)}&= \sqrt{\frac{\log(1.25/\delta_0)\varepsilon^2}{8B_{\delta}}\log\left(\frac{1}{\delta'}\right)}\overset{(a)}{\geq} \frac{\log(1.25/\delta_0)\varepsilon^2}{8B_{\delta}}\geq  \sum_{t=1}^T \frac{\varepsilon^t_p(e^{\varepsilon_p^t}-1)}{(e^{\varepsilon_p^t}+1)} ,
 \end{align*}
 where (a) is due to the fact that $ \frac{\log(1.25/\delta_0)\varepsilon^2}{8\log({1}/{\delta'})B_{\delta}}<1$. Let $\delta_{adpsgd}=\delta$. We now further simply $\varepsilon_{adpsgd}$ as follows
    \begin{align*}
\varepsilon_{adpsgd}^2\leq 8\log\left({1}/{\delta' }\right)\sum_{t=1}^{T} (\varepsilon_p^{t})^2 =\frac{\log\left({1}/{\delta' }\right)\log(1.25/\delta_0)\varepsilon^2}{B_{\delta}}\overset{(a)}{\leq }\varepsilon^2,
   \end{align*}
   where the last step (a) is due to \cref{eq:delta-adp}.
\end{itemize}

%  \cref{prop:mechasim} implies that each step for a fixed $\sigma$ we have $(\varepsilon_t, \delta_0)$-DP where for some $\delta_0>0$ 
% % $\sigma_z = \frac{( (2G)\sqrt{2 \log(1.25/\delta)}}{\varepsilon}$ 
% \[\varepsilon_t = \frac{ 2G\sqrt{2\ln(1.25/\delta_0)}}{ \sigma \alpha_{t+1}} \]
\end{proof}

\subsection{Proof for Theorem \ref{thm:optimal-bound}}\label{sec:optimal-bound}
We restate \cref{thm:optimal-bound-a} with the following theorem for a mini-batch described in \cref{alg:privacy-general-a}.

\begin{thm}
[\textbf{Convergence for ADP-SGD}]\label{thm:optimal-bound-a} Suppose we choose $\sigma^2$
% \todo{$\sigma^2$?}
- the variance of the random noise in \cref{alg:privacy-general-a} - according to \eqref{eq:key-gen-a}  in \cref{thm:privacy-a}.  Suppose Assumption \ref{asmp:l-lip}, \ref{asmp:G-bound} and \ref{asmp:sgd} hold. Furthermore, suppose $\alpha_t, b_t$ are deterministic. The utility guarantee of \cref{alg:privacy-general-a} with $\tau \deq \text{arg}\min_{k\in{[T-1]}}   \mathbb{E} [\|\nabla F(\theta_k)\|^2]$ and $B_{\delta}= \log(16Tm/n\delta))\log(1.25/\delta)$ is  
\begin{align}
\mathbb{E} \|\nabla F(\theta_\tau)\|^2
  & \leq   \frac{1}{\sum_{t=0}^{T-1}\frac{ 1}{b_{t+1}}} \left(\frac{D_F}{\eta} + \frac{\eta L}{2}\sum_{t=0}^{T-1} \frac{\mathbb{E}\left[ \|g_{\xi_t} \|^2\right] }{b^2_{t+1}}+ \frac{d(16G)^2B_{\delta}  }{2n^2\varepsilon^2 } M\right),\label{eq:adp-sgd}
  \end{align}
where 
\begin{align*}
g_{\xi_t}=\frac{1}{|\mathcal{B}_{\xi_t}|}\sum_{i\in\mathcal{B}_{\xi_t}}\nabla f(\theta_{t};x_{i}) \text{ and }    M(\{\alpha_t\},\{b_t\}) {\deq }\textstyle \sum_{t=0}^{T-1} (\alpha_{t+1}/b_{t+1})^2 \sum_{t=1}^{T-1}  1/\alpha^2_{t+1}.
\end{align*}
\end{thm}

\begin{proof}
Recall the update in Algorithm \ref{alg:privacy-general}:
\begin{align*}
\theta_{t+1}&= \theta_t -  \frac{\eta }{b_{t+1}}g_{\xi_j}  - \eta\frac{\alpha_{t+1}}{b_{t+1}} c_j.
\end{align*}

% For simplicity, we denote the stochastic gradient $g_{\xi_j}  =\nabla F( \theta_j;x_{\xi_j})$.
By Lipschitz gradient smoothness (c.f. Lemma \ref{lem:descend}):
\begin{align*}
F(\theta_{j+1}) &\leq F(\theta_{j}) - \eta\langle{ \nabla F(\theta_j),\frac{1 }{b_{j+1}}g_{\xi_j} - \frac{\alpha_{j+1}}{b_{j+1}} c_j\rangle}+\frac{ \eta^2 L}{2} \left\| \frac{g_{\xi_j}  }{b_{j+1}}+ \frac{\alpha_{j+1}}{b_{j+1}} c_j  \right\|^2, \nonumber \\
\mathbb{E}[F(\theta_{j+1})] &\leq \mathbb{E}[F(\theta_{j})] +\mathbb{E} \left[\frac{-\eta}{b_{j+1}}\langle{\nabla F(\theta_j),g_{\xi_j}\rangle} +\frac{ \eta ^2 L}{2b_{j+1}^2}\|g_{\xi_j}  \|^2  + \frac{ \eta ^2 d L \sigma^2}{2}\frac{\alpha^2_{j+1}}{b_{j+1}^2}\right]. % \label{remark2}
\end{align*}

So we have by telescoping from $j=0$ to $j=T-1$%think about $\eta_t = \frac{\eta}{\sqrt{b^2_{0}+ tc_b}}$
\begin{align*}
\mathbb{E}[F(\theta_{T-1})] \leq \mathbb{E}[F(\theta_{0})] +\sum_{j=0}^{T-1}\mathbb{E} \left[\frac{-\eta}{b_{j+1}}\langle{\nabla F(\theta_j),g_{\xi_j} \rangle} +\frac{ \eta ^2 L}{2b_{j+1}^2} \|g_{\xi_j}  \|^2  + \frac{ \eta ^2 d L \sigma^2}{2}\frac{\alpha^2_{j+1}}{b_{j+1}^2}\right].
\end{align*}
Moving the term $\sum_{j=0}^{T-1}\mathbb{E} \left[\frac{-\eta}{b_{j+1}}\langle{\nabla F(\theta_j),g_{\xi_j} \rangle}\right] $ to the left hand side gives
\begin{align*}
\sum_{j=0}^{T-1}\mathbb{E} \left[\frac{\langle{\nabla F(\theta_j),g_{\xi_j} \rangle}}{b_{j+1}}\right]
   & \leq  \frac{D_F  }{\eta}+ \frac{\eta L}{2}\mathbb{E}\left[\sum_{t=0}^{T-1} \frac{ \|g_{\xi_t} \|^2]}{b^2_{t+1}}  + \frac{ d  \sigma^2}{2}\sum_{t=0}^{T-1} \left(\frac{\alpha_{t+1}}{b_{t+1}} \right)^2 \right]\\
   & \overset{(a)}{=}  \frac{D_F}{\eta} + \frac{\eta L}{2}\sum_{t=0}^{T-1} \mathbb{E} \left[\frac{\|g_{\xi_t} \|^2}{b^2_{t+1}}\right] + \frac{\eta d L(16G)^2 B_{\delta} M}{2n^2\varepsilon^2 },
\end{align*}
where (a) follows by substituting $\sigma$ with \eqref{eq:key-gen-a} and denoting $M= \sum_{t=0}^{T-1} \frac{ \alpha^2_{t+1}}{b^2_{t+1}}  \sum_{t=0}^{T-1}  \frac{ 1}{\alpha^2_{t+1}} .$
We finish the proof by simplifying the left hand side in above inequality as follows
\begin{align}
    \sum_{j=0}^{T-1}\mathbb{E} \left[\frac{\langle{\nabla F(\theta_j),g_{\xi_j} \rangle}}{b_{j+1}}\right] \geq     \sum_{j=0}^{T-1}\frac{\mathbb{E} \left[ \|\nabla F(\theta_j)\|^2\right]}{b_{j+1}}  \geq  \min_{j\in{[T-1]}}\mathbb{E} \left[ \|\nabla F(\theta_j)\|^2\right] \sum_{j=0}^{T-1}\frac{1}{b_{j+1}}. \label{eq:sqaure-bound}
\end{align}
% \begin{align}
%   %\quad \text{ or } \quad \alpha_t\alpha_{T-(t-1)}=b_t
% \end{align}
\end{proof}
Now let us take a look at the \cref{remark:optimal}. For $M$, note  that 
\begin{align*}
 M= \sum_{t=0}^{T-1} \frac{ \alpha^2_{t+1}}{b^2_{t+1}}  \sum_{t=0}^{T-1}  \frac{ 1}{\alpha^2_{t+1}}     & \geq   \left(\sum_{t=0}^{T-1}\sqrt{ \frac{ \alpha^2_{t+1}}{b^2_{t+1}} }\sqrt{   \frac{ 1}{\alpha^2_{t+1}}}  \right)^2  =\left( \sum_{t=0}^{T-1} \frac{ 1}{b_{t+1}}  \right)^2=M_{\rm adp},
\end{align*}
where we apply the fact that  $ \|w\|^2\|v\|^2 \geq |\langle{w,v\rangle}|^2 $ and the equality holds  when $ \frac{ \alpha_{t+1}}{b_{t+1}}    =  \frac{ 1}{\alpha_{t+1}} $.  So the optimal relationship for $t = 1,2, \ldots, T$ is $\alpha^2_{t}=b_{t}.$
On the other hand, observe that
\begin{align*}
M = \sum_{t=0}^{T-1} \frac{ \alpha^2_{t+1}}{b^2_{t+1}}  \sum_{t=0}^{T-1}  \frac{ 1}{\alpha^2_{T-1-t}}     & \overset{(a)}{\geq}   \left(\sum_{t=0}^{T-1}\sqrt{ \frac{ \alpha^2_{t+1}}{b^2_{t+1}} }\sqrt{   \frac{ 1}{\alpha^2_{T-1-t}}}  \right)^2  =\left( \sum_{t=0}^{T-1} \frac{ 1}{ b_{t+1}}  \right)^2=M_{\rm adp},
\end{align*}
where the equality in (a) holds if  $\alpha_t\alpha_{T-(t-1)}=b_t$.
\section{Proofs for Section \ref{sec:dp-sgd-adaptive} }
\label{sec:proof-theorem-dp-adagrad-norm}
For this section, we restate the \cref{thm:theorem-adp-descrease} with \cref{thm:theorem-adp-descrease-a} for a more general setup --  \cref{alg:privacy-general-a} with a mini-batch of size $m$. 

\begin{prop}[\textbf{ADP-SGD v.s. DP-SGD with a polynomially decaying stepsize schedule}]\label{thm:theorem-adp-descrease-a}  
Under the  conditions of \cref{thm:optimal-bound-a} on $f$ and $\sigma^2$, let $b_t= (a+ct)^{1/2}$ in \cref{alg:privacy-general-a}, where $a>0, c>0$.  Denote $  \tau=\arg\min_{t \in [T-1]} \mathbb{E} [ \|\nabla F(\theta_{t})\|^2]$, $B_{\delta}= \log(16Tm/n\delta))\log(1.25/\delta)$ and $B_T=\log\left( 1+T{c}/{a}\right)$. %\todo{Call this $B_T$ so that in the bound below, it's clear that the second term grows faster than $\sqrt{T}$}
If we choose $T\geq 5+4{a}/{c}$, and $\alpha_t^2=b_{t}$, we have the following utility guarantee for ADP-SGD
\begin{footnotesize}
\begin{align}
 \textbf{(ADP-SGD)}\quad   \mathbb{E}  [\|\nabla F(\theta_{\tau}^{\rm ADP})\|^2] 
    &\leq
 \frac{\sqrt{c}\left(\frac{D_F}{\eta} +\frac{\eta G^2 L B_T}{2c}\right)}{\sqrt{T-1}}  +  \frac{ \eta d L(16G)^2 B_{\delta}\sqrt{T}}{n^2\varepsilon^2 \sqrt{c} }. %$T\geq 5+4 \frac{a}{c}$
\label{eq:decrease}%{thm:theorem-adp-descrease}  
\end{align}
\end{footnotesize}
In addition, if we choose $T\geq 5+4{a}/{c}$ and $\alpha_t=1$, we have the utility guarantee for DP-SGD:
\begin{footnotesize}
\begin{align}
 \textbf{(DP-SGD)}\quad \mathbb{E}  [\|\nabla F(\theta_{\tau}^{\rm DP})\|^2] 
 &\leq \frac{\sqrt{c}\left(  \frac{D_F}{\eta} +\frac{\eta G^2 L B_T}{ 2c}\right)}{\sqrt{T-1}} + \frac{ \eta d L (16G)^2B_{\delta}B_T(\sqrt{T-1}+1)  }{  2n^2 \varepsilon^2\sqrt{c} }. \label{eq:decrease1}
\end{align}
\end{footnotesize}
\end{prop}

\begin{proof}
The proof will be divided into two parts: \cref{sec:adp-decrease} is for ADP-SGD and \cref{sec:dp-decrease} is for DP-SGD.
%%%%%%%%%%%%%%%%%%%%%%%%%%%%%%%%proof1%%%%%%%%%%%%%%%%%%%%%%%%%%%%
\subsection{Proof for Proposition \ref{thm:theorem-adp-descrease-a} -- ADP-SGD with $b_t=\sqrt{a+ct}$} \label{sec:adp-decrease}
% Under the assumption A1, the left hand side in the bound \cref{eq:adp-sgd} of  Theorem \ref{thm:optimal-bound-a} becomes 
% \begin{align}
%     \sum_{j=0}^{T-1}\mathbb{E} \left[\frac{\langle{\nabla F(\theta_j),\nabla f_{\xi_j} \rangle}}{b_{j+1}}\right] \geq     \sum_{j=0}^{T-1}\frac{\mathbb{E} \left[ \nabla F(\theta_j)\right]}{b_{j+1}}  \geq  \min_{j\in{[T-1]}}\mathbb{E} \left[ \nabla F(\theta_j)\right] \sum_{j=0}^{T-1}\frac{1}{b_{j+1}} \label{eq:sqaure-bound}
% \end{align}

Note from Lemma \ref{lem:sum} we have
\begin{align}
  \frac{2}{\sqrt{c}}\left(\sqrt{\frac{a}{c}+ T}-\sqrt{\frac{a}{c}+1} \right) \leq  \sum_{j=0}^{T-1}\frac{1}{b_{j+1}} = \sum_{j=1}^{T}\frac{1}{\sqrt{a+ct}}  \leq \frac{2}{\sqrt{c}}\left(\sqrt{\frac{a}{c}+ T}-\sqrt{\frac{a}{c}} \right).\label{eq:p1-2}
\end{align}

Set $\widetilde{B_1}=\mathbb{E}\left[\sum_{t=0}^{T-1} \frac{ \|g_{\xi_t} \|^2}{b^2_{t+1}}\right]$.
Continue with  the bound \cref{eq:adp-sgd} of  Theorem \ref{thm:optimal-bound-a} with $\alpha_t^2=b_t$
\begin{align*}
   \mathbb{E} \|\nabla F(\theta_\tau)\|^2
  &\leq \frac{L }{\sum_{\ell=0}^{T-1}\frac{ 1}{b_{\ell+1}}} \left(\frac{D_F}{\eta L} + \frac{\eta}{2}\sum_{\ell=0}^{T-1} \frac{ \mathbb{E}[\|g_{\xi_\ell} \|^2]}{b^2_{\ell+1}}\right) + \frac{ \eta d L (16G)^2 B_{\delta} }{2n^2\varepsilon^2  } \left(\sum_{\ell=0}^{T-1} \frac{ 1}{b_{\ell+1}} \right)\\
   &\overset{(a)}{\leq}  \frac{D_F/{\eta} +\eta L\widetilde{B_1}/2}{\frac{2}{\sqrt{c}}\left(\sqrt{\frac{a}{c}+ T}-\sqrt{\frac{a}{c}+1}\right) }     + \frac{ \eta d L(16G)^2 B_{\delta} }{2n^2\varepsilon^2  } \left(\frac{2}{\sqrt{c}}\left(\sqrt{\frac{a}{c}+ T}-\sqrt{\frac{a}{c}} \right) \right) \\
       &\overset{(b)}{\leq} \frac{\sqrt{a+ {c}T}+\sqrt{a+c}}{2(T-1)}  \left(\frac{D_F}{\eta} +\frac{\eta L\widetilde{B_1}}{2}\right) + \frac{ \eta d L(16G)^2 B_{\delta} }{n^2\varepsilon^2  } \frac{2\sqrt{T}}{\sqrt{c}}\\
       & \overset{(c)}{\leq}   \frac{\sqrt{c}}{\sqrt{T-1}} \left(\frac{D_F}{\eta} +\frac{\eta L \widetilde{B_1}}{2}\right) +  \frac{ \eta d L(16G)^2 B_{\delta}\sqrt{T}}{n^2\varepsilon^2 \sqrt{c} }\\
             & \overset{(d)}{\leq}   \frac{\sqrt{c}}{\sqrt{T-1}} \left(\frac{D_F}{\eta} +\frac{\eta G^2 L B_T}{2c}\right) +  \frac{ \eta d L(16G)^2 B_{\delta}\sqrt{T}}{n^2\varepsilon^2 \sqrt{c} }, %$T\geq 5+4 \frac{a}{c}$
\end{align*}
where (a) is by replacing $\alpha_t^2=b_t$ in $M$; (b) follows by the fact that 
\[\frac{2}{\sqrt{c}}\left(\sqrt{\frac{a}{c}+ T}-\sqrt{\frac{a}{c}+1} \right)   = \frac{2(T-1)}{\sqrt{a+ {c}T}+\sqrt{a+c}} \quad \text{and} \quad \sqrt{\frac{a}{c}+ T}-\sqrt{\frac{a}{c}} \leq \sqrt{T}, \]
(c) is true due to the fact that  $\frac{\sqrt{a+{T}{c}}+\sqrt{a+c}}{T-1}\leq \frac{2\sqrt{a+c}}{T-1}+\frac{\sqrt{c}}{\sqrt{T-1}}\leq \frac{2\sqrt{c}}{\sqrt{T-1}}$ as $T\geq 5+4 
\frac{a}{c}$; (6) with $B_T=\log\left( 1+T{c}/{a}\right)$ is due to 
\[\widetilde{B_1}=\mathbb{E}\left[\sum_{t=0}^{T-1} \frac{ \|g_{\xi_t} \|^2}{a+c(t+1)}\right]\leq \sum_{t=0}^{T-1} \frac{G^2}{a+c(t+1)} \leq \frac{G^2}{c}\log\left(1+\frac{cT}{a}\right)=G^2B_T/c, \]
where we use Lemma  \ref{lem:sum} with $p=1$.
% \end{proof}

%%%%%%%%%%%%%%%%%%%%%%%%%%%%%%%%proof2%%%%%%%%%%%%%%%%%%%%%%%%%%%%
\subsection{Proof for Proposition \ref{thm:theorem-adp-descrease-a}   -- DP-SGD with $b_t=\sqrt{a+ct}$ }\label{sec:dp-decrease}
% \begin{proof}
Applying the fact in \eqref{eq:sqaure-bound}, \eqref{eq:p1-2}, and that $\|g_{\xi_t} \|^2\leq G^2$,  the bound \cref{eq:adp-sgd} of  Theorem \ref{thm:optimal-bound-a} with $\alpha_t^2=1$ reduces to
\begin{align*}
   \min_{k\in{[T-1]}} \mathbb{E} \|\nabla F(\theta_k)\|^2
%   &\leq  \frac{(F(\theta_0)-F^*)+L/2\left( G^2  +d\sigma^2 \right)\sum_{k=0}^{T-1} \eta_k^2 }{\sum_{k=0}^{T-1} \eta_k} \\
%   &\overset{(a)}{\leq}      \frac{\sqrt{a+ {c}T}+\sqrt{a+c}}{2(T-1)}\left( \frac{D_F}{\eta} +\frac{\eta L}{2}( G^2  +d\sigma^2)\frac{1}{c}\log\left( 1+T\frac{c}{a}\right)\right) \\
       &\overset{(b)}{\leq}  \frac{\sqrt{a+ {c}T}+\sqrt{a+c}}{2(T-1)}   \left( \frac{D_F}{\eta} +\frac{\eta  L G^2B_T}{2c} + \frac{ \eta d L (16G)^2T}{ 2c n^2 \varepsilon^2 } B_{\delta}B_T\right) \\
       & \overset{(c)}{\leq} \frac{\sqrt{c}}{\sqrt{T-1}}\left(  \frac{D_F}{\eta 
      } +\frac{\eta  LG^2 B_T}{ 2c}\right) + \frac{ \eta d L (16G)^2 }{  2n^2 \varepsilon^2\sqrt{c} } B_{\delta}B_T \frac{T}{\sqrt{T-1}},
\end{align*}
where (a) is by Lemma \ref{lem:sum} (see \eqref{eq:sqaure-bound}); (b) follows by substituting $\sigma$ and setting $B_T=\log\left( 1+T\frac{c}{a}\right)$; (c) is true due to    $\frac{\sqrt{a+{T}{c}}+\sqrt{a+c}}{T-1}\leq \frac{2\sqrt{a+c}}{T-1}+\frac{\sqrt{c}}{\sqrt{T-1}}\leq \frac{2\sqrt{c}}{\sqrt{T-1}}$ as $T\geq 5+4 
\frac{a}{c}$.
\end{proof}

\subsection{Convergence for an adaptive stepsize schedule}\label{sec:adp-adagrad-norm}

\begin{thm}[\textbf{Convergence for an adaptive stepsize schedule}]\label{thm:adp-adagrad-norm} Under the conditions of \cref{thm:optimal-bound-a} on $f$ and $\sigma^2$, let $b^2_{t+1}=b^2_{t}+\max{\left\{\frac{1}{|\mathcal{B}_{\xi_j}|}\|\sum_{i\in\mathcal{B}_{\xi_j}}\nabla f(\theta_{t};x_{i})\|^2,\nu\right\}}, \nu\in(0,G]$ in \cref{alg:privacy-general-a}. Denote $B_{\delta}= \log(16Tm/n\delta))\log(1.25/\delta)$  and $ \tau=\arg\min_{t \in [T-1]} \mathbb{E} [ \|\nabla F(\theta_{t})\|^2]$.
If $T\geq 5+4b^2_0/G^2$, then the utility guarantee follows
\begin{align*}
   \mathbb{E}  [\|\nabla F(\theta_{\tau})\|^2] &\leq  \frac{2G}{\sqrt{T-1}}\left(B_{\rm sgd}+\frac{\eta d L(16G)^2B_\delta\mathbb{E} \left[M\right] }{2n^2\varepsilon^2 } \right), 
   \end{align*}
where 
{\small $\displaystyle B_{\rm sgd}= \frac{D_F }{\eta}+ \left(2G+ \frac{\eta L}{2}\right) \left(1+ \log \left(\frac{T(G^2+\nu^2)}{b_0^2}  
+1 \right)\right)$ } and {\small  $\displaystyle M =\sum_{t=0}^{T-1} \frac{ \alpha^2_{t+1}}{b^2_{t+1}}  \sum_{t=1}^{T-1}  \frac{ 1}{\alpha^2_{t+1}}$}.
\end{thm}

When $\alpha_t=0$, \cref{thm:theorem-adp-descrease} (result in \eqref{eq:decrease} with $M=0$) and \cref{thm:adp-adagrad-norm} (with $M=0$) corresponds to the standard SGD algorithms with decaying and adaptive stepsizes, respectively.
% \cref{thm:adp-adagrad-norm} can be compared with \cref{thm:theorem-adp-descrease} when $\alpha_t=0$ (and thus $M$ is zero), which reduces to the standard SGD algorithms with decaying and adaptive stepsizes, respectively.
In particular, if we set  $a=b_0^2$, $c=G^2$ and $\alpha_t=0$ in \cref{thm:theorem-adp-descrease}, then the result in \cref{thm:theorem-adp-descrease} becomes {\small ${G\left({D_F}/{\eta} +{\eta G^2 L B_T}/{2c}\right)}/{\sqrt{T-1}}\overset{\Delta}{=}Q_{\rm decay}$}, while the result in \cref{thm:adp-adagrad-norm} is {\small $2GB_{\rm sgd}/{\sqrt{T-1}}\overset{\Delta}{=}Q_{\rm adapt}$}. We see that for a sufficiently large $G>L$, the advantage of using this variant of adaptive stepsizes is that {\small $Q_{\rm adapt}=\mathcal{O}(G^2\log(T))/\sqrt{T})$} is smaller than {\small $Q_{\rm decay}=\mathcal{O}(G^3\log(T))/\sqrt{T})$} by an order of $G$. 
Note that the proof follows closely with \cite{ward2019adagrad}.

\begin{proof}
Write $F_j = F(\theta_j)$ and {\small $g_{\xi_j}=\frac{1}{|\mathcal{B}_{\xi_j}|}\sum_{i\in\mathcal{B}_{\xi_j}}\nabla f(\theta_{t};x_{i})$ }. In addition, we write 
$\mathbb{E}_j[\cdot]$ means taking expectation with respect to the randomness of $\xi_j$ and $c_j$ conditional on $\{\xi_t\}_{t=0}^{j-1}$ and $\{c_t\}_{t=0}^{j-1}$; 
$\mathbb{E}_{c_j}[\cdot]$ means taking expectation with respect to the randomness of $c_j$ conditional on $\{\xi_t\}_{t=0}^{j-1}$ and $\{c_t\}_{t=0}^{j-1}$; $\mathbb{E}_{\xi_j}[\cdot]$ means taking expectation with respect to the randomness of $\xi_j$ conditional on $\{\xi_t\}_{t=0}^{j-1}$ and $\{c_t\}_{t=0}^{j-1}$. Note that since $c_j$ and $\xi_j$ is independent, thus we have $\mathbb{E}_j[\cdot]=\mathbb{E}_{c_j}[\cdot]\mathbb{E}_{\xi_j}[\cdot]$.

By Decent Lemma \ref{lem:descend}, 
\begin{align}
F_{j+1} &\leq F_j - \frac{\eta}{b_{j+1}}\langle{ \nabla F_j, g_{\xi_j} +\alpha_{j+1}c_j\rangle}+\frac{ \eta  ^2 L}{2b^2_{j+1}} \|g_{\xi_j}+\alpha_{j+1}c_j\|^2 \nonumber \\
&= F_{j}-\frac{ \eta \|\nabla F_{j}\|^2}{b_{j+1}} +\frac{ \eta}{b_{j+1}}\langle{ \nabla F_{j}, \nabla F_j - g_{\xi_j} \rangle} - \frac{ \eta \alpha_{j+1}  }{b_{j+1}} \langle{ \nabla F_j, c_j\rangle}\nonumber \\
& \quad +\frac{ \eta^2 L\alpha_{j+1}  }{b^2_{j+1}}\langle{ g_{\xi_j}, c_j\rangle}  +\frac{ \eta ^2 L \alpha^2_{j+1}}{2b^2_{j+1}}\|c_j\|^2+\frac{ \eta ^2 L}{2b^2_{j+1}} \|g_{\xi_j}\|^2.\label{remark2}
\end{align}
Observe that  taking expectation with respect to $c_j$,  conditional on $\xi_1, \dots, \xi_{j-1},\xi_{j}$ gives
\begin{align}
\mathbb{E}_{c_j} \left[\langle{ \nabla g_{\xi_j}, c_j\rangle}\right] = 0 \quad  \mathbb{E}_{c_j} \left[\langle{ \nabla F_j, c_j\rangle}\right] &= 0  \text{ and }
\mathbb{E}_{c_j} \left[ \|c_j\|^2\right] = d\sigma^2.
\end{align}
Thus, we have 
\begin{align}
\mathbb{E}_{c_j} \left[ F_{j+1}\right] &\leq F_j -\frac{ \eta \|\nabla F_{j}\|^2}{b_{j+1}} +\frac{ \eta}{b_{j+1}}\langle{ \nabla F_{j}, \nabla F_j - g_{\xi_j} \rangle} +\frac{ \eta  ^2 L}{2b^2_{j+1}} \left(\|g_{\xi_j}\|^2+\alpha^2_{j+1}d\sigma^2 \right).\label{eq:exp-z}
\end{align}

Note that taking expectation of $ \frac{1}{b_{j} + G} \langle{ \nabla F_j,  \nabla F_j - g_{\xi_j} \rangle}$ with respect to $\xi_{j}$  conditional on $\xi_1, c_1, \dots, \xi_{j-1}, c_{j-1}$ gives
\begin{align} 
\mathbb{E}_{\xi_j} \left[  \frac{1}{b_{j} + G} \langle{ \nabla F_j,  \nabla F_j - g_{\xi_j} \rangle} \right] &=  \frac{1}{b_{j} + G}  \mathbb{E}_{\xi_j} \left[ \langle{ \nabla F_j,  \nabla F_j - g_{\xi_j} \rangle} \right]  = 0.
%\nonumber \\
%\mathbb{E}_{\xi_j} \left[  \frac{1}{b_{j+1}} \| \nabla F_j\|\right] \geq  \frac{1}{\mathbb{E}\left[ \sqrt{ b_{j} + \|g_{\xi_j}\|^2 } right] }  \mathbb{E}_{\xi_j} \left[ \langle{ \nabla F_j,  \nabla F_j - g_{\xi_j} \rangle} \right]  = 0
\end{align}
Applying above inequalities back to the inequality \ref{eq:exp-z} becomes
\begin{align}
\mathbb{E}_{j} \left[ \frac{F_{j+1}}{\eta}\right]
% &\leq  \frac{F_{j}}{\eta}  -  \frac{  \mathbb{E}_{\xi_j} \left[\langle{ \nabla F_j,  \nabla F_j - g_{\xi_j} \rangle}\right]}{b_{j} + G}- \mathbb{E}_{\xi_j}\bigg\langle{ \nabla F_j,\frac{g_{\xi_j}}{b_{j+1}}\bigg\rangle}  +\frac{\eta L}{2} \mathbb{E}_{\xi_j} \left[ \frac{\|g_{\xi_j}\|^2+d\alpha^2_{j+1}\sigma^2}{b^2_{j+1}}\right]  \nonumber\\
&\leq \frac{F_{j}}{\eta}    - \frac{  \|\nabla F_{j}\|^2}{b_{j}+G}  + \mathbb{E}_{\xi_j} \left[  \left( \frac{1}{b_{j} + G} - \frac{1}{b_{j+1} }\right) \langle{ \nabla F_j, g_{\xi_j} \rangle} \right] +\frac{\eta  L}{2}  \mathbb{E}_{\xi_j} \left[ \frac{\|g_{\xi_j}\|^2+d\alpha^2_{j+1}\sigma^2}{b^2_{j+1}} \right]. \label{eq:first}
\end{align}
Observe the identity 
\[
\frac{1}{b_{j} + G}- \frac{1}{b_{j+1}} = \frac{ \max{ \{\|g_{\xi_j}\|^2, \nu\}}}{b_{j+1}(b_{j}+G)(b_{j}+b_{j+1})}-\frac{G}{b_{j+1}(b_{j}+G)};\]
thus, applying Cauchy-Schwarz,
\begin{align}
 \left(\frac{1}{b_{j}+G} -  \frac{1}{b_{j+1}} \right) \langle{ \nabla F_{j}, g_{\xi_j}\rangle}
&
= \left(\frac{ \max{ \{\|g_{\xi_j}\|^2, \nu\}}}{b_{j+1}(b_{j}+G)(b_{j+1}+b_j)}-\frac{G}{b_{j+1}(b_{j}+G)}\right) \langle{ \nabla F_{j}, g_{\xi_j}\rangle}
\nonumber \\
&\leq 
\frac{ \max{ \{\|g_{\xi_j}\|^2, \nu\}} \|g_{\xi_j}\|\| \nabla F_{j}\|}{b_{j+1}(b_{j+1}+b_{j})(b_{j}+G)} + \frac{G|\langle{ \nabla F_{j}, g_{\xi_j}\rangle}|}{b_{j+1}(b_{j}+G)} 
\nonumber \\
&\leq
\frac{ \sqrt{\max{ \{\|g_{\xi_j}\|^2, \nu\}}} \|g_{\xi_j}\|\| \nabla F_{j}\|}{(b_{j+1}+b_{j})(b_{j}+G)} + \frac{G\| \nabla F_{j}\|\| g_{\xi_j}\|}{b_{j+1}(b_{j}+G)}. 
\label{eq:keyinq1}
\end{align}
By applying the inequality  $ a b \leq \frac{\lambda }{2}a^2 + \frac{1}{2\lambda} b^2$  with $\lambda=\frac{2G^2}{b_{j}+G}$, $a=\frac{\|g_{\xi_j}\|}{b_{j}+b_{j+1}}$, and $b= \frac{\|\nabla F_j\| \|g_{\xi_j}\| }{(b_{j}+G)}$, the first term in \eqref{eq:keyinq1} can be bounded as
\begin{align}
 \mathbb{E}_{\xi_j}\frac{\sqrt{\max{ \{\|g_{\xi_j}\|^2, \nu\}}}\|g_{\xi_j}\|\| \nabla F_{j}\| }{(b_{j}+b_{j+1})(b_{j}+G)} 
 & \leq 
  \mathbb{E}_{\xi_j}  \frac{G^2}{(b_{j}+G)}\frac{\max{ \{\|g_{\xi_j}\|^2, \nu\}}}{(b_{j}+b_{j+1})^2}
   + \mathbb{E}_{\xi_{j}} \frac{(b_{j}+G)}{4G^2}\frac{  \|\nabla F_{j}\|^2 \| g_{\xi_j}\|^2}{(b_{j}+G)^2} 
   \nonumber \\ 
 &\leq    \frac{G^2}{b_{j}+G}\mathbb{E}_{\xi_j} \left[ \frac{\max{ \{\|g_{\xi_j}\|^2, \nu\}}}{b_{j+1}^2} \right] +  \frac{(b_{j}+G)}{4G^2}\frac{  \|\nabla F_{j}\|^2 \mathbb{E}_{\xi_{j}} \left[ \| g_{\xi_j}\|^2 \right]}{(b_{j}+G)^2} \nonumber \\ 
  &\leq   G\mathbb{E}_{\xi_j} \left[ \frac{\max{ \{\|g_{\xi_j}\|^2, \nu\}}}{b_{j+1}^2} \right] +  \frac{ \|\nabla F_{j}\|^2}{4(b_{j}+G)}.  \nonumber 
\end{align}
Similarly, applying the inequality  $a b \leq \frac{\lambda}{2} a^2 + \frac{1}{2\lambda} b^2$  with $\lambda=\frac{2}{b_{j}+G}$, $a=\frac{G\|g_{\xi_j}\|}{b_{j+1}}$, and $b= \frac{\|\nabla F_j\| }{b_{j}+G}$, the second term of the right hand side in equation \eqref{eq:keyinq1} is bounded by
\begin{align}
 \mathbb{E}_{\xi_j}\frac{G\| \nabla F_{j}\|\| g_{\xi_j}\|}{b_{j+1}(b_{j}+G)}  
 & \leq  G \mathbb{E}_{\xi_{j}} \frac{\|g_{\xi_j}\|^2}{b_{j+1}^2}+ \frac{  \|\nabla F_{j}\|^2}{4(b_{j}+G)} \leq  G \mathbb{E}_{\xi_{j}} \frac{\max{ \{\|g_{\xi_j}\|^2, \nu\}}}{b_{j+1}^2}+ \frac{  \|\nabla F_{j}\|^2}{4(b_{j}+G)}. \label{eq:keyinq_b}
\end{align}
Thus, we have
\begin{align}
 \mathbb{E}_{\xi_j} \left[  \left( \frac{1}{b_j} - \frac{1}{b_{j+1} + G} \right) \langle{ \nabla F_j, g_{\xi_j} \rangle} \right] 
&\leq 2G\mathbb{E}_{\xi_j} \left[ \frac{\max{ \{\|g_{\xi_j}\|^2, \nu\}}}{b_{j+1}^2} \right] +  \frac{\|\nabla F_{j}\|^2}{2(b_{j}+G)}, 
\end{align}
and, therefore, back to \eqref{eq:first},
\begin{align}
\mathbb{E}_{\xi_j} [F_{j+1}] &\leq F_j  -\frac{ \eta \|\nabla F_{j}\|^2}{b_{j} + G}  + 2 \eta G\mathbb{E}_{\xi_j} \left[ \frac{\max{ \{\|g_{\xi_j}\|^2, \nu\}}}{b_{j+1}^2} \right]\nonumber \\
&\quad +  \frac{  \eta \|\nabla F_{j}\|^2}{2(b_{j}+G)}  + \frac{ \eta^2 L}{2} \mathbb{E}_{\xi_j} \left[ \frac{\|g_{\xi_j}\|^2}{b^2_{j+1}}+\frac{d\alpha^2_{j+1}\sigma^2}{b^2_{j+1}} \right].\nonumber 
\end{align}
We divided above inequality by $\eta$ and then move the term $\frac{ \|\nabla F_{j}\|^2}{2(b_{j} + G)}$ to the left hand side:
\begin{align}
\frac{ \|\nabla F_{j}\|^2}{2(b_{j} + G)} &\leq  \frac{F_j - \mathbb{E}_{\xi_j} [F_{j+1}] }{\eta}+ (2G+ \frac{\eta L}{2})\mathbb{E}_{\xi_j} \left[ \frac{\max{ \{\|g_{\xi_j}\|^2, \nu\}}}{2b_{j+1}^2}\right] +\mathbb{E}_{\xi_j} \left[\frac{\eta d L \alpha^2_{j+1}\sigma^2}{b^2_{j+1}}\right].\nonumber 
\end{align}
Applying the law of total expectation, we take the expectation of each side with respect to $ z_{j-1}, \xi_{j-1}, z_{j-2}, \xi_{j-2}, \dots$, and arrive at the recursion
\begin{align}
\mathbb{E} \left[ \frac{ \|\nabla F_{j}\|^2}{2(b_{j} + G)} \right] &\leq \frac{\mathbb{E}[F_j] - \mathbb{E} [F_{j+1}] }{\eta}+ (2G+ \frac{\eta L}{2})\mathbb{E} \left[ \frac{\max{ \{\|g_{\xi_j}\|^2, \nu\}}}{b_{j+1}^2}\right] + \eta d L\mathbb{E}\left[\frac{ \sigma^2\alpha^2_{j+1}}{2b^2_{j+1}}\right]. \nonumber 
\end{align}
Taking $j=T$ and summing up from $k= 0$ to $k = T-1$,
\begin{align}
\sum_{k=0}^{T-1}  \mathbb{E} \left[ \frac{ \|\nabla F_{k}\|^2}{ 2(b_{k}  +G)} \right]  &\leq\frac{ F_{0} - F^{*}}{\eta} + (2G+ \frac{\eta L}{2})\mathbb{E}\sum_{k=0}^{T-1}   \left[ \frac{\max{ \{\|\nabla f_{\xi_k}\|^2, \nu\}}}{b_{k+1}^2}\right] + \eta d L\mathbb{E} \left[\sum_{k=0}^{T-1}  \frac{ \alpha^2_{k+1}\sigma^2}{2b^2_{k+1}}\right].\label{eq:lips_sum}
\end{align}
For the second term of right hand side in inequality \ref{eq:lips_sum}, we apply Lemma \ref{lem:logsum} and then Jensen's inequality to bound the final summation:
\begin{align}
\mathbb{E}\sum_{k=0}^{T-1}   \left[ \frac{\max{ \{\|\nabla f_{\xi_k}\|^2, \nu\}}}{b_{k+1}^2} \right] &\leq  \mathbb{E} \left[1+ \log \left(1+\sum_{k=0}^{T-1}\max{ \{\|\nabla f_{\xi_k}\|^2, \nu\}}/b_0^2 \right) \right] \nonumber\\
&\leq 1+ \log \left( \frac{T (G^2+\nu)}{b_0^2}  
+1 \right)\overset{\triangle}{=}D_{1}	
\end{align}
As for term of left hand side in equation \eqref{eq:lips_sum},
% we apply H\"{o}lder's inequality,
% \[\frac{\mathbb{E}|XY|}{\left(\mathbb{E}|Y|^{3} \right)^{\frac{1}{3}} } \leq \left(\mathbb{E} |X|^{\frac{3}{2}} \right)^ {\frac{2}{3} }\quad \text{ with } X= \left(\frac{  \|\nabla F_{k}\|^2}{b_{k}+G}\right)^{\frac{2}{3}} \text{ and } Y = (b_{k}+G)^{\frac{2}{3}}, \]
we obtain
\begin{align}  
 \mathbb{E}\left[ \frac{  \|\nabla F_{k}\|^2}{ 2(b_{k}+G)}  \right]
   & \geq \frac{ \mathbb{E}  \|\nabla F_{k}\|^{2}
}{2\sqrt{b_0^2+ (k+1)G^2}}
\end{align}
since we have $b_{k}= \sqrt{b_0^2+\sum_{t=0}^{k-1}\max\{\| \nabla f_{\xi_j}\|_2^2},\nu \}\leq \sqrt{b_0^2+kG^2}$ since $\nu\leq G^2$
%  \geq  \frac{ \left( \min_{0 \leq j \leq T} \mathbb{E}  \|\nabla F_{k}\| ^{\frac{4}{3}} \right)^\frac{3}{2}
%}{2\sqrt{2(b_0 + (T+1)G^2)}}  

Thus \eqref{eq:lips_sum}  arrives at the inequality
\begin{align}
 \min_{0 \leq k \leq {T-1}} \mathbb{E}  [\|\nabla F_{k}\|^2] \sum_{k=1}^{T}\frac {1}{2\sqrt{b_0^2+ kG^2} } \leq \underbrace{\frac{ F_{0} - F^{*} }{\eta}+ (2G+ \frac{\eta L}{2}) D_{1}}_{B_{sgd}}+\eta d L\mathbb{E} \left[\sum_{k=0}^{T-1}  \frac{ \alpha^2_{k+1}\sigma^2}{2b^2_{k+1}}\right].   \label{eq:final168}
\end{align}

Divided  by $\sum_{k=1}^{T}\frac {1}{2\sqrt{b_0^2+ kG^2} } $ and replaced $\sigma$ with 
 \begin{align} 
   \sigma^2 
   &= \frac{(16G)^2B_\delta}{n^2 \varepsilon^2 } {\sum_{t=0}^{T-1}  \frac{ 1}{\alpha^2_{t+1}}  }, 
\end{align}
the above inequality \eqref{eq:final168}, for $T\geq\frac{4b^2_0}{G^2}$, results in 
\begin{align}
   \min_{\ell \in [T-1]} \mathbb{E}  \|\nabla F_{\ell}\|^2
   &\leq \frac{1}{\sum_{k=1}^{T}\frac {1}{2\sqrt{b_0^2+ kG^2} }}\left(B_{sgd} +\frac{\eta d L(16G)^2B_\delta}{2n^2\varepsilon^2 }\mathbb{E} \left[\sum_{j=0}^{T-1}  \frac{ \alpha^2_{j+1}}{b^2_{j+1}}  {\sum_{t=0}^{T-1}  \frac{ 1}{\alpha^2_{t+1}}  } \right]  \right).\label{eq:final}
\end{align}
% \begin{align}
%   \min_{\ell \in [T-1]} \mathbb{E}  \|\nabla F_{\ell}\|^2
%   &\leq \frac{1}{\sum_{k=1}^{T}\frac {1}{2\sqrt{b_0^2+ kG^2} }} \left(B_{sgd}+\frac{64\eta d LG^2B_\delta{T}}{n^2\varepsilon^2 \sqrt{\nu}}\right)\label{eq:ada1}
% \end{align}
Observe that,
\begin{align}
\sum_{k=1}^{T}\frac {1}{2\sqrt{b_0^2+ kG^2} }
&\geq  \frac {1}{G} \left( \sqrt{T+(b_0/G)^2} - \sqrt{1+({b_0}/{G})^2}\right)=\frac{T-1 }{\sqrt{TG^2+b_0^2} + \sqrt{b^2_0+G^2}} \label{eq:sqrt-1}
\end{align}
% $\frac {G}{ \sqrt{T+(b_0/G)^2} - \frac{b_0}{G}} =\frac{\sqrt{TG^2+b_0^2} + {b_0} }{T}$ 
and the fact that 
\begin{align}\frac{\sqrt{TG^2+b_0^2} + \sqrt{G^2+b_0^2} }{T-1}\leq \frac{2\sqrt{G^2+b_0^2}}{T-1}+\frac{G }{\sqrt{T-1}}\leq\frac{2G}{\sqrt{T-1}} \quad\text{ for } \quad T\geq 5+\frac{4b^2_0}{G^2}.\label{eq:sqrt-2}
\end{align}
Thus, applying \eqref{eq:sqrt-1} and \eqref{eq:sqrt-2} to \cref{eq:final} finishes the proof.
% \begin{align*}
%   \min_{\ell \in [T-1]} \mathbb{E}  \|\nabla F_{\ell}\|^2
%   &\leq \frac{\sqrt{TG^2+b_0^2} + {b_0} }{T}B_{sgd} \\ 
%     &\quad 
%   +\frac{\sqrt{TG^2+b_0^2} + {b_0} }{T}\frac{32G^2\eta d LB_\delta}{n^2\varepsilon^2 }\mathbb{E} \left[\sum_{j=0}^{T-1}  \frac{ \alpha^2_{j+1}}{2b^2_{j+1}}  {\sum_{t=1}^{T-1}  \frac{ 1}{\alpha^2_{t+1}}  } \right]  \\
%   &\overset{(a)}{\leq}\frac{2G}{\sqrt{T}}B_{sgd} + \frac{64G^3\eta d LB_\delta}{\sqrt{T}n^2\varepsilon^2 }\mathbb{E} \left[\sum_{j=0}^{T-1}  \frac{ \alpha^2_{j+1}}{2b^2_{j+1}}  {\sum_{t=1}^{T-1}  \frac{ 1}{\alpha^2_{t+1}}  } \right]  
% \end{align*}
% where (a) follows from the fact that 
% \[\frac{\sqrt{TG^2+b_0^2} + {b_0} }{T}\leq \frac{G}{\sqrt{T}}+\frac{2{b_0} }{T}\leq\frac{2G}{\sqrt{T}} \quad\text{ for } \quad \sqrt{T}\geq\frac{2b_0}{G}.\]
\end{proof}

\subsection{Proof for Proposition \ref{thm:adp-adagrad-norm}} \label{proof:adp-adagrad-norm}
We restate \cref{prop:adp-adagrad-norm} with \cref{prop:adp-adagrad-norm-a} for a more general algorithm -- \cref{alg:privacy-general-a}.
\begin{prop}[\textbf{ADP v.s. DP with an adaptive stepsize schedule}]\label{prop:adp-adagrad-norm-a}
Under the same conditions of \cref{thm:adp-adagrad-norm} 
% \becca{undefined; fix!!!!}
on $f$, $\sigma^2$, and $b_{t}$, if $\alpha_t = (b_0^2+tC)^{1/4}$ for some $C\in [\nu, G^2]$, 
then 
\begin{footnotesize}
\begin{align*}
 \textbf{(ADP-SGD)} \quad \mathbb{E}  \|\nabla F(\theta_{\tau}^{\rm ADP})\|^2
  \leq \frac{2GB_{\rm sgd}}{\sqrt{T-1}} +\frac{ 4G(16G)^2\eta d L\log^2(1.25/\delta )(\sqrt{T-1}+1)}{n^2\varepsilon^2 \nu}.
\end{align*}
\end{footnotesize}
In addition, if $\alpha_t =1$, then 
\begin{footnotesize}
\begin{align*}
  \textbf{(DP-SGD)} \quad \mathbb{E}  \|\nabla F(\theta_{\tau}^{\rm DP})\|^2
  \leq \frac{2GB_{\rm sgd}}{\sqrt{T-1}} +\frac{G(16G)^2\eta d L\log^2(1.25/\delta )(\sqrt{T-1}+1)\log \left(1+T\frac{\nu}{b_0^2} \right)}{n^2\varepsilon^2\nu}.
\end{align*} 
\end{footnotesize}
\end{prop}

\begin{proof}
% Setting $\alpha_t^2 = b_0^2+t\nu^2$, the bound in Lemma \ref{lem:adp-adagrad-norm} reduces to 
% \begin{align*}
%   \min_{\ell \in [T-1]} \mathbb{E}  \|\nabla F_{\ell}\|^2
%   &\leq \frac{1}{\sum_{k=1}^{T}\frac {B_{sgd}}{2\sqrt{b_0^2+ kG^2} }} +\frac{64G^3\eta d LB_\delta}{\sqrt{T}n^2\varepsilon^2 }\mathbb{E} \left[\sum_{j=0}^{T-1}  \frac{ \alpha^2_{j+1}}{2b^2_{j+1}}  {\sum_{t=0}^{T-1}  \frac{ 1}{\alpha^2_{t+1}}  } \right]  \right)
% \end{align*}
Starting with $M$ in Theorem \ref{thm:adp-adagrad-norm} with 
 $\alpha^2_t = \sqrt{b_0^2+tC}$, we have
  \begin{footnotesize}
\begin{align}
 \sum_{j=0}^{T-1}  \frac{ \alpha^2_{j+1}}{b^2_{j+1}}  \sum_{t=0}^{T-1}  \frac{ 1}{\alpha^2_{t+1}} 
  &\leq \sum_{j=1}^{T}  \frac{\sqrt{b_0^2+jC}}{b_0^2 +\sum_{t=0}^j\max\{\nu,\|\nabla f(\theta_j;x_{\xi_j})\|^2\}} \sum_{j=1}^{T}  \frac{1}{\sqrt{b_0^2 +jC}}\\
    &\leq \sum_{j=1}^{T}  \frac{\sqrt{b_0^2 +jC}}{{b_{0}^2+j\nu}}\sum_{t=1}^T\frac{1}{\sqrt{b_0^2 +jC}}\\
  &\leq \frac{2\sqrt{T}}{\sqrt{C}}\sum_{j=1}^{T}  \frac{\sqrt{b_0^2 +jC}}{{b_{0}^2+j\nu}}\\
   &=\frac{2\sqrt{T}}{\nu}\sum_{j=1}^{T}  \frac{\sqrt{b_0^2/C +j}}{{b_{0}^2/\nu+j}} \\
      &\leq\frac{2\sqrt{T}}{\nu}\sum_{j=1}^{T} \frac{1}{\sqrt{b_{0}^2/\nu+j}} \\
         &\leq\frac{4{T}}{\nu}. 
\end{align}
 \end{footnotesize}
%   $\alpha^2_t = \sqrt{b_0^2+t\nu}$, we have
%  \begin{footnotesize}
% \begin{align*}
%  \sum_{j=0}^{T-1}  \frac{ \alpha^2_{j+1}}{b^2_{j+1}}  \sum_{t=0}^{T-1}  \frac{ 1}{\alpha^2_{t+1}}  &\leq \sum_{j=1}^{T}  \frac{\alpha_{j}^2}{{b_{0}^2+\nuj}}\sum_{t=1}^T\frac{ 1}{\alpha^2_{t}} = \left(\sum_{j=1}^{T}  \frac{1}{\sqrt{b_0^2 +j \nu}}\right)^2  \leq \frac{4}{\nu^2}\left(\sqrt{\frac{b_0^2}{\nu}+ T}-\sqrt{\frac{b_0^2}{\nu}} \right)^2\leq \frac{4{T}}{{\nu}}
% \end{align*}
%  \end{footnotesize}
Thus, the bound in Theorem \ref{thm:adp-adagrad-norm} reduces to
\[
  \min_{\ell \in [T-1]} \mathbb{E}  \|\nabla F_{\ell}\|^2
  \leq \frac{2G}{\sqrt{T-1}} \left(B_{sgd}+\frac{2(16G)^2\eta d LB_\delta T}{n^2\varepsilon^2 \nu}\right).
\]

As for $\alpha^2_t = 1$, we have
\begin{align}
M = \sum_{j=0}^{T-1}  \frac{ \alpha^2_{j+1}}{b^2_{j+1}}  \sum_{t=1}^{T-1}  \frac{ 1}{\alpha^2_{t+1}}  =  T \sum_{j=0}^{T-1}  \frac{1}{b^2_{j+1}} \leq T\sum_{k=1}^{T}  \frac{1}{b_0^2 +k \nu}  \leq T\left(\frac{1}{\nu}\log \left(1+T\nu/b_0^2 \right)\right).
\end{align}
Applying the above inequality for $M$ reduces to the bound for $\alpha^2_t = 1$.
\end{proof}

% \begin{algorithm}[H]
%   \caption{AdaGrad-Torm} \label{alg:Adagrad}
% \begin{algorithmic}[1]
%   \State {\bfseries Input:} 
%   Initialize $x_0 \in \mathbb{R}^d, b_{0}>0, \eta>0$
%     \For{\texttt{ $j = 1,2, \ldots$ }}
%      \State Generate  $\xi_{j-1}$ and $G_{j-1} = G(x_{j-1}, \xi_{j-1})$
%       \State${b}_{j}^2 \leftarrow  {b}_{j-1} ^2+ { \| G_{j-1} \|^2}$  
%       \State $x_{j} \leftarrow x_{j-1} -  \frac{\eta}{{b}_{j}}G_{j-1} $
%       \EndFor
% \end{algorithmic}
%   \end{algorithm}
\section{Technical Lemma}\label{sec:tech}

\begin{lemma}[Descent Lemma]
\label{lem:descend}
Let $F\in C_L^1$.  Then,
$$
F(x) \leq F(y) + \langle{\nabla F(y), x-y \rangle} + \frac{L}{2} \| x - y \|^2.
$$
\end{lemma}

\begin{lemma} [Summation with power $p$] \label{lem:sum} For any positive number $a_1$ and $a_2$
\begin{align}
\sum_{t=1}^{T} \frac{ 1}{(a_1+a_2t)^{p}} &\leq 
\begin{cases} 
 \frac{1}{(1-p)a_2^p}((a_1/a_2+T)^{1-p}-(a_1/a_2)^{1-p}) & p<1\\
\frac{1}{a_2}\log(1+Ta_2/a_1) & p=1
% \\
% %   \frac{1}{p-1}\left(1-T^{1-p}\right)& p>1
%     \frac{a_2^p}{p-1} \left( \frac{a_1}{a_2}\right)^{1-p}& p>1
\end{cases}
% \end{align}
% \begin{align}
\\
\sum_{\ell=1}^{T} \frac{ 1}{(a_1+a_2t)^{p}} &\geq  
\begin{cases} 
 \frac{1}{(1-p)a_2^p}((a_1/a_2+1+T)^{1-p}-(a_1/a_2+1)^{1-p}) & p<1\\
\frac{1}{a_2}\log(1+T/(a_1/a_2+1)) & p=1
% \\
% %   \frac{1}{p-1}\left(1-T^{1-p}\right)& p>1
%     \frac{a_2^p}{p-1} \left( \frac{a_1}{a_2}\right)^{1-p}& p>1
\end{cases}
\end{align}
\end{lemma}

\begin{lemma}
\label{lem:logsum}
 For any non-negative $a_1,\cdots, a_T$, such that $a_1 > 1$,
\begin{equation} 
\label{eq:log}
\sum_{\ell=1}^T \frac{  a_{\ell}}{ {\sum_{i=1}^{\ell}a_{i}} }\leq  \log \left({\sum_{i=1}^{T}a_{i}}\right)+1.
\end{equation}
\end{lemma}

\section{Additional Experiments}
%\vspace{-1.3cm}
\begin{table}[H]
\caption{\small \textbf{Mean accuracy of ADP-SGD/DP-SGD with polynomially decaying stepsizes $\eta_t = 0.1+\alpha_T\sqrt{t}$ where $\alpha_T$ is the ratio depending on the final epochs/iterations $T$ such that  the learning rate at $T$ is $\eta_T=10^{-10}$ (see the orange curves in Figure 
\ref{fig:stepsize}).} This table reports \emph{accuracy} for CIFAR10 with the mean and the corresponding standard deviation over $\{\text{acc}^{last}_i\}_{i=1}^5$. Here,  $\text{acc}^{last}_i$ is the accuracy at the final iteration for the $i$-th independent experiment.
Each set $\{\text{acc}^{last}_i\}_{i=1}^5$ corresponds to a pair of $(\bar{\varepsilon},C_G,T, \text{Alg})$. The difference  (``Gap") between DP and ADP is provided for visualization purpose. However, we ignore those differences (``Gap")  between DP and ADP when one has an accuracy of less than 15$\%$. The results suggest that the more iterations or epochs we use, the more improvements ADP-SGD can potentially gain over DP-SGD. The results are reported in percentage ($\%$). The bolded number is the best accuracy in a row among epoch 60, 120 and 200 for the same gradient clipping $C_G$. See paragraph \textbf{Datasets and models} and \textbf{Performance Measurement} for details. } 
\label{table:cifar-decay-b}
\scalebox{0.65}{ 
\begin{tabular}{l|l|ccc|ccc|ccc}
\hline
\multirow{2}{*}{$\bar{\varepsilon}$}  & \multirow{2}{*}{Alg}   & \multicolumn{3}{c|}{Gradient clipping $C_G=0.5$}     & \multicolumn{3}{c|}{Gradient clipping $C_G=1$}     & \multicolumn{3}{c}{Gradient clipping $C_G=2.5$}          \\
                  &               & epoch=$60     $ & epoch$=120    $ & epoch$=200    $ & epoch$=60     $ & epoch$=120    $ & epoch$=200   $ & epoch$=60     $ & epoch$=120    $ & epoch$=200   $\\ \hline\hline
\multirow{3}{*}{0.8} & ADPSGD & $\mathbf{56.44} \pm 0.577$& $47.46 \pm 0.431$& $26.25 \pm 1.850$& $\mathbf{47.34} \pm 0.887$& $21.34 \pm 2.404$& $11.48 \pm 2.767$& $\mathbf{13.22} \pm 3.192$& $10.08 \pm 0.249$& $10.15 \pm 0.113$\\ 
 & DPSGD & $\mathbf{55.99} \pm 0.647$& $44.05 \pm 0.194$& $22.56 \pm 3.488$& $\mathbf{43.99} \pm 1.345$& $10.04 \pm 0.021$& $10.04 \pm 0.047$& $10.13 \pm 0.054$& $\mathbf{10.47} \pm 0.940$& $10.37 \pm 0.704$\\ 
 & Gap& $0.45$& $3.41$& $3.69$& $3.35$& $11.3$&N/A&N/A&N/A&N/A\\ \hline 
\multirow{3}{*}{1.2} & ADPSGD & $\mathbf{59.52} \pm 0.369$& $58.23 \pm 1.086$& $48.21 \pm 0.765$& $\mathbf{57.85} \pm 0.180$& $41.58 \pm 0.920$& $22.63 \pm 2.249$& $\mathbf{32.46} \pm 1.140$& $10.14 \pm 0.196$& $10.06 \pm 0.095$\\ 
 & DPSGD & $\mathbf{59.64} \pm 0.671$& $56.87 \pm 0.526$& $45.03 \pm 1.312$& $\mathbf{57.05} \pm 0.322$& $39.1 \pm 1.047$& $16.77 \pm 5.753$& $\mathbf{30.63} \pm 2.828$& $11.56 \pm 3.047$& $10.0 \pm 0.030$\\ 
 & Gap& $-0.12$& $1.36$& $3.18$& $0.8$& $2.48$& $5.86$& $1.83$&N/A&N/A\\ \hline 
\multirow{3}{*}{1.6} & ADPSGD & $61.02 \pm 0.284$& $\mathbf{61.45} \pm 0.291$& $57.45 \pm 0.414$& $\mathbf{61.75} \pm 0.545$& $54.29 \pm 0.578$& $36.4 \pm 1.243$& $\mathbf{48.39} \pm 0.866$& $18.39 \pm 5.369$& $10.63 \pm 1.065$\\ 
 & DPSGD & $60.67 \pm 0.429$& $\mathbf{61.26} \pm 0.216$& $55.22 \pm 0.981$& $\mathbf{61.4} \pm 0.674$& $52.58 \pm 0.469$& $35.7 \pm 2.453$& $\mathbf{45.44} \pm 0.672$& $15.73 \pm 6.957$& $9.982 \pm 0.046$\\ 
 & Gap& $0.35$& $0.19$& $2.23$& $0.35$& $1.71$& $0.7$& $2.95$& $2.66$&N/A\\ \hline 
\multirow{3}{*}{3.2} & ADPSGD & $61.7 \pm 0.252$& $65.57 \pm 0.371$& $\mathbf{66.21} \pm 0.587$& $65.09 \pm 0.345$& $\mathbf{65.6} \pm 0.134$& $62.77 \pm 0.491$& $\mathbf{65.54} \pm 0.384$& $55.44 \pm 0.359$& $38.65 \pm 1.554$\\ 
 & DPSGD & $61.61 \pm 0.290$& $65.36 \pm 0.126$& $\mathbf{66.07} \pm 0.168$& $65.03 \pm 0.233$& $\mathbf{65.71} \pm 0.343$& $61.56 \pm 0.475$& $\mathbf{64.56} \pm 0.467$& $53.32 \pm 0.939$& $32.94 \pm 5.487$\\ 
 & Gap& $0.09$& $0.21$& $0.14$& $0.06$& $-0.11$& $1.21$& $0.98$& $2.12$& $5.71$\\ \hline 
\multirow{3}{*}{6.4} & ADPSGD & $62.07 \pm 0.622$& $66.13 \pm 0.202$& $\mathbf{68.29} \pm 0.141$& $65.97 \pm 0.103$& $68.87 \pm 0.220$& $\mathbf{69.6} \pm 0.189$& $\mathbf{69.3} \pm 0.189$& $68.89 \pm 0.369$& $64.66 \pm 0.413$\\ 
 & DPSGD & $61.86 \pm 0.441$& $66.18 \pm 0.281$& $\mathbf{68.23} \pm 0.238$& $66.29 \pm 0.255$& $68.65 \pm 0.185$& $\mathbf{69.15} \pm 0.161$& $\mathbf{69.29} \pm 0.080$& $68.66 \pm 0.251$& $63.65 \pm 0.277$\\ 
 & Gap& $0.21$& $-0.05$& $0.06$& $-0.32$& $0.22$& $0.45$& $0.01$& $0.23$& $1.01$\\ \hline 
\end{tabular}}
\vspace{-0.15cm}
% \end{minipage}
\end{table}

\begin{table}[H]
% \begin{minipage}[t]{1\hsize}\centering
\caption{\small \textbf{Mean accuracy of ADP-SGD/DP-SGD with polynomially decaying stepsizes $\eta_t=\eta/b_{t+1}=1/\sqrt{20+t}$ (see the blue curve in Figure 
\ref{fig:stepsize}).} This table reports \emph{accuracy} for CIFAR10 with the mean and the corresponding standard deviation over $\{\text{acc}^{last}_i\}_{i=1}^5$. Here,  $\text{acc}^{last}_i$ is the accuracy at the final iteration for the $i$-th independent experiment.
Each set $\{\text{acc}^{last}_i\}_{i=1}^5$ corresponds to a pair of $(\bar{\varepsilon},C_G,T, \text{Alg})$.  See Table \ref{table:cifar-decay-b}  for reading instruction. } \label{table:cifar-decay-a}
\scalebox{0.7}{ 
\begin{tabular}{l|l|ccc|ccc}
\hline
\multirow{2}{*}{$\bar{\varepsilon}$}  & \multirow{2}{*}{Alg}   & \multicolumn{3}{c|}{Gradient clipping $C_G=0.5$}     & \multicolumn{3}{c}{Gradient clipping $C_G=1.0$}             \\
                  &               & epoch=$60     $ & epoch$=120    $ & epoch$=200    $ & epoch$=60     $ & epoch$=120    $ & epoch$=200   $ \\ \hline\hline
\multirow{3}{*}{0.8} & ADPSGD & $55.41 \pm 0.592$& $56.67 \pm 0.377$& $\mathbf{56.82} \pm 0.474$& $\mathbf{55.73} \pm 0.396$& $52.52 \pm 0.830$& $47.57 \pm 0.747$\\ 
 & DPSGD & $55.56 \pm 0.502$& $\mathbf{56.66} \pm 0.537$& $56.0 \pm 0.510$& $\mathbf{56.01} \pm 1.030$& $52.43 \pm 0.977$& $44.04 \pm 0.613$\\ 
 & Gap& $-0.15$& $0.01$& $0.82$& $-0.28$& $0.09$& $3.53$\\ \hline 
\multirow{3}{*}{1.2} & ADPSGD & $56.52 \pm 0.475$& $59.01 \pm 0.411$& $\mathbf{59.65} \pm 0.441$& $\mathbf{60.08} \pm 0.465$& $59.72 \pm 0.291$& $57.57 \pm 0.169$\\ 
 & DPSGD & $55.85 \pm 0.920$& $58.92 \pm 0.371$& $\mathbf{59.98} \pm 0.910$& $\mathbf{59.93} \pm 0.410$& $59.78 \pm 0.160$& $57.18 \pm 0.497$\\ 
 & Gap& $0.67$& $0.09$& $-0.33$& $0.15$& $-0.06$& $0.39$\\ \hline 
\multirow{3}{*}{1.6} & ADPSGD & $57.64 \pm 0.073$& $59.54 \pm 0.452$& $\mathbf{59.87} \pm 0.177$& $61.06 \pm 0.045$& $\mathbf{61.86} \pm 0.392$& $61.49 \pm 0.516$\\ 
 & DPSGD & $56.28 \pm 0.293$& $59.0 \pm 0.617$& $\mathbf{61.49} \pm 0.195$& $61.09 \pm 0.239$& $\mathbf{61.68} \pm 0.325$& $61.29 \pm 0.456$\\ 
 & Gap& $1.36$& $0.54$& $-1.62$& $-0.03$& $0.18$& $0.2$\\ \hline 
\multirow{3}{*}{3.2} & ADPSGD & $57.15 \pm 1.181$& $59.73 \pm 0.239$& $\mathbf{61.54} \pm 0.348$& $61.59 \pm 0.481$& $64.12 \pm 0.202$& $\mathbf{65.36} \pm 0.049$\\ 
 & DPSGD & $57.7 \pm 0.192$& $60.12 \pm 0.099$& $\mathbf{61.65} \pm 0.179$& $61.9 \pm 0.209$& $63.8 \pm 0.286$& $\mathbf{64.98} \pm 0.302$\\ 
 & Gap& $-0.55$& $-0.39$& $-0.11$& $-0.31$& $0.32$& $0.38$\\ \hline 
\multirow{3}{*}{6.4} & ADPSGD & $58.05 \pm 0.275$& $59.98 \pm 0.281$& $\mathbf{61.62} \pm 0.399$& $62.16 \pm 0.274$& $64.44 \pm 0.492$& $\mathbf{65.54} \pm 0.299$\\ 
 & DPSGD & $56.74 \pm 0.591$& $59.79 \pm 0.802$& $\mathbf{61.78} \pm 0.390$& $61.99 \pm 0.241$& $64.51 \pm 0.170$& $\mathbf{65.79} \pm 0.249$\\
 & Gap& $1.31$& $0.19$& $-0.16$& $0.17$& $-0.07$& $-0.25$\\ \hline 
\end{tabular}}
\scalebox{0.7}{ 
\begin{tabular}{l|l|ccc|ccc}
\hline
\multirow{2}{*}{$\bar{\varepsilon}$}  & \multirow{2}{*}{Alg}   & \multicolumn{3}{c|}{Gradient clipping $C_G=2.5$}     & \multicolumn{3}{c}{Gradient clipping $C_G=5.0$}             \\
                  &               & epoch=$60     $ & epoch$=120    $ & epoch$=200    $ & epoch$=60     $ & epoch$=120    $ & epoch$=200   $ \\ \hline\hline
\multirow{3}{*}{0.8} & ADPSGD & $\mathbf{35.88} \pm 0.620$& $23.98 \pm 2.957$& $10.75 \pm 1.062$& $\mathbf{10.32} \pm 0.299$& $10.19 \pm 0.053$& $10.01 \pm 0.230$\\ 
 & DPSGD & $\mathbf{37.23} \pm 1.305$& $20.09 \pm 3.306$& $10.07 \pm 0.084$& $\mathbf{11.11} \pm 2.052$& $9.963 \pm 0.119$& $10.53 \pm 1.101$\\ 
 & Gap& $-1.35$& $3.89$&N/A&N/A&N/A&N/A\\ \hline 
\multirow{3}{*}{1.2} & ADPSGD & $\mathbf{54.32} \pm 0.480$& $43.38 \pm 0.487$& $32.58 \pm 0.580$& $\mathbf{31.21} \pm 1.202$& $13.09 \pm 2.242$& $12.8 \pm 2.765$\\ 
 & DPSGD & $\mathbf{55.02} \pm 0.389$& $42.65 \pm 0.661$& $30.97 \pm 2.262$& $\mathbf{33.1} \pm 1.928$& $10.17 \pm 0.033$& $10.0 \pm 0.028$\\ 
 & Gap& $-0.7$& $0.73$& $1.61$& $-1.89$&N/A&N/A\\ \hline 
\multirow{3}{*}{1.6} & ADPSGD & $\mathbf{60.61} \pm 0.384$& $55.68 \pm 0.284$& $48.44 \pm 0.394$& $\mathbf{47.0} \pm 0.482$& $30.66 \pm 2.568$& $17.31 \pm 2.557$\\ 
 & DPSGD & $\mathbf{61.31} \pm 0.419$& $55.23 \pm 0.400$& $46.09 \pm 0.424$& $\mathbf{46.36} \pm 0.588$& $31.37 \pm 1.787$& $12.94 \pm 4.860$\\ 
 & Gap& $-0.7$& $0.45$& $2.35$& $0.64$& $-0.71$& $4.37$\\ \hline 
\multirow{3}{*}{3.2} & ADPSGD & $65.25 \pm 0.0$& $\mathbf{66.08} \pm 0.185$& $65.37 \pm 0.212$& $\mathbf{64.92} \pm 0.178$& $60.66 \pm 0.345$& $56.28 \pm 0.238$\\ 
 & DPSGD & $\mathbf{65.88} \pm 0.295$& $65.54 \pm 0.276$& $65.02 \pm 0.039$& $\mathbf{64.8} \pm 0.449$& $60.75 \pm 0.492$& $54.87 \pm 0.528$\\ 
 & Gap& $-0.63$& $0.54$& $0.35$& $0.12$& $-0.09$& $1.41$\\ \hline 
\multirow{3}{*}{6.4} & ADPSGD & $67.25 \pm 0.083$& $68.56 \pm 0.110$& $\mathbf{69.33} \pm 0.155$& $69.31 \pm 0.395$& $\mathbf{69.47} \pm 0.146$& $68.75 \pm 0.258$\\ 
 & DPSGD & $66.89 \pm 0.235$& $68.28 \pm 0.279$& $\mathbf{69.13} \pm 0.125$& $69.12 \pm 0.175$& $\mathbf{69.37} \pm 0.128$& $68.26 \pm 0.134$\\ 
 & Gap& $0.36$& $0.28$& $0.2$& $0.19$& $0.1$& $0.49$\\ \hline 

\end{tabular}}
\end{table}

\subsection{DP-SGD v.s. ADP-SGD with decaying stepsizes}\label{sec:decay-stepsize-compare} 
 See Table \ref{table:cifar-decay-b} and Table \ref{table:cifar-decay-a} for the statistics at the last iteration. Comparing Table \ref{table:cifar-decay-b} and Table \ref{table:cifar-decay1} (Table \ref{table:cifar-decay-a} and \ref{table:cifar-decay}), it appears that both DP-SGD and ADP-SGD do better at earlier iterations.  It appears that the difference (the row ``Gap'' in the tables) between the best iteration and iteration $T$ is pretty minimal for large $\varepsilon$ and $T$. Thus, our observation in the main text for Table \ref{table:cifar-decay1} and  \ref{table:cifar-decay} holds also for Table \ref{table:cifar-decay-b} and \ref{table:cifar-decay-a}.

%  \cref{table:cifar-decay-a-c}  indicates that, under the same epoch and same privacy budget $\varepsilon$, the parameters corresponding to the lowest averaged test error are mostly consistent between DP-SGD and ADP-SGD.  When comparing across the epochs given the same privacy budget, we see that, as the final number of iterations iterations $T$ increases, larger values of $c$ result in better test error.  However, When comparing across the privacy budgets given the same epoch, we see that, as the privacy budget $\varepsilon$ increases, smaller $c$ values yield better test errors.
 
% We present in \cref{fig:epoch-200} and \cref{fig:epoch-60} a comparison of the test errors between ADP-SGD and DP-SGD during the training process, for a range of $c$ values. \cref{fig:epoch-200} corresponds to \cref{table:cifar-decay} or \cref{table:cifar-decay-a-c} for the fixed epoch 200, and \cref{fig:epoch-60} is for epoch 60. We see from  \cref{fig:epoch-200} that the ADP-SGD plotted by orange solid curves is consistently better than the DP-SGD (blue dash curves) for the fixed epoch 200, while the two curves in \cref{fig:epoch-60} overlap mostly for the fixed epoch 60.    
% \input{zfigure1.tex}

\subsection{DP-SGD v.s. ADP-SGD with adaptive stepsizes}\label{sec:adaptive-stepsize-compare}%{sec:decay-stepsize-compare} 
For this set of experiments, we first tune the hyper-parameters, namely $\beta$ for DP-SGD and a pair of $(\beta, C)$ for ADP-SGD, whose optimal values are shown in \cref{table:cifar-adptive-beta-c}. Based on these hyper-parameters, we repeat the experiments five times and report the results in \cref{table:cifar-adptive}. 
\begin{table}[H]
\hspace{-0.2cm}
\small
\caption{\small \textbf{ADP-SGD v.s. DP-SGD with adaptive stepsizes.} The corresponding  $(\beta,C)$ for \cref{table:cifar-adptive}. }\label{table:cifar-adptive-beta-c}
\begin{tabular}{l|l|cccc}
\hline
Gradient Clipping                          & Algorithms   & $\bar{\varepsilon}=0.8$ & $\bar{\varepsilon}=1.6$ &  $\bar{\varepsilon}=3.2$&   $\bar{\varepsilon}=6.4$\\  \hline\hline
\multicolumn{1}{l|}{\multirow{2}{*}{$C_G=1.0$}}    &  DP-SGD with $\beta$   & $ 1024 $ & $ 512  $ & $1  $ & $ 1 $  \\
\multicolumn{1}{l|}{}                        &  ADP-SGD with $(\beta,C)$   & $ (1024,10^{-5})  $ & $ (512,10^{-4})  $ & $ (512,10^{-4}) $ & $ (1,10^{-4})  $ \\ \hline
\multicolumn{1}{l|}{\multirow{2}{*}{$C_G=2.5$}}    &  DP-SGD with $\beta$   & $ 1024 $ & $ 512  $ & $512  $ & $ 1 $  \\  
\multicolumn{1}{l|}{}                        &  ADP-SGD with $(\beta,C)$    & $ ( 1024,10^{-5})$ & $(512,10^{-5}) $ & $ (512,10^{-5})  $ & $(1,10^{-2}) $ \\ \hline
\end{tabular}
% \begin{tabular}{l|l|cccc}
% \hline
% Epoch                          & Algorithms   & $\varepsilon=1.6$ & $\varepsilon=3.2$ &  $\varepsilon=6.4$&   $\varepsilon=12.8$\\  \hline\hline
% \multicolumn{1}{l|}{\multirow{2}{*}{60}}    &  DP-SGD with $\beta$   & $ 1024 $ & $ 512  $ & $1  $ & $ 1 $  \\
% \multicolumn{1}{l|}{}                        &  ADP-SGD with $(\beta,C)$   & $ (1024,10^{-4})  $ & $ (512,10^{-4})  $ & $ (1,10^{-4}) $ & $ (1,10^{-4})  $ \\ \hline
% \multicolumn{1}{l|}{\multirow{2}{*}{120}}    &  DP-SGD with $\beta$   & $ 8192 $ & $ 1024  $ & $512  $ & $ 1 $  \\  
% \multicolumn{1}{l|}{}                        &  ADP-SGD with $(\beta,C)$    & $ ( 8192,10^{-5})$ & $(1024,10^{-3}) $ & $ (512,10^{-4})  $ & $(1,10^{-5}) $ \\ \hline
% \end{tabular}
\end{table}

\subsection{Constant stepsizes v.s. decaying stepsizes for DP-SGD}\label{sec:decaying-stepsize-better}

% \vspace{-1cm}
In this section, we justify why using a decaying stepsize  in \cref{alg:privacy-general} $\eta_t =\eta/b_{t+1}=1/\sqrt{2+ct}, c>0$ is better than a constant one for DP-SGD $\eta_t =1/\sqrt{2}$.  %Secondly, we also use the performance results in order to tune the hyperparameters $a$ and $c$ which define the stepsizes $\eta/b_t=1/\sqrt{2+ct}$ for \cref{thm:theorem-adp-descrease}. 
We use a  convolutional neural network (with the network parameters randomly initialized, see Figure \ref{fig:network2} for the architecture design) applied to the MNIST dataset. We analyze the accuracy of the classification results under several noise regimes characterized by $\sigma \in \{1.6, 3.2, 6.4, 12.8\}$ in \cref{alg:privacy-general}. We  vary  $c$ in $\{0, 10^{-5}, 10^{-3}, 10^{-2}, 10^{-1}, 5 \cdot 10^{-1}\}$. We note that $c = 0$ and $a=2$ correspond to a constant learning rate of $1/\sqrt{2}$.

We plot the test error (not accuracy) with respect to epoch, in order to better understand how the test error varies over time (see \cref{fig:privacy_dp-sgd}). On the same plot, we will also represent the privacy budget $\varepsilon$, obtained at each epoch, computed by using both the available code\footnote{\url{https://github.com/tensorflow/privacy/blob/master/tensorflow_privacy/privacy/analysis/rdp_accountant.py}} and the theoretical bound. 
% (detailed in \cref{sec:epsilon-disc}). 
From \cref{fig:privacy_dp-sgd}, we see that in all cases $c>0$ (corresponding to a non-constant learning rate) consistently performs better than  constant learning rate $c=0$. 

\begin{figure}[H]
    \centering
    % \subfloat[$\sigma = 1.6$ for  $20$  epochs]{\includegraphics[width = 2.5in]{experiments/privacy_vs_epoch_DPSGD_conv/conv_DPSGD_test_error_vs_epoch20_1.6.png}} 
    \subfloat[$\sigma = 1.6$ for  $50$  epochs]{\includegraphics[width = 2.5in]{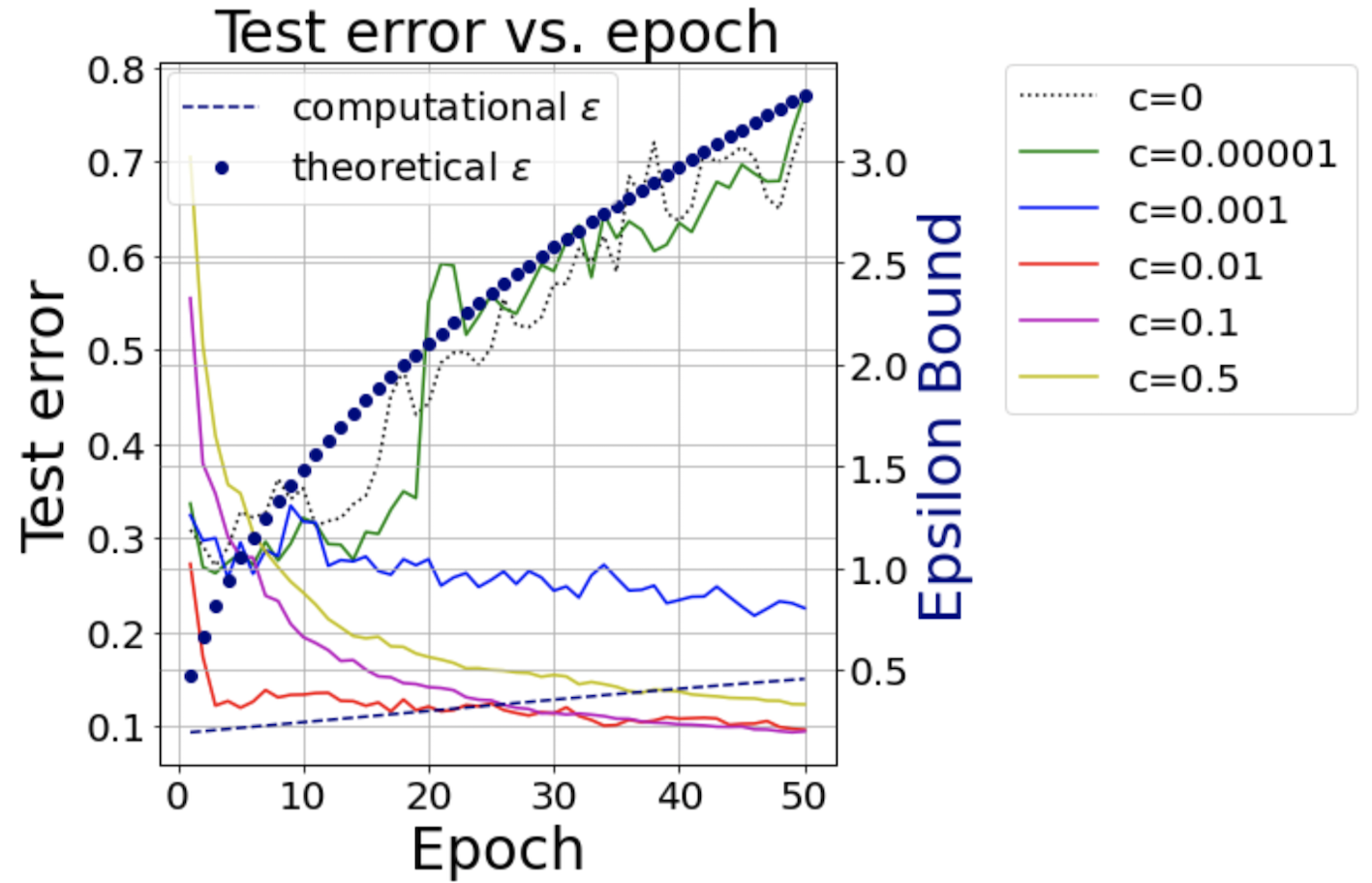}} 
    % \subfloat[$\sigma = 3.2$ for  $20$  epochs]{\includegraphics[width = 2.5in]{experiments/privacy_vs_epoch_DPSGD_conv/conv_DPSGD_test_error_vs_epoch20_3.2.png}}
    \subfloat[$\sigma = 3.2$ for  $50$  epochs]{\includegraphics[width = 2.5in]{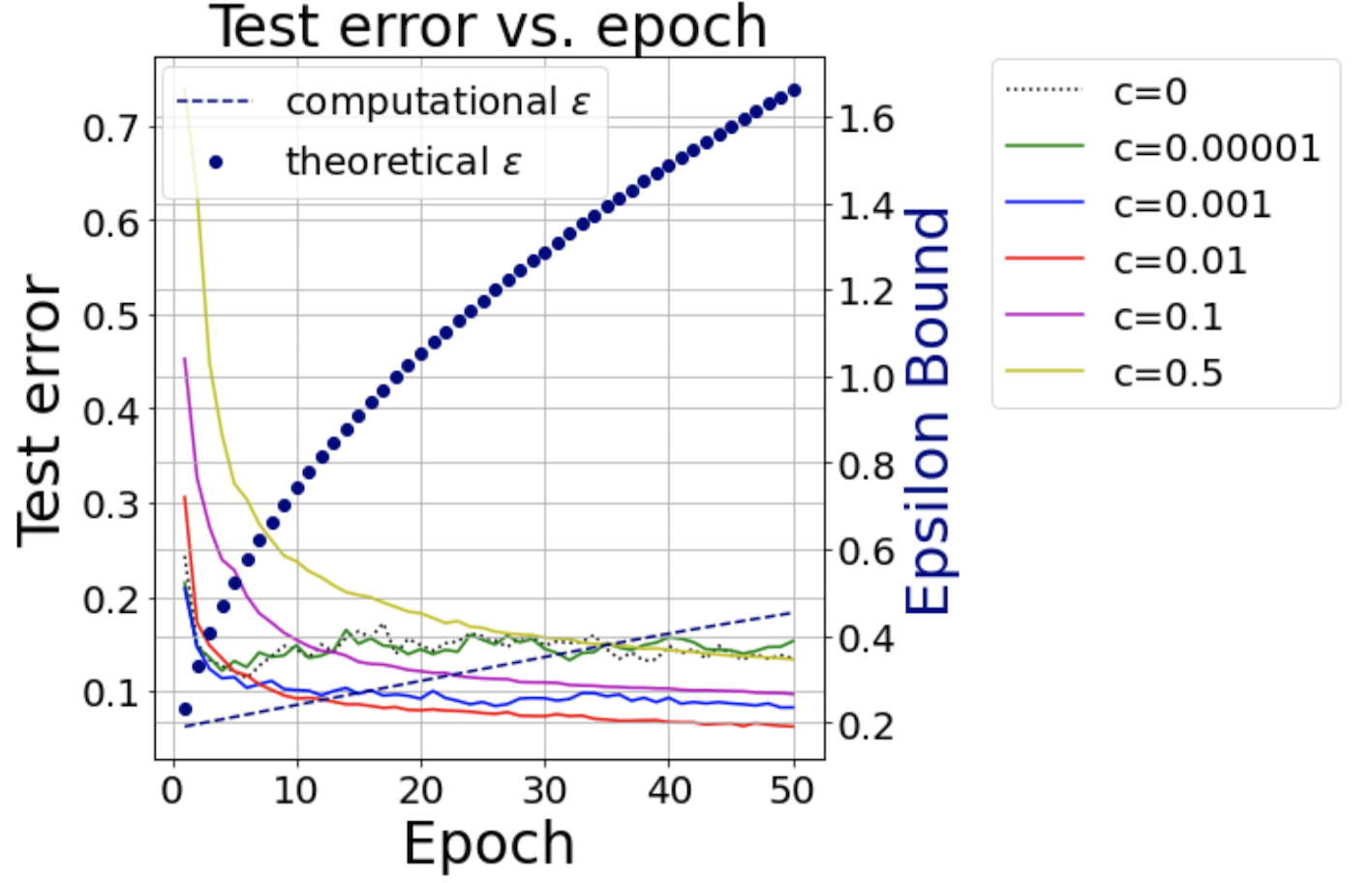}}\\
    % \subfloat[$\sigma = 6.4$ for  $20$  epochs]{\includegraphics[width = 2.5in]{experiments/privacy_vs_epoch_DPSGD_conv/conv_DPSGD_test_error_vs_epoch20_6.4.png}}
    \subfloat[$\sigma = 6.4$ for  $50$  epochs]{\includegraphics[width = 2.5in]{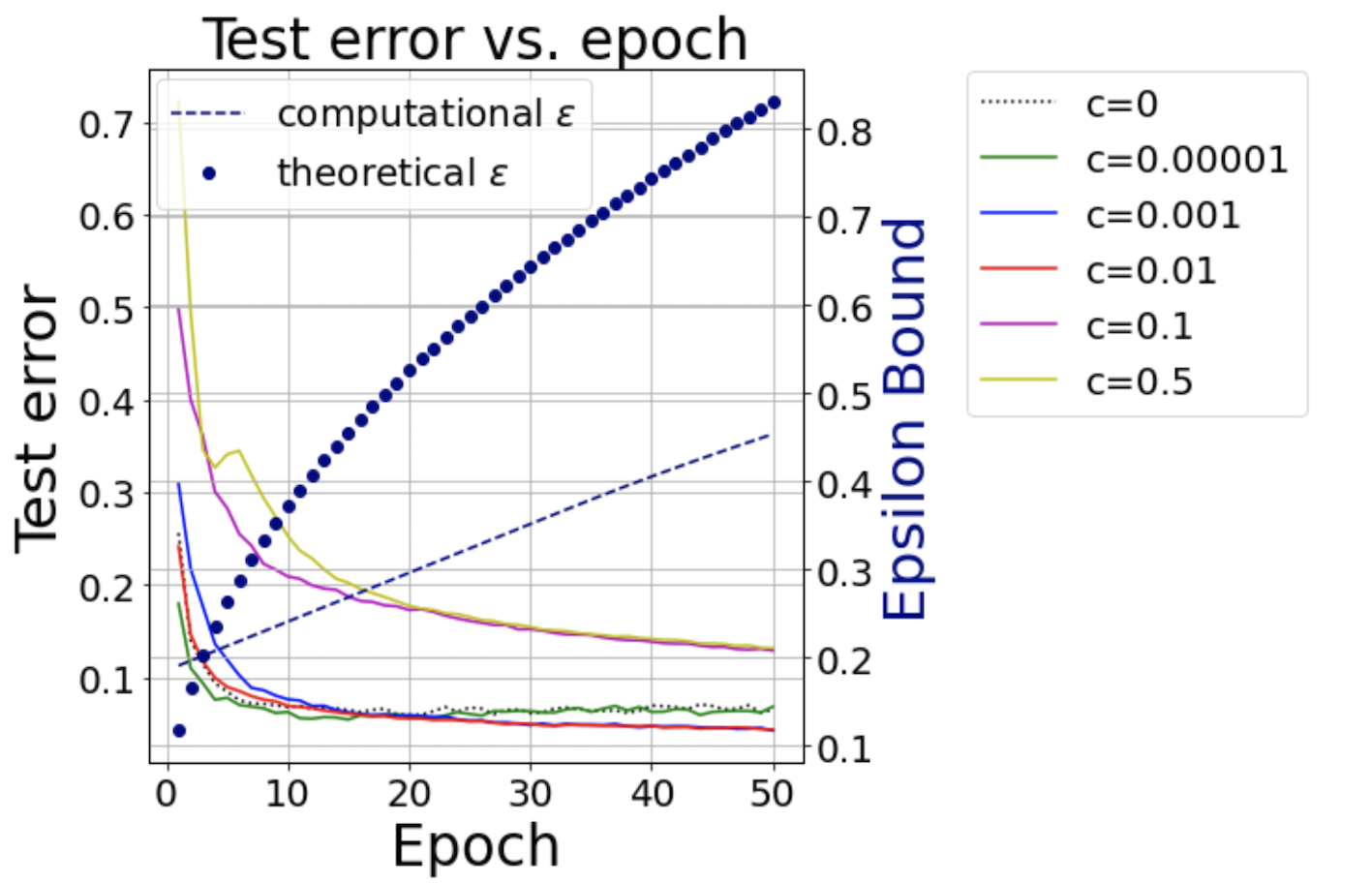}}
% \end{figure}
% \begin{figure}
%   \ContinuedFloat 
%   \centering 
%   \subfloat[$\sigma = 12.8$ for  $20$  epochs]{\includegraphics[width = 2.5in]{experiments/privacy_vs_epoch_DPSGD_conv/conv_DPSGD_test_error_vs_epoch20_12.8.png}}
  \subfloat[$\sigma = 12.8$ for  $50$  epochs]{\includegraphics[width = 2.5in]{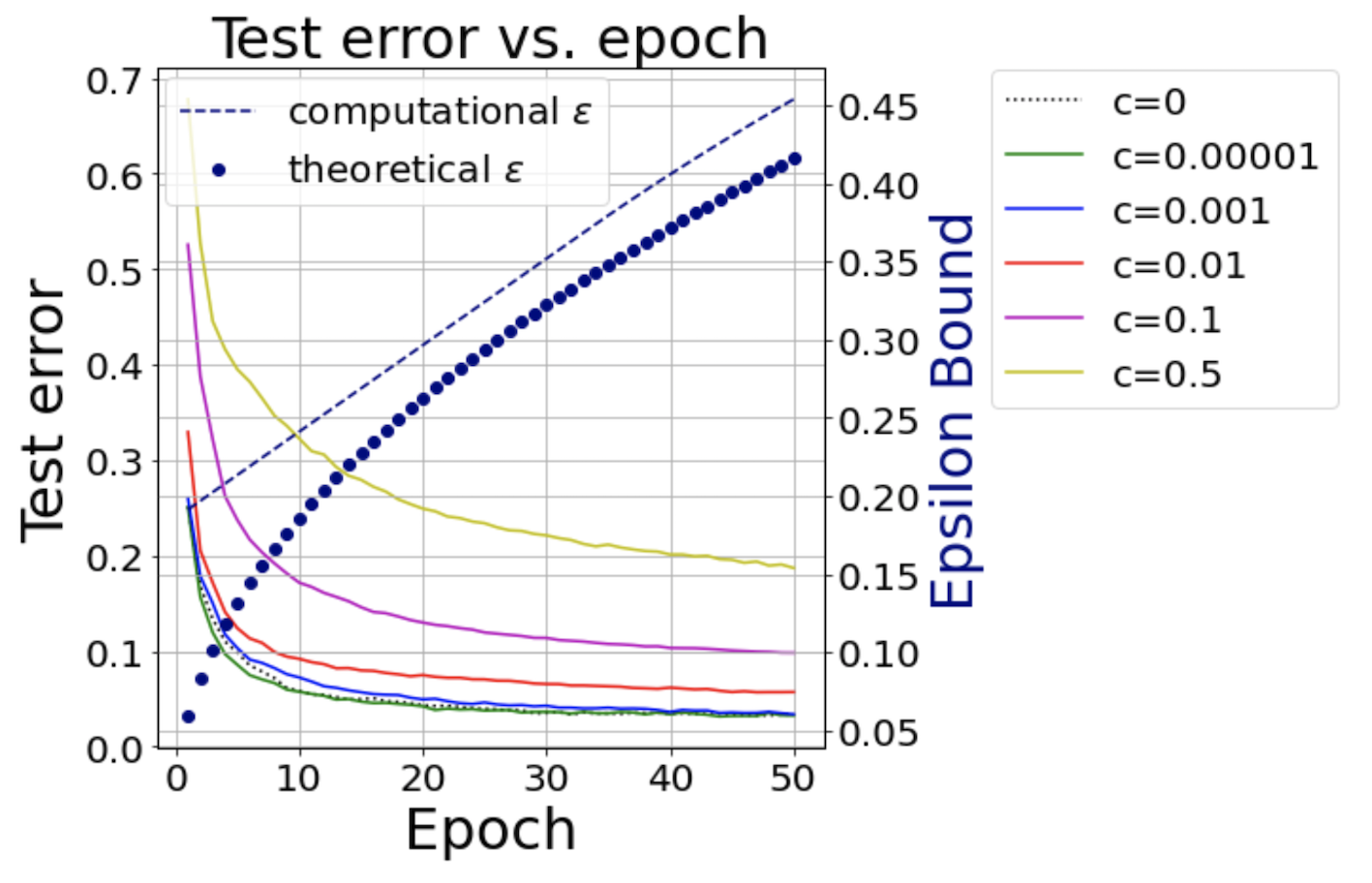}}
\caption{{\small \textbf{Constant stepsize v.s. decaying stepsize for DP-SGD.} We plot the test error (not accuracy), corresponding to the left y-axis, with respect to the epoch. Each plot corresponds to a fixed noise $\sigma$. Different color corresponds to a learning rate schedule  $\eta_t=\eta/b_t=1/\sqrt{2+ct}$ with $c$ described in the legend. In addition, we plot the numerical $\varepsilon$ (dash line)  and theoretical $\bar{\epsilon}$ (dot plot), corresponding to the right y-axis.  We see that the constant learning rate ($c=0$) is not as good as the decaying ones ($c>0$ ). }}
  \label{fig:privacy_dp-sgd}
\end{figure} 
% \vspace{-3cm}
 \begin{minipage}{0.35\textwidth}
 \centering
\begin{figure}[H]
 \centering
\includegraphics[width=0.8\linewidth]{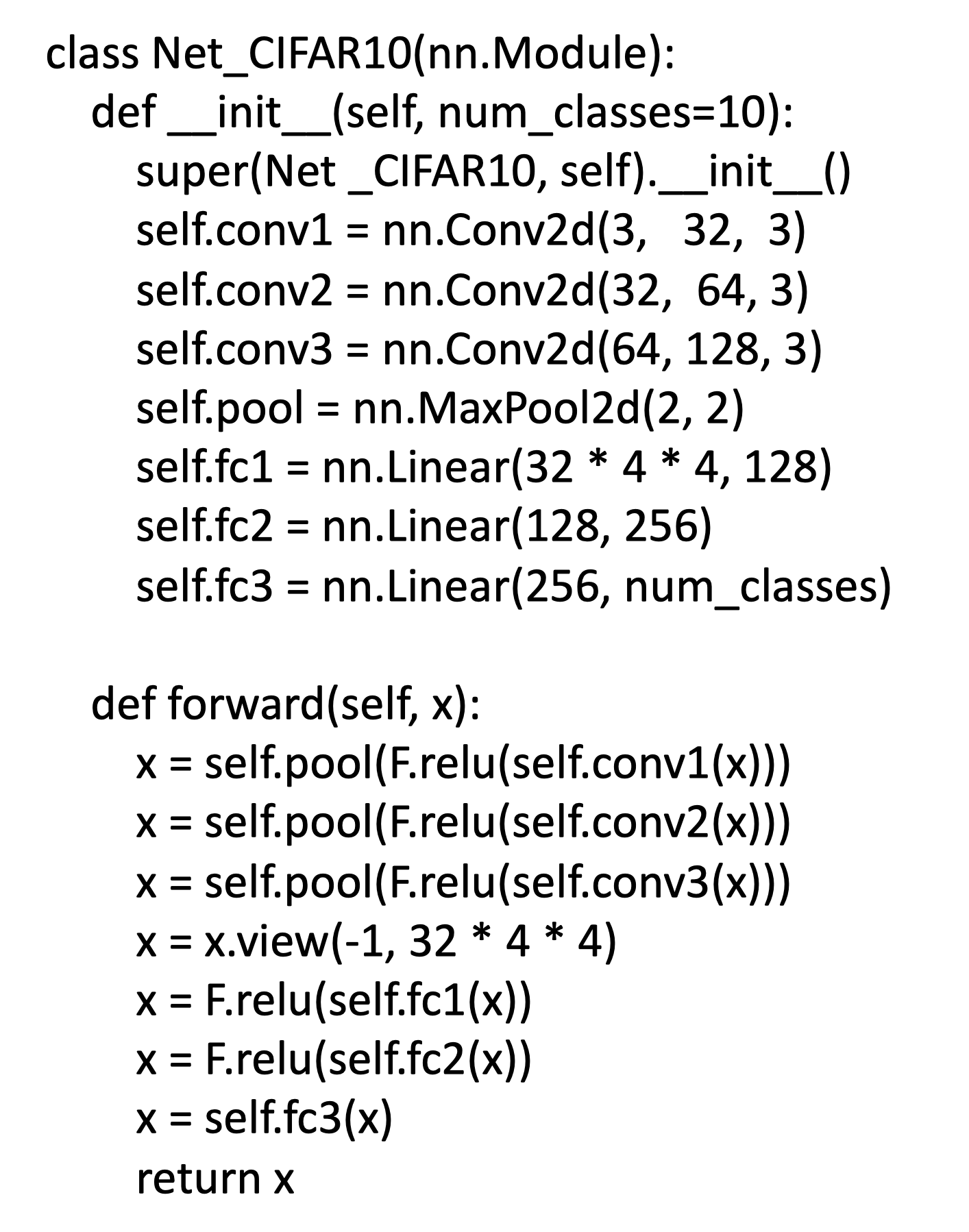}
\caption{{\small{Convolutional Neural Network for CIFAR10.} }} \label{fig:network1}
\end{figure}
 \end{minipage}
 \hspace{1.3cm}
 \begin{minipage}{0.4\textwidth}
 \begin{figure}[H]
 \centering
\includegraphics[width=0.8\linewidth]{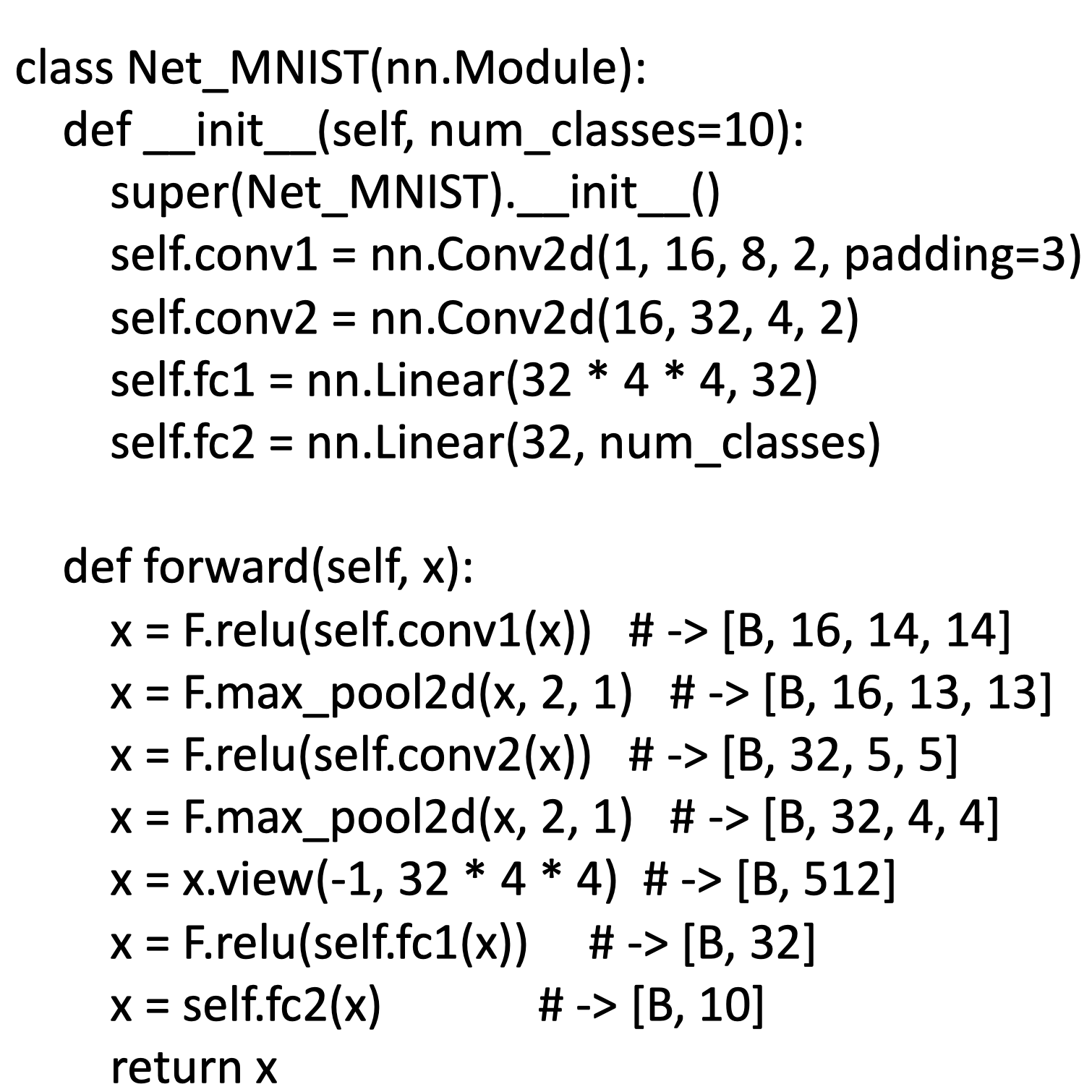} 
\caption{{\small {Convolutional Neural Network for MNIST.} }}  \label{fig:network2}
\end{figure}
 \end{minipage}
\subsection{Model architectures}
In Figure \ref{fig:network1} and \ref{fig:network2}, we present the CNN models in our experiments written in Python code based on PyTorch.\footnote{\url{https://pytorch.org/}}

\section{Code Demonstration}\label{sec:code}
% See Figure \ref{fig:code}

\begin{figure}[H]
 \centering
\includegraphics[width=1.\linewidth]{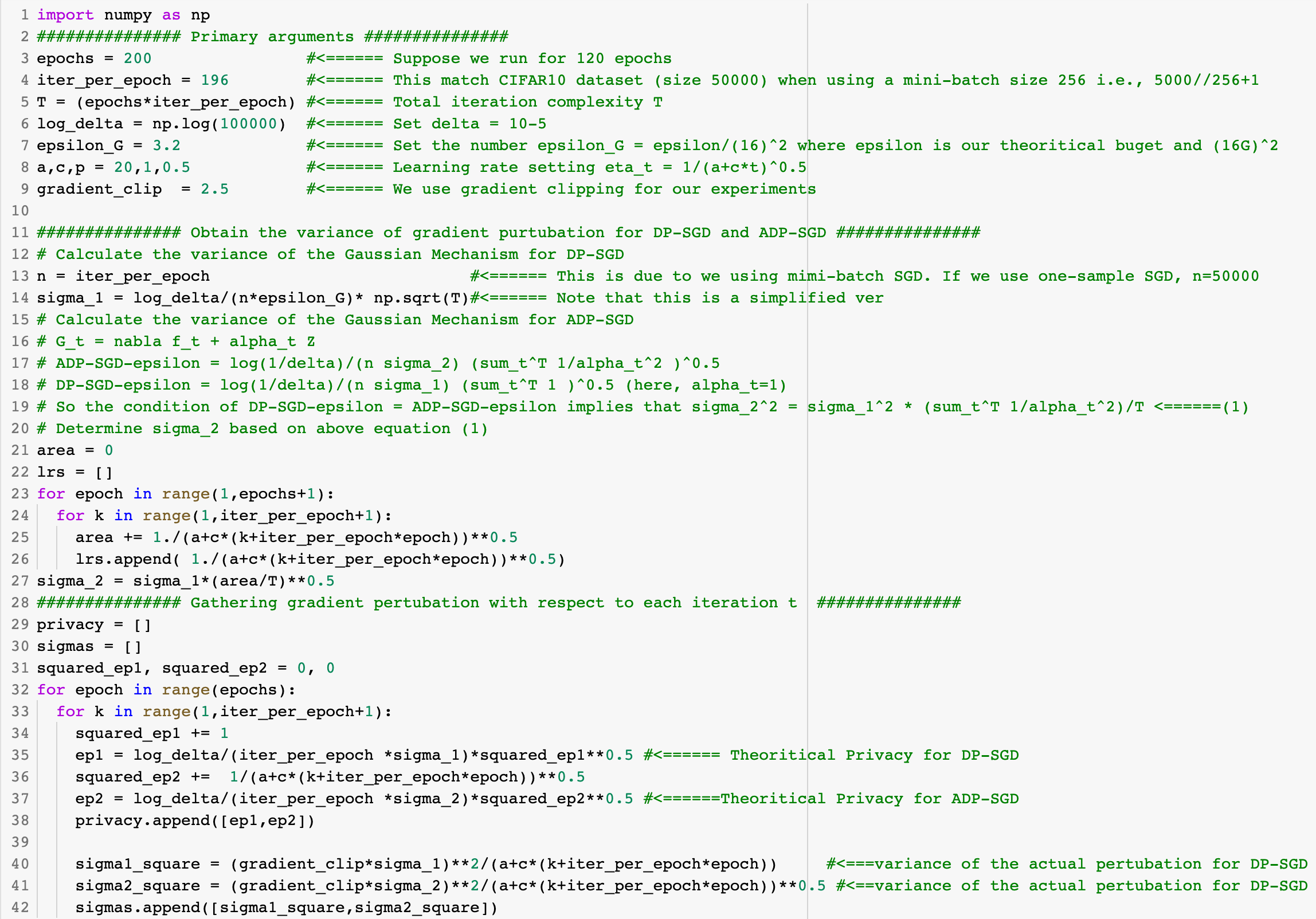}
\caption{Code to obtain Figure \ref{fig:compare} }\label{fig:code}
\end{figure}

\begin{figure}[tb]
 \centering
\includegraphics[width=1.\linewidth]{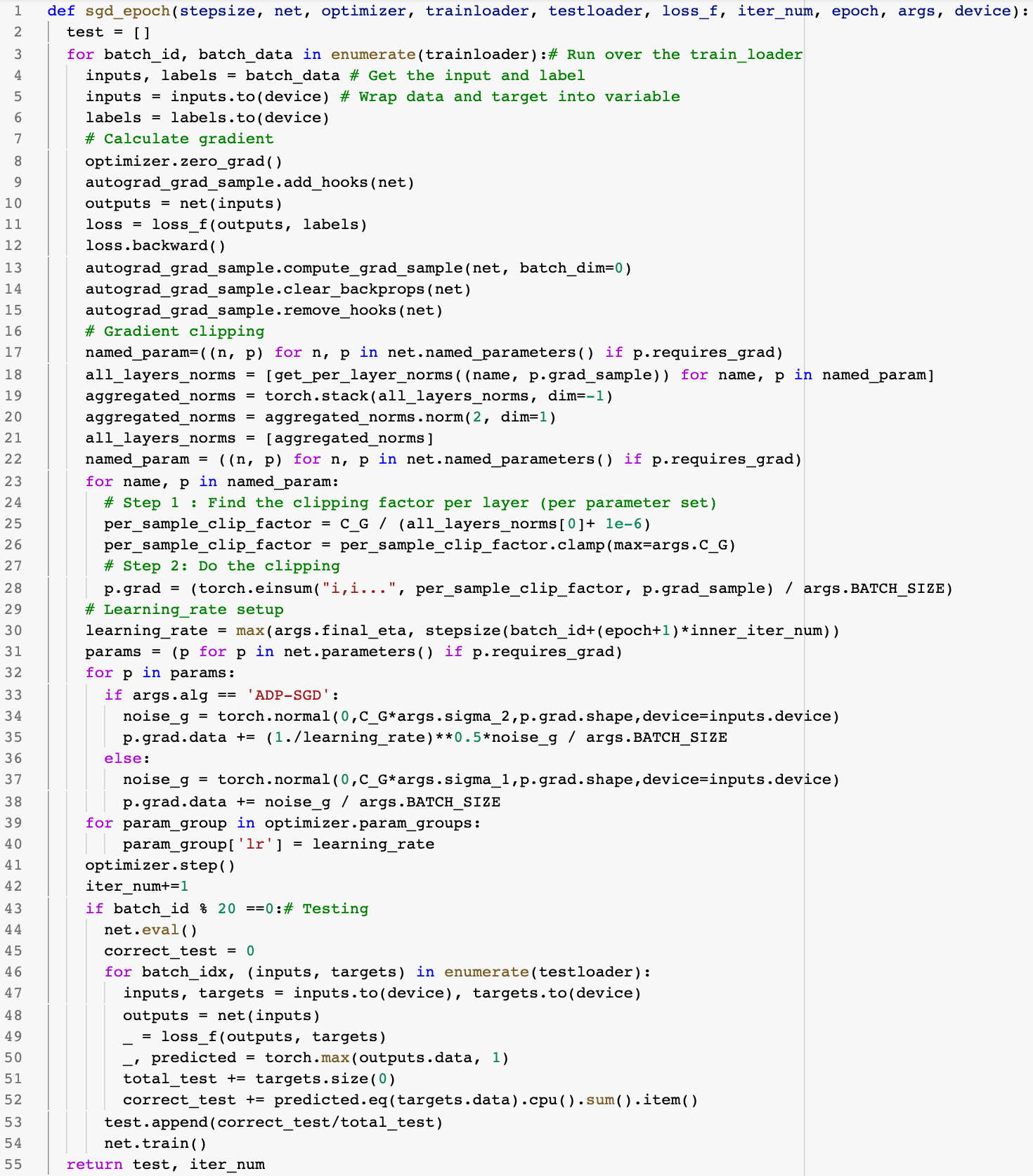}
\caption{Sample code (one epoch) based on PyTorch for training a CNN model over CIFAR10 data, whose results are shown in Table \ref{table:cifar-decay-a} and \ref{table:cifar-decay-b}}\label{fig:code1}
\end{figure}
\end{document}